\documentclass{article}

\usepackage{microtype}
\usepackage{graphicx}
\usepackage{subfigure}
\usepackage{booktabs} 

\usepackage{wrapfig}

\usepackage{hyperref}

\usepackage[accepted]{icml2024}

\usepackage{natbib}
\usepackage{amsmath}
\usepackage{amssymb}
\usepackage{mathtools}
\usepackage{amsthm}

\usepackage[capitalize,noabbrev]{cleveref}

\usepackage{xcolor}
\usepackage{verbatim}
\usepackage{tikz-cd}
\usepackage{thm-restate}
\usepackage{multirow}
\usepackage{quiver}
\newtheorem{theorem}{Theorem}[section]
\newtheorem{proposition}[theorem]{Proposition}
\newtheorem{lemma}[theorem]{Lemma}
\newtheorem{corollary}[theorem]{Corollary}
\newtheorem{fact}{Fact}
\theoremstyle{definition}
\newtheorem{definition}[theorem]{Definition}

\theoremstyle{remark}
\newtheorem{remark}[theorem]{Remark}
\newtheorem{example}{Example}

\DeclareMathOperator{\im}{im}
\DeclareMathOperator{\diag}{diag}
\DeclareMathOperator{\SPE}{SPE}
\DeclareMathOperator{\row}{row}
\DeclareMathOperator{\col}{col}

\usepackage[textsize=tiny]{todonotes}

\newcommand{\R}{\mathbb{R}}

\newcommand{\multiset}[1]{\{\!\!\{#1\}\!\!\}}
\newcommand{\wlcolor}[2]{\ensuremath{\chi_{{#1}}^{({#2})}}}

\newcommand\zhengchao[1]  {\textcolor{red}{ZW: #1}}

\newcommand{\mitchell}[1]  {\textcolor{cyan}{Mitchell: #1}}

\newcommand{\yusu}[1] {\textcolor{blue}{{\sf [Yusu]: #1}}}

\icmltitlerunning{Comparing Graph Transformers via Positional Encodings}

\begin{document}

\twocolumn[
\icmltitle{Comparing Graph Transformers via Positional Encodings}



\icmlsetsymbol{equal}{*}

\begin{icmlauthorlist}
\icmlauthor{Mitchell Black}{equal,osu}
\icmlauthor{Zhengchao Wan}{equal,ucsd}
\icmlauthor{Gal Mishne}{ucsd}
\icmlauthor{Amir Nayyeri}{osu}
\icmlauthor{Yusu Wang}{ucsd}
\end{icmlauthorlist}

\icmlaffiliation{osu}{School of Electrical Engineering and Computer Science, Oregon State University, Corvallis, Oregon, USA}
\icmlaffiliation{ucsd}{Hal\i{}c\i{}o\v{g}lu Data Science Institute, University of California San Diego, San Diego, California, USA}

\icmlcorrespondingauthor{Mitchell Black}{blackmit@oregonstate.edu}

\icmlkeywords{Machine Learning, ICML}

\vskip 0.3in
]



\printAffiliationsAndNotice{\icmlEqualContribution} 

\begin{abstract}
    The distinguishing power of graph transformers is tied to the choice of \textit{positional encoding}: features used to augment the base transformer with information about the graph. There are two primary types of positional encoding: \emph{absolute positional encodings (APEs)} and \emph{relative positional encodings (RPEs)}. APEs assign features to each node and are given as input to the transformer. RPEs instead assign a feature to each \emph{pair of nodes}, e.g., shortest-path distance, and are used to augment the attention block. A priori, it is unclear which method is better for maximizing the power of the resulting graph transformer. In this paper, we aim to understand the relationship between these different types of positional encodings. Interestingly, we show that graph transformers using APEs and RPEs are equivalent in their ability to distinguish non-isomorphic graphs. In particular, we demonstrate how to interchange APEs and RPEs while maintaining their distinguishing power in terms of graph transformers. However, in the case of graphs with node features, we show that RPEs may have an advantage over APEs. Based on our theoretical results, we provide a study of different APEs and RPEs---including the shortest-path and resistance distance and the recently introduced stable and expressive positional encoding (SPE)---and compare their distinguishing power in terms of transformers. We believe our work will help navigate the vast number of positional encoding choices and provide guidance on the future design of positional encodings for graph transformers.
\end{abstract}

\section{Introduction}
Graph transformers (GTs)~\citep{dwivedi2020generalization,ying2021transformers} have recently emerged as a competitor to message passing neural networks (MPNNs)~\citep{gilmer2017neural} for graph representation learning. While MPNNs only consider the immediate neighbors of a node when updating that node's feature, GTs consider all nodes in the graph. Accordingly, graph transformers can capture global information of graphs that can potentially be used to address the issue of oversquashing faced by MPNNs. Likewise, the attention mechanism of transformers can address the issue of oversmoothing faced by MPNNs. See~\cite{muller2023attending}.

However, one caveat in applying transformers to graph data is that transformers are \textit{\textbf{equivariant}}, meaning they treat nodes symmetrically and cannot distinguish two nodes at different positions in the graph. Accordingly, much of the research on graph transformers has been focused on the design of positional encodings that incorporate information about the graph structure into the input or architecture of transformers. There are currently two major types of positional encodings:

\textbf{Absolute positional encodings (APEs)} encode the graph structure as an embedding of the vertices $\phi:V\to\R^{d}$. APEs are either added to or concatenated with the initial vertex features that are the input to the transformer. Typical examples are the Laplacian eigenvectors~\cite{dwivedi2020generalization} or information about random walks~\citep{rampavsek2022recipe}. Recent works have also proposed learnable APEs~\citep{lim2022sign, huang2023stability}.

\textbf{Relative positional encodings (RPEs)} encode the graph structure as an embedding of pairs of vertices $\psi:V\times V\to\R^{d}$. RPEs are incorporated into a GT via a modified attention mechanism. Examples of RPEs include the shortest-path distance~\citep{ying2021transformers}, resistance distance~\cite{zhang2023rethinking}, (directed) random walk matrices~\cite{ma2023inductive,geisler2023transformers}, and heat kernels~\citep{choromanski22blocktoeplitz}. The adjacency matrix can also serve as an RPE, which will result in an RPE GT with the same distinguishing power as MPNNs (\Cref{cor:combinatorially-aware-rpe-stronger-than-WL}).



Although different kinds of APEs and RPEs have been proposed to encode graph structure and enhance the performance of GTs, there is a lack of understanding of how different positional encodings compare---e.g.,~shortest-path vs.~resistance distance---for distinguishing non-isomorphic graphs. Only a few works \cite{zhang2023rethinking} have attempted to compare different RPEs, let alone conduct a systematic study comparing APEs and RPEs. This results in a lack of guidance for the design of positional encodings for GTs: which type of positional encoding---absolute or relative---should be used in practice? Furthermore, this could lead to inefficiency in constructing positional encodings. In some recent works on designing APEs like  BasisNet~\citep{lim2022sign} and the stable and expressive positional encoding (SPE)~\citep{huang2023stability}, the authors first construct well-defined RPEs and then use some extra machinery to artificially transform RPEs into APEs. This raises the natural question of whether the transition is worthwhile given the expensive extra computation.



{\bf Our contributions.} We establish a theoretical framework for comparing positional encodings, not only for those in the same category, i.e., absolute or relative, but also across APEs and RPEs. Our main contributions are as follows:

(1) In \Cref{sec:equivalence between PEs}, we establish that specific types of APE and RPE GTs (cf. \Cref{def:graph_transformer}) are equivalent in their ability to distinguish non-isomorphic graphs. In particular, we show how to map APEs to RPEs and vice versa while maintaining their distinguishing power. We empirically validate our theoretical results on several graph classification datasets.

(2) We show that RPEs may have an advantage over APEs for distinguishing graphs with node features. In particular, we show that transforming an RPE to an APE results in a GT that is able to distinguish strictly fewer graphs with node features than a GT that uses the RPE directly.

(3) The techniques for proving the above results enable us to compare the distinguishing power of different positional encodings. In \Cref{sec:comparing_different_pes}, we provide a case study on several popular APEs and RPEs by comparing their distinguishing power. For example, we prove that GTs using SPE~\citep{huang2023stability} as an APE are stronger than the GTs using resistance distance as an RPE~\citep{zhang2023rethinking}.

(4) In the process of establishing our main results, we identify a new variant of the Weisfeiler-Lehman (WL)  test \cite{weisfeiler1968reduction} that we call the RPE-augWL test that allows us to connect the distinguishing power of RPEs and corresponding GTs. This RPE-augWL test is interesting in itself as a way to further the theoretical formulation of the representation power of graph learning frameworks. We generalize the equivalence between the WL and 2-WL tests to our RPE-augWL test and establish the equivalence between the resistance distance and the pseudoinverse of the Laplacian in terms of distinguishing power via the RPE-augWL test. Furthermore, we identify a class of RPEs resulting in a family of RPE GTs at least as strong as MPNNs.

\section{Positional Encodings and Graph Transformers}\label{sec:position_encoding}
In this section, we introduce some preliminaries on positional encodings and graph transformers, as well as some auxiliary results that will be used to prove our main results.

Let $G=(V,E)$ denote an unweighted graph, which can be either directed or undirected. A \textit{\textbf{featured graph}} $(G,X)$ includes node features \( X\in \mathbb{R}^{|V| \times d} \), indicating a \( d \)-dimensional feature $X(v)$ for each node in $v\in V$. Let $A$ denote the adjacency matrix, $D$ the diagonal matrix of the degrees, and $L=D-A$ the graph Laplacian matrix. To prevent ambiguity when working with multiple graphs, we use \( G \) as a subscript in our notations, such as \( V_G \) for the vertex set or \( X_G \) for the node features of graph \( G \).

For two graphs $G$ and $H$, a graph isomorphism is a bijection $\sigma:V_G\to V_H$ such that $\{u,v\}\in E_G$ if and only if $\{ \sigma(u), \sigma(v) \}\in E_H$. Two graphs $G$ and $H$ are isomorphic if there is a graph isomorphism $\sigma:V_G\to V_H$.

\subsection{Positional Encodings}
A positional encoding (PE) of a graph is a way of summarizing structural or positional information of the graph. In their most general form, there are two types of positional encodings we consider in this paper:
\begin{definition}
\label{def:pe_maps}
(Positional Encodings)
 \\
    An \textit{\textbf{absolute positional encoding (APE)}} $\phi$ assigns each graph $G$ a map $\phi_{G}:V_{G}\to\R^{l}$ such that for any two isomorphic graphs $G$ and $H$ and graph isomorphism $\sigma:V_{G}\to V_{H}$, one has that $\phi_{G} = \phi_{H}\circ\sigma$.
    \\
    A \textit{\textbf{relative positional encoding (RPE)}} $\psi$ assigns each graph $G$ a map $\psi_{G}:V_G\times V_G\to\R^k$ such that for any two isomorphic graphs $G$ and $H$ and graph isomorphism $\sigma:V_{G}\to V_{H}$, one has that $\psi_{G} = \psi_{H}\circ(\sigma\times\sigma)$.

    An APE is also naturally expressed as a matrix in $\R^{|V_G|\times l}$ and an RPE as a tensor $\R^{|V_G|\times |V_G|\times k}$.
\end{definition}

\textbf{Note:} While node features and APEs both assign a vector to each node, we emphasize that the difference between node features and APEs is that APEs are dependent on the topology of a graph, while two isomorphic graphs can have different node features.
\par
One of the simplest APEs is the degree map $\deg$: for any graph $G$, $\deg_{G}(u)$ is the degree of the vertex $u$.  Any graph distance is an RPE, e.g., shortest-path distance (SPD) or resistance distance (RD).
\par
Laplacian eigenvectors are commonly used as APEs \citep{dwivedi2020generalization, kreuzer2021rethinking}. However, there is no unique choice of eigenvector for a given eigenspace.
Therefore, Laplacian eigenvectors are not a well-defined APE by our definition, as isomorphic graphs may have different eigenvector encodings due to the choice of basis. However, the projection matrices onto the eigenspaces are a well-defined \textit{RPE}~\cite{lim2022sign}.

\subsection{Graph Transformers}
A \textit{\textbf{transformer}} is an architecture $T$ composed of \textit{\textbf{multiheaded attention layers}} $T = \mathrm{MHA}^{(L)}\circ\cdots \circ \mathrm{MHA}^{(1)}$, where a multiheaded attention layer with heads $1\leq h\leq H$ is a map $\mathrm{MHA}^{(l)}$ that sends $X^{(l)}\in\R^{n\times d}$ to $\mathrm{MHA}^{(l)}(X^{(l)}) = X^{(l+1)}\in\R^{n\times d}$ as defined below:
\begin{gather*}
    A^{(l,h)}(X^{(l)})=\mathrm{softmax}\left(\frac{X^{(l)}W^{(l,h)}_{Q}(X^{(l)}W^{(l,h)}_{K})^T}{\sqrt{d_h}}\right) \\
    Y^{(l)} = X^{(l)}+\left(\sum_{h=1}^{H} A^{(l,h)}(X^{(l)})X^{(l)}W^{(l,h)}_V \right)W^{(l)}_{O} \\
    X^{(l+1)} = Y^{(l)} + \sigma\left(Y^{(l)}W_1^{(l)}\right)W_2^{(l)}
\end{gather*}
where $W^{(l,h)}_{Q},W^{(l,h)}_{K},W^{(l,h)}_{V}, W^{(l)}_{O}\in\R^{d\times d_h}$, $W_1\in\R^{d_h\times d_r}$, $W^{(l)}_2\in\R^{d_r\times d}$ and $\sigma$ is an activation function. Here $Q,K,V$ refer to ``query'', ``key'' and ``value'', respectively, in the terminology of the attention literature \cite{vaswani2017attention}, and $d_h$ and $d_r$ are the dimension of the hidden layer and the dimension of the residual layer, respectively. The activation function $\sigma$ is usually the Gaussian error linear units (GELU) \cite{hendrycks2016gaussian}.

\begin{fact}[Transformers are Permutation Equivariant]
\label{thm:transformers_permutation_equivariant}
    Let $T$ be a transformer. Let $X\in\R^{n\times d}$. Let $P\in\R^{n\times n}$ be a permutation matrix. Then $PT(X) = T(PX)$.
\end{fact}

When applying a transformer to a graph $G$, the input is a set of node features $X_{G}\in\R^{|V_G|\times d}$; however, initial node features often do not contain any information about the topology of a graph, so two graphs with the same multiset of features but different topologies will have the same output from a transformer. Graph transformers augment the transformer by incorporating information about the input graph in the form of absolute or relative positional encodings.

\begin{definition}[Graph Transformers]
\label{def:graph_transformer}
    An \emph{\textbf{APE graph transformer with APE $\boldsymbol{\phi}$ ($\boldsymbol{\phi}$-APE-GT)}} is a transformer that either: concatenates the node features $X_G\in\R^{|V_G|\times d}$ with the APE $\phi_G\in\R^{|V_G|\times l}$, i.e.~$\widehat{X_G}=[X_G|\phi_G]\in \R^{|V_G|\times (d+l)}$, or adds the node features and APE, i.e.~$\widehat{X_G} = X_G + \phi_{G}$, before passing them through a transformer.
    \par
    An \emph{\textbf{RPE graph transformer with RPE $\boldsymbol{\psi}$ ($\boldsymbol{\psi}$-RPE-GT)}} is a transformer whose attention layer is modified as follows for a graph $G$ with node features $X\in\R^{|V_G|\times d}$:
    \begin{equation*}\label{eq:RPE transformer}
        A(X)=f_1(\psi_G)\odot\mathrm{softmax}\left(\frac{XW_{Q}(XW_{K})^T}{\sqrt{d_h}}+f_2(\psi_G)\right)
    \end{equation*}
    where $f_1,f_2:\R^k\to\R$ are functions applied entrywise to the tensor $\psi_G\in\R^{|V_G|\times |V_G|\times k}$.
\end{definition}

The original transformer~\cite{vaswani2017attention} used sines and cosines as APEs for sequences. Transformers for sequences with RPEs were subsequently proposed by~\citet{shaw2018self}. The first instance of an APE-GT for graphs was proposed by \citet{dwivedi2020generalization} using the Laplacian eigenvectors. The first instance of an RPE-GT was proposed by \citet{ying2021transformers} using the shortest-path distance, while the form of RPE-GT we consider was proposed by~\citet{zhang2023rethinking}. Although other ways of incorporating an RPE into a transformer have been proposed~\citep{mialon2021graphit,ma2023inductive}, in this paper, APE-GT and RPE-GT only refer to transformers of the forms in \Cref{def:graph_transformer}.

\subsection{Properties of RPEs}
In this subsection, we describe two special types of RPEs that will be useful later.



\paragraph{Diagonally-Aware RPEs.}

One benefit of using graph distances as RPEs---compared to e.g.,~the adjacency matrix---is their distinct diagonal entries are all zeros, in contrast to the positive off-diagonal terms. We characterize this property as ``diagonal-awareness''. As we will see throughout the paper, this property is important as it allows various algorithms to distinguish the feature at a node from all other features.







\begin{definition}
    An RPE $\psi$ is \emph{\textbf{diagonally-aware}} if for any two graphs $G$ and $H$ and vertices $v\in V_G$ and $x,y\in V_H$, then $\psi_{G}(v,v) = \psi_{H}(x,y)$ only if $x=y$.
\end{definition}



Diagonal-awareness is a very mild condition as one can always augment an RPE so that it becomes diagonally-aware:

\begin{definition}
\label{def:diag-aug}
    The \textbf{diagonal augmentation} $D^\psi$ of an RPE $\psi$ is defined, for a graph $G$, as the stacking $D^\psi_G:=(I_G,\psi_G)$ where $I_G$ is the identity matrix.
\end{definition}




For any graph $G$, the diagonal augmentation $D^{\psi}_G(u,v)$ has a 1 in the first coordinate iff $u=v$ and 0 otherwise. Therefore, $D^{\psi}$ is diagonally-aware:

\begin{proposition}\label{prop:diag-augmentation}
    Let $\psi$ be an RPE. The diagonal augmentation $D^{\psi}$ is diagonally-aware.
\end{proposition}

\paragraph{Asymmetric RPEs.}
Most popular RPEs for undirected graphs are symmetric. However, for directed graphs, it is natural to consider asymmetric RPEs such as the adjacency matrix, the directed distance matrix, the Laplacian matrix, or directed random walk matrices~\cite{geisler2023transformers}.


We identify the following special type of (possibly asymmetric) RPEs, which will be useful later:

\begin{definition}[Pseudo-symmetric RPEs]
    Let $\psi$ be an RPE valued in $\R^{k}$ for either directed or undirected graphs. Let $f:\R^k\to\R^k$ be any injective function such that $f\circ f=\mathrm{Id}$. Then, $\psi$ is \emph{\textbf{($\boldsymbol{f}$-)pseudo-symmetric}} if $\psi_{G}(u,v)=f(\psi_{G}(v,u))$ for any graph $G$ and any vertices $u,v\in V_G$.
\end{definition}

Examples of $f$ satisfying $f \circ f = \mathrm{Id}$ include the identify map, the reflection map, and any coordinate switching map.
%
Obviously, any symmetric RPE $\psi$ is $\mathrm{Id}$-pseudo-symmetric. Any skew-symmetric matrix is pseudo-symmetric for $f(x) = -x$.
If we consider complex-valued RPEs, then any Hermitian matrix is pseudo-symmetric for $f(x) = \overline{x}$, the
complex conjugate map.

Although the adjacency or the Laplacian matrix of a {\bf directed} graph can also serve as RPEs, they are in general not pseudo-symmetric. In order to guarantee pseudo-symmetry, given any RPE $\psi$, we can define an augmented RPE $S^\psi:=(\psi,\psi^T)$, i.e., for any graph $G$, $S^\psi_G:=(\psi_G,\psi^T_G)$.
\vspace{-0.5cm}
\begin{lemma}[Pseudo-symmetric augmentation]\label{lem:ps-augmentation}
    Let $\psi$ be any RPE. Define $f:\R^{2k}\to\R^{2k}$ by letting $(x,y)\mapsto (y,x)$ for any $x,y\in\R^k$. Then, $S^\psi$ is $f$-pseudo-symmetric.
\end{lemma}

\section{Comparison of APE-GT and RPE-GT}\label{sec:equivalence between PEs}

\begin{figure*}[htbp!]
    \centering

    \begin{tikzcd}
        {\text{APE-GT}} &&&& {\text{APEs}} \\ \\ \\
        {\text{RPE-GT}} && {\text{RPE-augWL}} && {\text{RPE-2-WL}}
        \arrow[
        "\tiny{\begin{tabular}{c} \textbf{DeepSets} \\ \textbf{\Cref{lem:APE-to-RPE-WL}} \end{tabular}}"{description}, tail reversed, from=1-5, to=4-5, curve={height=-30pt}, dashed
        ]
        \arrow[
        "{{\text{\textbf{\Cref{lem:APE_transformer_equal_APE}}}}}"', tail reversed, from=1-5, to=1-1
        ]
        \arrow[
        "{{\text{\textbf{\Cref{thm:wl-equals-2-wl}}}}}", tail reversed, from=4-5, to=4-3
        ]
        \arrow[
        "{{\text{\Cref{lem:RPE_transformers_equal_w_WL}}}}", tail reversed, from=4-3, to=4-1
        ]
        \arrow[
        "\tiny{\begin{tabular}{c} \textbf{EGNs} \\ \textbf{~\Cref{lem:IGN}} \end{tabular}}"{description}, from=4-5, to=1-5, curve={height=-30pt}, dashed, tail reversed, two heads
        ]
        \arrow[
        "{{\text{\textbf{\Cref{thm:main-APE-to-RPE}}}}}"{description}, shift left=5, color={rgb,255:red,214;green,92;blue,92}, curve={height=-18pt}, from=1-1, to=4-1
        ]
        \arrow[
        "{{\text{\textbf{\Cref{thm:main-RPE-to-APE}}}}}"{description}, shift left=5, color={rgb,255:red,214;green,92;blue,92}, curve={height=-18pt}, from=4-1, to=1-1
        ]
    \end{tikzcd}

    \caption{\textbf{Illustration of main results.} Arrows denote non-decreasing in distinguishing power. Our main results are the two red arrows on the left.
    The proofs of the two theorems are illustrated by other parts of the diagram. Our contributions are in \textbf{bold}. The two-head arrow from RPE-2-WL to APEs indicates that the non-decreasing property only holds for unfeatured graphs.}
\label{fig:equivalence-diagram}
\end{figure*}
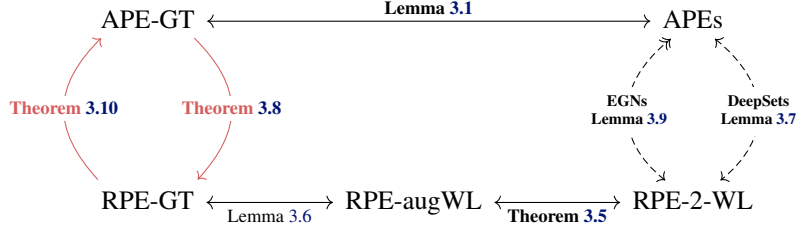

In this section, we establish the equivalence of the distinguishing power of APE and RPE GTs for distinguishing graphs without node features. This is done by showing that any RPE can be mapped to an APE without changing its distinguishing power and vice versa.
Our strategy is to relate the distinguishing power of APE and RPEs, as well as the maps between the two, to variants of the Weisfeiler-Lehman tests. This approach generalizes the method of~\citet{zhang2023rethinking} for studying the expressive power of the Graphormer-GD. Our main results are summarized in \Cref{fig:equivalence-diagram}.

However, for the case of graphs with node features, we prove that our map from RPEs to APEs can result in a strict decrease in the distinguishing power of the corresponding GTs. In this way, our results suggest that (1) the two types of positional encodings are theoretically equivalent for solving graph isomorphism problems but
(2) RPE GTs may be better suited for tasks where graph node features are involved.

Our theoretical results are tested empirically in \Cref{sec:experiments}.


\subsection{Distinguishing Power: PEs vs. Graph Transformers}
\label{sec:distinguishing-power}


%
Positional encodings provide a way of distinguishing different, non-isomorphic graphs. In their most basic form, two graphs can be distinguished by a positional encoding if it takes different values on the two graphs. For graphs with node features, a positional encoding can distinguish two graphs if it takes different values on nodes with different features. Specifically, two featured graphs $(G,X_G)$ and $(H,X_H)$ are \textit{\textbf{indistinguishable}} by an APE $\phi$ if the concatenation of their features with their images by the APE $\phi$ are the same multiset, i.e., $\multiset{ (X_G(v), \phi_{G}(v)) :v\in V_G} = \multiset{ (X_H(v),\phi_{H}(x)):x\in V_H}$, and $G$ and $H$ are indistinguishable by an RPE $\psi$ if $\multiset{(X_G(u),X_G(v),\psi_{G}(u,v)) : (u,v)\in V_G\times V_G} = \multiset{(X_H(x),X_H(y),\psi_{H}(x,y) ): (x,y)\in V_H\times V_H }$. Likewise, two graphs are \textit{\textbf{indistinguishable by a graph transformer}} if their outputs have the same multiset of node features.
\par
Next, we compare the distinguishing power of positional encodings and GTs using these positional encodings.


\paragraph{APE vs. APE-GT.}


Interestingly, APE-GTs are only as strong as their initial APEs in distinguishing power. This was previously noted by \citet[Section 2.1]{muller2023attending}, although we include a short proof (which follows from the permutation-equivariance of transformers) in \Cref{sec:ape_ape_gt} for completeness.


\begin{restatable}[Equivalence of APEs and APE-GT]{lemma}{apeapegt}
\label{lem:APE_transformer_equal_APE}
    Any two graphs $(G,X_G)$ and $(H,X_H)$ are indistinguishable by an APE $\phi$ iff $(G,X_G)$ and $(H,X_H)$ are indistinguishable by all $\phi$-APE-GTs.
\end{restatable}

\paragraph{RPE vs. RPE-GT.}
\label{sec:weighted_wl_tests}
An analogous result does \textbf{not} hold for RPEs and RPE-GTs, as using an RPE in an RPE-GT can increase its distinguishing power; see \Cref{sec:rpe_vs_rpe_wl_example} for an example. In fact, one can obtain a stronger notion of distinguishing power of RPEs via a variation of the Weisfeiler-Lehman (WL) graph isomorphism test~\citep{weisfeiler1968reduction}. This new notion of indistinguishability for RPEs will be equivalent to the indistinguishability of RPE-GTs.


We first recall the standard WL algorithm,  which iteratively assigns colors to the vertices of a graph and utilizes the multisets of node colors to distinguish non-isomorphic graphs. The color updating process is defined explicitly as follows:

\begin{definition}
\label{def:classical_wl}
Let $G$ be a graph. The \textit{\textbf{Weisfeiler-Lehman (WL) algorithm}} assigns a coloring to the vertices $\wlcolor{G}{t}:V_G\to C_t$ for each $t\geq 0$ defined by
\vspace{-0.2cm}
\begin{gather*}
\wlcolor{G}{0}(v) = 1,  \tag{1 denotes an arbitrary constant} \\
\wlcolor{G}{t+1}(v) = (\wlcolor{G}{t}(v),\multiset{\wlcolor{G}{t}(u) : \{v,u\}\in E_G }).
\end{gather*}
\end{definition}

However, to understand the distinguishing power of RPE transformers, we need to consider the following augmented WL test for graphs with node features and RPEs.

\begin{definition}
\label{def:w-WL}
Let $\psi$ be an RPE. Let $(G,X_G)$ be a featured graph. The \textit{\textbf{Weisfeiler-Lehman algorithm augmented with RPE $\boldsymbol{\psi}$ (RPE-augWL or $\boldsymbol{\psi}$-WL)}} assigns a coloring to the vertices $\chi^{(t)}_{\psi_G}:V_G\to C_t$ for each $t\geq 0$ defined by
\vspace{-0.2cm}
\begin{gather*}
    \chi_{\psi_G}^{(0)}(v) = X_G(v),\\
    \chi_{\psi_G}^{(t+1)}(v) = (\chi^{(t)}_{\psi_G}(v),\multiset{(\chi^{(t)}_{\psi_G}(u),\, \psi_G(v,u)) : u\in V_G }).
\end{gather*}
\end{definition}

\textbf{Note:} For unfeatured graphs, $\psi$-WL is defined by setting $\chi_{\psi_G}^{(0)}(v) = 1$ (where 1 is an arbitrary constant.) We adopt this convention for all later constructions.
\par
Note that the RPE-augWL algorithm has a different color updating mechanism from the WL algorithm: RPE-augWL considers {\bf all} nodes in the update stage, whereas WL considers only {\bf neighbors}. However, when the RPE is the adjacency matrix $A$, $A$-WL is equivalent to WL (cf. \Cref{lem:wl-equals-tradition-wl}). Hence, we consider RPE-augWL to be a generalization of the WL algorithm.

The definition of RPE-augWL is inspired by the GD-WL test of \citet{zhang2023rethinking} where the new feature is updated as $\widehat{\chi}_{\psi_G}^{(t+1)}(v) = \multiset{(\widehat{\chi}_{\psi_G}^{(t)}(u),\, \psi_{G}(v,u)) : u\in V }$, where $\psi_G$ is a graph distance. We note that GD-WL is equivalent to our $\psi$-WL; in GD-WL, the previous color of the node can be deduced from the fact that $\psi_G(v,v)=0$ since $\psi_G$ is a graph distance.


The WL algorithm has a family of higher-order generalizations called the $k$-WL algorithms that assign colors to $k$-tuples of nodes. While \Cref{def:w-WL} augments the WL test by using RPEs within the feature updating step, we can also directly use RPEs as the initial colors for 2-WL.

\begin{definition}
    Let $\psi$ be an RPE. Let $(G,X)$ be a featured graph. The \textit{\textbf{2-Weisfeiler-Lehman algorithm with RPE $\boldsymbol{\psi}$ (RPE-2-WL or $\boldsymbol{\psi}$-2-WL)}} assigns a color to each pair of vertices $\chi^{(t)}_{2}:V\times V\to C_t$ for each $t\geq 0$ defined by\footnote{In the literature \cite{chen2020can}, a pair $(u,v)$ is colored differently from ours at the 0-th step by considering both its isomorphism type and its RPE: $\chi_{2,\psi}^{(0)}(u,v) = (1_{u=v},A_{uv},\psi(u,v))$. We ignore the first two terms as one can easily augment $\psi$ with these terms to obtain a new RPE that reflects the isomorphism type if needed; see \Cref{sec:combinatorial_awareness} for a further discussion.
    }
    \vspace{-0.2cm}
    \begin{gather*}
    \chi_{2,\psi_G}^{(0)}(u,v) = (X_G(u), X_G(v), \psi_G(u,v)),\\
    \chi_{2,\psi_G}^{(t+1)}(u,v) = \big(\chi^{(t)}_{2,\psi_G}(u,v), \multiset{ \chi^{(t)}_{2,\psi_G}(u,w) : w\in V }, \\
        \multiset{ \chi^{(t)}_{2,\psi_G}(w,v) : w\in V }\big).
    \end{gather*}
\end{definition}

Because RPEs are dependent on the graph structure and will be the same (up to permutation) for two isomorphic graphs, we can use the RPE-augWL or RPE-2-WL as heuristics to distinguish non-isomorphic graphs. We say two featured graphs $(G,X_G)$ and $(H,X_H)$ are \textit{\textbf{$\boldsymbol{\psi}$-WL indistinguishable}} if $\multiset{ \chi^{(t)}_{\psi_{G}}(v) :v\in V_G } = \multiset{ \chi^{(t)}_{\psi_{H}}(v) :v\in V_H }$ for all $t\geq 0$. We denote this as $\chi_{\psi}(G)=\chi_{\psi}(H)$. Likewise, they are \textit{\textbf{$\boldsymbol{\psi}$-2-WL indistinguishable}}  if $\multiset{ \chi^{(t)}_{2, \psi_{G}}(u,v) : u,v\in V_G } = \multiset{ \chi^{(t)}_{2, \psi_{H}}(u,v) : u,v\in V_H }$  for all $t\geq 0$. We denote this as $\chi_{2,\psi}(G)=\chi_{2,\psi}(H)$.

$\psi$-WL and $\psi$-2-WL indistinguishability can be used as a heuristic test for a special type of graph isomorphism. Two featured graphs $(G,X_G)$ and $(H,X_H)$ are \textit{\textbf{feature isomorphic}} if there is a graph isomorphism $\sigma:V_G\to V_H$ such that $X_G(v) = X_H(\sigma(v))$ for all $v\in V_G$. Two feature isomorphic graphs will be $\psi$-WL and $\psi$-2-WL indistinguishable.

\begin{restatable}{fact}{isomorphicimpliessamewlcolors}
\label{prop:isomorphic_implies_same_wl_colors}
    If $(G,X_G)$ and $(H,X_H)$ are feature isomorphic, then $(G,X_G)$ and $(H,X_H)$ are $\psi$-WL indistinguishable and $\psi$-2-WL indistinguishable for any RPE $\psi$.
\end{restatable}

Interestingly, these two WL tests---RPE-augWL and RPE-2-WL---are  equivalent for pseudo-symmetric RPEs. See \Cref{sec:wl-equals-2-wl} for a proof.

\begin{restatable}[Equivalence of RPE-augWL and RPE-2-WL.]{theorem}{wlequalstwowl}
    \label{thm:wl-equals-2-wl}
        Let $\psi$ be a \emph{pseudo-symmetric} RPE. Then two featured graphs $(G,X_G)$ and $(H,X_H)$ are $\psi$-WL indistinguishable iff they are  $\psi$-2-WL indistinguishable.
\end{restatable}

This result generalizes the well-known fact that the WL and 2-WL tests are equivalent; see \citep{huang2021short}.

Finally, in contrast to Lemma \ref{lem:APE_transformer_equal_APE}, we have the following slight generalization of \citep[Theorem 4]{zhang2023rethinking} who only consider the case where the RPE is a distance function. See \Cref{sec:proof-RPE-transformers-equal-w-WL} for a sketch of the proof.


\begin{restatable}[Equivalence of RPE-augWL and RPE-GT]{lemma}{rpetransformersequalwl}
    \label{lem:RPE_transformers_equal_w_WL}
    Let $\psi$ be a \emph{diagonally-aware} RPE. Let $(G,X_G)$ and $(H,X_H)$ be featured graphs. Then $(G,X_G)$ and $(H,X_H)$ are indistinguisable by the $\psi$-WL test iff $(G,X_G)$ and $(H,X_H)$ are indistinguisable by all $\psi$-RPE-GTs.
\end{restatable}

\subsection{Main Results: APE vs RPE Transformers}\label{sec:main}
In this section, we show that any APE can be turned into an RPE without changing its distinguishing power in terms of transformers and vice versa.

\paragraph*{Mapping APEs to RPEs.}
Let $S$ be any set, and let $\mathrm{Mul}_2(S)$ denote the collection of all 2-element multi-subsets of $S$. For a map $f:S\to\R^{d}$, define the map $h_{f}:\mathrm{Mul}_2(S)\to\R^{d}$ as $h_f(\multiset{x,y}) = f(x)+f(y)$. The map $h_f$ is a special case of the DeepSets network~\citep{zaheer2017deep} as it operates on multisets. Accordingly, we refer to $h_f$ as a DeepSet network in this paper.
\par
We propose the following way of mapping APEs to RPEs: given a function $f$ and an APE $\phi$, we define the RPE $\psi^{f}$ for a graph $G$ as $\psi^{f}_{G}(u,v):=h_f(\multiset{\phi_G(u),\phi_G(v)\}})$.
Our key observation is that given an APE $\phi$, there exists a universal function $f$ whose induced RPE $\psi^f$ distinguishes the same pairs of graphs as $\phi$ (Lemma \ref{lem:APE-to-RPE-WL}). The main \Cref{thm:main-APE-to-RPE} will then follow. Proofs of these results can be found in \Cref{sec:proof-main-APE-to-RPE}.




\begin{restatable}{lemma}{apetorpewl}
\label{lem:APE-to-RPE-WL}
    Let $\phi$ be an APE. Then there is a function $f$ such that any two featured graphs $(G,X_G)$ and $(H,X_H)$ are indistinguishable by $\phi$ iff $(G,X_G)$ and $(H,X_H)$ cannot be distinguished by the $\psi^{f}$-2-WL test.
\end{restatable}


\begin{restatable}{theorem}{mainapetorpe}
\label{thm:main-APE-to-RPE}
     For any APE $\phi$, there exists a function $f$ such that any two graphs $(G,X_G)$ and $(H,X_H)$ are indistinguishable by all $\phi$-APE-GTs iff they are indistinguishable by all $\psi^{f}$-RPE-GTs.
\end{restatable}

\paragraph*{Mapping RPEs to APEs.}
In the other direction, we can transform an RPE $\psi$ to an APE $\phi^g$ by passing $\psi$ through a 2-equivariant graph network (2-EGN) $g$. (See \cite{maron2019invariant} or \Cref{sec:egn} for a definition of 2-EGN.) As 2-EGNs are permutation equivariant, the output $\phi^{g}$ will be a well-defined APE in the sense of~\Cref{def:pe_maps}. This is the approach taken by \citet{lim2022sign} to compute an APE from the eigenspaces of the graph Laplacian. It turns out that mapping RPEs to APEs using 2-EGNs can preserve their ability to distinguish non-featured graphs; however, mapping an RPE to an APE may decrease its ability to distinguish featured graphs. Proofs of the following theorems and example can be found in~\Cref{apx:rpe_to_ape}.

\begin{lemma}[restate=rpetoapewl,name=]\label{lem:IGN}
    Let $\psi$ be a \emph{diagonally-aware} RPE.
    \par
    For any 2-EGN $g$, if $(G,X_G)$ and $(H,X_H)$ are indistinguishable by the $\psi$-2-WL test then $(G,X_G)$ and $(H,X_H)$ are indistinguishable by $\phi^g$.
    \par
    Moreover, for any finite set of unfeatured graphs $\mathcal{G}$, there is a 2-EGN $g$ such that if $G,H\in\mathcal{G}$ are indistinguishable by $\phi^g$, then $G$ and $H$ are indistinguishable by the $\psi$-2-WL test.
\end{lemma}

\begin{restatable}{theorem}{thmmainrpetoape}
\label{thm:main-RPE-to-APE}
     Let $\psi$ be a \emph{diagonally-aware} RPE.
     \par
     For any 2-EGN $g$, if $(G,X_G)$ and $(H,X_H)$ are indistinguishable by all $\psi$-RPE-GTs, then $(G,X_G)$ and $(H,X_H)$ are indistinguishable by $\phi^g$.
     \par
     Moreover, for any finite set of \emph{unfeatured} graphs $\mathcal{G}$, there is a 2-EGN $g$ such that if $G,H\in\mathcal{G}$ are indistinguishable by all $\phi^g$-APE-GTs, then $G$ and $H$ are indistinguishable by all $\psi$-RPE-GTs.
\end{restatable}

\begin{example}[restate = rpetoapecounterexample , name = ]
\label{lem:rpe_to_ape_counterexample}
    There exists an RPE $\psi$ and featured graphs $(G,X_G)$ and $(H,X_H)$ that are distinguishable by $\psi$-2-WL but are indistinguishable by $\phi^{g}$ for any 2-EGN $g$.
\end{example}

\begin{remark}
    Even in the case of unfeatured graphs, there is still a slight asymmetry in the statements of the results for APEs (\Cref{thm:main-APE-to-RPE}) and RPEs (\Cref{thm:main-RPE-to-APE}), as the result for APEs does not have the restriction on ``a finite set'' of graphs. This restriction arises because there is no known universal 2-EGN equal in power to the RPE-2-WL test. This is in part due to the fact that known 2-EGNs would need an unbounded hidden dimension and number of layers to distinguish all pairs of graphs~\citep{maron2019provably}.
\end{remark}

\paragraph*{More expressive RPE to APE maps.}

\Cref{thm:main-RPE-to-APE} shows that RPEs can be transformed into equally-powerful (for unfeatured graphs) APEs using 2-EGNs. However, it is worth noting that RPEs can be turned into APEs \textit{that are even more powerful} than the original RPE using $k$-EGNs for $k> 2$~\citep{maron2019provably}. In the extreme case of $k\in\Omega(n)$, $k$-EGNs can distinguish all pairs of non-isomorphic graphs with $n$ vertices. Such techniques for constructing APEs from RPEs have been previously explored in the literature~\citep{huang2023stability}; for example, Expressive BasisNet~\citep{lim2022sign}. However, the downside to these techniques is that their computational cost increases with their expressive power; $k$-EGNs take $\Theta(n^{k})$ time as they operate on tensors of size $\Theta(n^{k})$. In practice, techniques that map RPEs to APEs typically only use $2$-EGNs because of the computational cost of higher $k$-EGNs.

\paragraph{Restrictions and Implications.}
Our main findings (\Cref{thm:main-APE-to-RPE} and \Cref{thm:main-RPE-to-APE}) show that APE-GTs and RPE-GTs have comparable distinguishing capabilities for unfeatured graphs. However, in practice, learning the DeepSet (to serve as $f$) in \Cref{thm:main-APE-to-RPE} or the 2-EGN (to serve as $g$) in \Cref{thm:main-RPE-to-APE} may not always yield the desired maps. Hence, it remains an interesting open question to figure out whether it is theoretically easier to learn the DeepSet or the 2-EGN. For example, one would benefit in designing architectures from understanding which method requires a lower dimension for hidden layers.
\par
Finally, we would like to point out that converting RPEs to APEs to fit a specific graph transformer architecture is not recommended, as shown by \Cref{lem:rpe_to_ape_counterexample}. See also discussion in \Cref{rem:rspe_to_spe} and \Cref{sec:ZINC} for some empirical validation of this point.

\section{Comparing Graph Transformers with Different Positional Encodings}\label{sec:comparing_different_pes}

\begin{figure*}[htbp!]
    \centering
\begin{tikzcd}
	{\text{Adjacency \& Laplacian}} & {\text{WL \& MPNN}} & {\text{Combinatorially-Aware RPEs}}  && {\text{SPD}} \\ \\
    {\text{SPE}} & {\text{Spectral Kernels}} & {\text{Powers of $L$}} & L^{\dagger} & {\text{RD}} \\
    \arrow["{\text{\Cref{prop:common matrices}}}",curve={height=-15pt}, tail reversed, from=1-1, to=1-2]
    \arrow["{\text{\Cref{thm:combinatorially-aware-rpe-stronger-than-WL}}}", curve={height=-15pt}, from=1-2, to=1-3]
    \arrow["{\text{SPD is combinatorially-aware}}", curve={height=-15pt}, from=1-3, to=1-5, dotted]
    \arrow["{\text{\Cref{thm:resistance_wl_equals_pinv_wl}}}", curve={height=-15pt}, tail reversed, from=3-4, to=3-5]
    \arrow["{\text{\Cref{lem:SPE spectral kernel}}}"', curve={height=15pt}, from=3-2, to=3-1]
    \arrow["{?}"{description}, dotted, from=1-5, to=3-5]
    \arrow["{\text{\Cref{thm:powers_of_laplacian_equal_spectral_kernel}}}", curve={height=15pt}, from=3-2, to=3-3]
    \arrow["{\text{$L^{\dagger}$ is a spectral kernel}}"', curve={height=25pt}, from=3-4, to=3-2]
    \arrow["\text{\Cref{thm:spe_stronger_than_resistance}}"', curve={height=-45pt}, from=3-5, to=3-1, two heads]
\end{tikzcd}
\caption{\textbf{Hierarchy of PEs.} The arrows indicate that the corresponding positional encoding is less strong than the one it points to in terms of distinguishing power. The two-head arrow indicates that the non-decreasing property only holds for unfeatured graphs. The dotted arrow between SPD and RD refers to some partial evidence (cf. \Cref{thm:cut edge}) that RD is stronger than SPD in some respects; however, it is an open question how the two compare as RPEs.}\label{fig:positional_encoding_hierarchy}
\end{figure*}
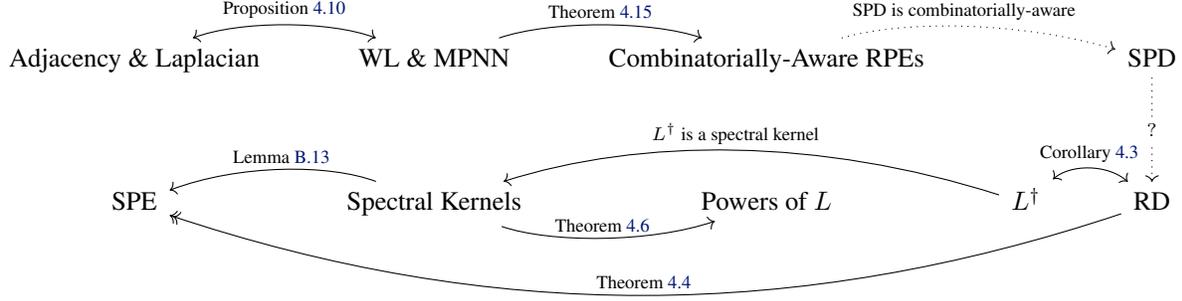

In this section, we compare the distinguishing power of graph transformers with different (absolute or relative) positional encodings. The main results are summarized in \Cref{fig:positional_encoding_hierarchy}.
Intuitively, by \Cref{thm:main-APE-to-RPE}, APE-GTs can be turned into RPE-GTs preserving the distinguishing power. This enables us to compare APE-GTs with RPE-GTs.
When comparing two RPE-GTs, by \Cref{lem:RPE_transformers_equal_w_WL} we simply need to compare the corresponding RPEs via the RPE-augWL test.
For two RPEs $\psi_1$ and $\psi_2$, we say that $\psi_1$-WL is \textit{\textbf{at least as strong as}} $\psi_2$-WL if for any two featured graphs $(G,X_G)$ and $(H,X_H)$ such that $\chi_{\psi_1}(G)=\chi_{\psi_1}(H)$, then we have $\chi_{\psi_2}(G)=\chi_{\psi_2}(H)$. Moreover, we say that $\psi_1$-WL and $\psi_2$-WL are \textit{\textbf{equally strong}} if for any two featured graphs $(G,X_G)$ and $(H,X_H)$, we have that $\chi_{\psi_1}(G)=\chi_{\psi_1}(H)$ iff $\chi_{\psi_2}(G)=\chi_{\psi_2}(H)$.

Some highlights of this section include a comparison of SPE-APE-GTs with RD-RPE-GTs (\Cref{thm:spe_stronger_than_resistance}), identification of a class of RPEs whose corresponding RPE-GTs are at least as strong as MPNNs (\Cref{prop:common matrices,cor:combinatorially-aware-rpe-stronger-than-WL}), and suggesting new RPEs for directed graphs (\Cref{prop:magnetic_laplacian}).

\subsection{Resistance Distance, Spectral Kernels, and SPE}

Many popular APEs (such as BasisNet~\citep{lim2022sign} and SPE~\citep{huang2023stability}) and RPEs (such as RD~\citep{zhang2023rethinking}) are defined using the spectrum of the Laplacian $L$. While BasisNet and SPE have been shown to be equivalent~\citep[Proposition 3.2]{huang2023stability}, it is unclear how either of these compares to RD in terms of transformers. More generally, it is unknown how much information RD-WL carries about the spectrum of $L$.

In this section, we establish that SPE-APE-GTs are stronger than resistance distance RPE-GTs (\Cref{thm:spe_stronger_than_resistance}). As one step to prove this, we show that two graphs are indistinguishable by RD-WL iff they are indistinguishable by $L^\dagger$-WL where $L^\dagger$ is the pseudoinverse of $L$ (\Cref{thm:resistance_wl_equals_pinv_wl}). This result gives a partial answer to the question of how much information RD-WL carries about the Laplacian spectrum. This is weaker than saying that RD-WL indistinguishable graphs have the same $L^\dagger$ up to permutation, but it does show that $L^\dagger$ of these graphs are indistinguishable when combined with RPE-augWL (and thus RPE-GTs).

In fact, our result about RD-WL is a corollary to a stronger result about the equivalence of RPE-augWL tests with spectral distance and spectral kernel as RPEs.


\begin{definition}
    Let $L=\sum_{i=2}^{n}\lambda_{i}z_{i}z_{i}^{T}$ be the spectral decomposition, where $\{\lambda_{2},\ldots,\lambda_{n}\}$ are the non-zero eigenvalues and $\{z_2,\ldots,z_n\}$ is an orthonormal basis of eigenvectors. Let $f:\R^{+}\to\R^{+}$ be a function. The \textit{\textbf{spectral kernel}}~\citep{hammond2011wavelets} corresponding to $f$ is the matrix
        $K^f_{G} = \sum_{i=2}^{n} f(\lambda_i)z_iz_i^T.$
    The \textit{\textbf{spectral distance}} corresponding to $f$ is  defined
    $$
        d^f_{G}(u,v) := \sqrt{K_G^f(u,u)+K_G^f(v,v)-2K_G^f(u,v)}.
    $$
\end{definition}

\begin{example}\label{ex:heat kernel}
    The \textit{\textbf{diffusion distance}} (at time $t$)~\cite{coifman2006diffusion} is the spectral distance $D^{f}$ for the function $f(x) = e^{-tx}$. The \textit{\textbf{heat kernel}} is the corresponding spectral kernel denoted by $H^{(t)}=\sum_{i=2}^{n} e^{-\lambda_it}z_iz_i^T$, which has been proposed as an RPE by~\citet{choromanski22blocktoeplitz}.
\end{example}

\textbf{Note:} Although eigenvectors are not unique due to the choice of bases for eigenspaces, the spectral kernel and spectral distance are unique up to graph isomorphism and so are well-defined RPEs.

Let $D^{f}$ denote the RPE assigning the matrix of spectral distance $d_G^f$ to a graph $G$ and let $K^{f}$ denote the RPE assigning the matrix of spectral kernel $K_G^f$.

\begin{theorem}[restate = spectraldistanceequalskernel , name = ]
\label{thm:spectral distance = kernel}
    Let $f:\R^{+}\to\R^{+}$. $D^{f}$-WL is at least as strong as $K^{f}$-WL.  $K^{f}$-WL with diagonal augmentation is at least as strong as $D^{f}$-WL.
\end{theorem}


\Cref{thm:spectral distance = kernel} is the result of the fact that the spectral distance and kernel are the distance and Gram matrix, respectively, of a point cloud in $\R^{d}$. A variant of the WL-algorithm for general distance matrices of point clouds was studied by \citet{rose2023iterations} as a heuristic for determining if two point clouds are isometric, so our proof may have implications for these algorithms as well. A proof of this theorem can be found in \Cref{sec:gram_wl_and_distance_wl}.

The \textit{\textbf{resistance distance (RD)}} is the \textit{squared} spectral distance of the function $f(x)=x^{-1}$, and the \textit{\textbf{pseudoinverse}} of the Laplacian $L^{\dagger}$ is the corresponding spectral kernel, so~\Cref{thm:spectral distance = kernel} applies to RD and $L^{\dagger}$. Interestingly, RD has additional properties that allow us to drop the assumption of diagonal augmentation. A proof of the following theorem can be found in \Cref{sec:resistance_wl_and_laplacian_wl}.

\begin{corollary}
\label{thm:resistance_wl_equals_pinv_wl}
    RD-WL and $L^{\dagger}$-WL are equally strong.
\end{corollary}

Finally, as an immediate consequence of~\Cref{thm:resistance_wl_equals_pinv_wl} (combined with \Cref{thm:main-RPE-to-APE}), we compare RD-RPE-GT, with existing APE-GPTs using the so-called \emph{Stable and Expressive Positional Encodings (SPE)}~\cite{huang2023stability}. A proof of this theorem can be found in~\Cref{sec:spe_and_resistance}.


\begin{restatable}{theorem}{spestrongerthanresistance}
\label{thm:spe_stronger_than_resistance}
    For unfeatured graphs, diagonally augmented SPE-APE-GT is at least as strong as RD-RPE-GT.
\end{restatable}

\begin{remark}\label{rem:rspe_to_spe}
    In Appendices \ref{sec:CSL} and \ref{sec:experiments_brec}, we empirically validate this theorem by showing that SPE-APE-GTs perform similarly to RD-RPE-GTs on some unfeatured graph classification/isomorphism tasks. However, when node features are involved, the theorem may not hold since the 2-EGN involved in the construction of SPE-APE-GT \textit{may} have weaker distinguishing power than the RD-RPE-GT (cf. \Cref{lem:rpe_to_ape_counterexample}).
    In \Cref{sec:ZINC}, we empirically observe that for graph regression tasks when node features are involved, using RD-RPE-GT can indeed give us an edge in performance over SPE-APE-GT.
\end{remark}

\subsection{Resistance Distance and Cut Edges.}

The first proposed RPE for graph transformers was the \emph{\textbf{shortest-path distance (SPD)}}~\citep{ying2021transformers}. Later, \citet{zhang2023rethinking} proposed to use the stack of RD and SPD as an RPE.
It is natural to wonder how RD-WL compares with SPD-WL. While this remains an open question, we prove that RD-WL is at least as strong as SPD-WL in at least one regard.
\par
\citet{zhang2023rethinking} showed that SPD-WL could distinguish cut edges but not cut vertices in unfeatured graphs. Furthermore, the authors showed that RD-WL could distinguish cut vertices. However, it was not known whether or not RD-WL could distinguish cut edges in a graph. We answer this open question in the affirmative: RD-WL can distinguish cut edges in a graph. See \Cref{thm:cut edge} for a precise theorem statement and proof.
\par
However, our experiments on the BREC dataset~\cite{wang2024brec} suggest (but do not prove) that RD-WL may not be strictly stronger than SPD-WL, as there are pairs of graphs in this dataset that an SPD-RPE-GT was able to learn to distinguish, while we were unsuccessful in training a RD-RPE-GT to distinguish these graphs. See~\Cref{sec:experiments_brec} for details.

\subsection{Powers of Matrices and Spectral Kernels}
\label{sec:powers_of_matrices}

Another common type of RPE is the stack of powers of some matrix associated with the graph; for example, GRIT~\citep{ma2023inductive} uses a stack of powers of the random-walk adjacency matrix of a graph called the \textit{\textbf{relative random-walk positional encoding (RRWP)}}. In this section, we show that using stacks of various matrices as an RPE is at least as strong as any spectral kernel RPEs for any $f:\R^{+}\to\R^{+}$. Proofs for this section can be found in~\Cref{sec:proof-powers-of-laplacian-equal-spectral-kernel}.
\par
We first prove that using sufficiently many powers of the Laplacian is at least as strong as any spectral kernel.

\begin{restatable}{theorem}{powersoflaplacianspectralkernel}
\label{thm:powers_of_laplacian_equal_spectral_kernel}
$(I,L,\ldots,L^{2n-1})$-WL is at least as strong as $K^{f}$-WL on graphs with at most $n$ nodes.
\end{restatable}

\begin{remark}
    While it may seem excessive to consider $2n$ powers of $L$ to match the power of a single spectral kernel, this is not computationally more expensive than computing the spectral kernel in the first place, which requires adding $n$ matrices $f(\lambda_i)z_iz_i^T$ of size $n^{2}$, which takes $O(n^{3})$ time.
\end{remark}

A variant of RRWP that uses stacks powers of the symmetrically-normalized adjacency matrix $\widehat{A} = D^{-1/2}AD^{-1/2}$ is at least as strong as spectral kernels using the symmetrically-normalized Laplacian $\widehat{K}^{f}$ (see \Cref{sec:proof-powers-of-laplacian-equal-spectral-kernel} for definition).

\begin{restatable}{theorem}{stackofadjacencyspectralkernel}
     $(I,\widehat{A},\ldots,\widehat{A}^{2n-1})$-WL is at least as strong as $\widehat{K}^{f}$-WL on graphs with at most $n$ nodes.
\end{restatable}

Finally, we show that stacked heat kernels $H^{(t)}$ (cf. \Cref{ex:heat kernel}) are also at least as strong as any spectral kernel.

\begin{restatable}{theorem}{heatkernelsgeneralkernels}
     $(I,H^{(1)},\ldots,H^{(2n-1)})$-WL is at least as strong as $K^{f}$-WL on graphs with at most $n$ nodes.
\end{restatable}

\subsection{RPEs: Common Matrices for Graphs}
There are multiple common matrices used to characterize graphs, such as the adjacency matrix, Laplacian matrix and their normalized versions. We will examine their corresponding RPEs and compare their distinguishing power.

\paragraph{Undirected graphs.}

For undirected graphs, we show that various common matrices lead to equivalent RPEs.


Let $\widetilde{A}=D^{-1}A$. Let
$\widehat{L}=I-\widehat{A}$ and $\widetilde{L}=I-\widetilde{A}$ denote the symmetrically-normalized and the random-walk graph Laplacians, respectively. These, along with $A$, $L$ and $\widehat{A}$, give rise to six RPEs and hence five corresponding RPE-augWL tests. The following result suggests that in practice, perhaps there is no need to use more than one such type of information with RPE-GTs and one can interchangeably use any of the matrix as RPE.

\begin{restatable}{proposition}{commonmatrices}
\label{prop:common matrices}
    $A,\widehat{A}, \widetilde{A},L,\widehat{L}$ and $\widetilde{L}$  induce equivalent RPE-augWL tests. In particular, all of these RPE-augWL tests are equally strong as the WL test.
\end{restatable}

\paragraph{Directed graphs.}

Directed graphs are sometimes turned into a magnetic graph to obtain a Hermitian matrix called the \textit{\textbf{magnetic Laplacian}} \citep{fanuel2018magnetic}. Let $G$ be a directed graph with adjacency matrix $A$. Let ${A^*}:=(A+A^T)/2$ and let ${D^*}$ denote the corresponding degree matrix. Given a parameter $\alpha$, the magnetic Laplacian is defined
$$
    L^\alpha:={D^*}-T^\alpha\odot{A^*},
$$
where $T^\alpha_{ij}=\exp(\iota\cdot 2\pi\alpha \cdot\mathrm{sgn}(A_{ji}-A_{ij}))$ and $\iota$ is the imaginary unit. As mentioned in \Cref{sec:position_encoding}, $L^\alpha$ can serve as a pseudo-symmetric RPE. However, $L^\alpha$ is no stronger as an RPE than the simpler RPE that concatenates the pseudo-symmetric augmentation $(A,A^T)$ of $A$ (cf. \Cref{lem:ps-augmentation}) with the degree matrix $D^*$, i.e. $(D^*,A,A^T)$. Hence, instead of using magnetic Laplacian as an RPE, one should consider using $(D^*,A,A^T)$ as an RPE instead.

\begin{restatable}{proposition}{magneticlaplacians}
\label{prop:magnetic_laplacian}
    The $(D^*,A,A^T)$-WL test is at least as strong as the $L^\alpha$-WL test for any $\alpha$.
\end{restatable}

\subsection{RPEs: Combinatorial-Awareness and WL}
\label{sec:combinatorial_awareness}

It is an interesting question what information about a graph a transformer can extract from a positional encoding. In \Cref{eq:RPE transformer}, the learnable functions in an RPE-GT $f_i$ are applied elementwise to the RPEs. Moreover, the functions are also applied elementwise to the attention scores. In other words, given an RPE $\psi$, even though certain graph structure can be obtained indirectly through more complicated operations on $\psi$, such hidden information may not be accessible to the transformer because of the limited set of operations it can perform on the RPE. For example, while a graph can be reconstructed from its resistance distance matrix, an RPE-GT may not have access to the edges of a graph as it cannot perform the necessary operations to reconstruct the graph.
\par
One basic piece of information a GT may want to access through an RPE is---of course---the adjacency matrix. We identify a desirable type of RPE where a GT will have this information.

\begin{definition}
An RPE $\psi$ is \emph{\textbf{combinatorially-aware}} if for any two graphs $G$ and $H$ and vertices $u,v\in V_G$ and $x,y\in V_H$, then $\psi_G(u,v)=\psi_H(x,y)$ iff one of the following conditions hold:  (i) $\{u,v\}\in E_G$ and $\{x,y\}\in E_H$ (ii) $\{u,v\}\notin E_G$ and $\{x,y\}\notin E_H$.
\end{definition}

As it turns out, elementwise operations are insufficient for reconstructing the graph structure from the resistance distance, as evidenced by the following example.

\begin{example}\label{ex:SPD}
    The shortest-path distance (SPD) is combinatorially-aware as $\mathrm{SPD}(u,v)=1$ iff $\{u,v\}\in E$. The resistance distance is, however, not combinatorially-aware. See \Cref{sec:RD not ca} for a counterexample.
\end{example}

As in the case of diagonal-awareness, given any RPE, one can always augment it to be combinatorially-aware.

\begin{definition}[Combinatorial Augmentation]\label{ex:aug}
Let $\psi$ be any RPE. We define its \textbf{combinatorial augmentation} $C^\psi$ as follows: for any graph $G$, $C^\psi_G(u,v):=(A_G(u,v),\psi_G(u,v))$ for any $u,v\in V$, where $A_G$ denotes the adjacency matrix of $G$.
\end{definition}

This in fact provides an equivalent characterization of combinatorially-aware RPEs.




\begin{proposition}\label{prop:augmentation}
    An RPE $\psi$ is {combinatorially-aware} if for any $G$, the following condition holds for all $u,v,x,y\in V_G$: $\psi_G(u,v)=\psi_G(x,y)$ iff $C_G^\psi(u,v)=C_G^\psi(x,y)$.
\end{proposition}



The reason combinatorially-aware RPEs are important is because RPE-augWL with combinatorially-aware RPEs are at least as strong as the WL test~\citep{weisfeiler1968reduction}. A proof of this theorem can be found in~\Cref{sec:wl_and_combinatorially_aware_rpes}.

\begin{restatable}{theorem}{combinatoriallyawarerpewl}
\label{thm:combinatorially-aware-rpe-stronger-than-WL}
    Let $\psi$ be a combinatorially-aware RPE. Then, $\psi$-WL is at least as strong as WL.
\end{restatable}

\Cref{thm:combinatorially-aware-rpe-stronger-than-WL} is of interest because the  WL test is equally strong at distinguishing graphs as message passing graph neural networks (MPNNs)~\citep{xu2018powerful}.

\begin{restatable}{corollary}{corocombinatoriallyawarerpewl}
\label{cor:combinatorially-aware-rpe-stronger-than-WL}
    Let $\psi$ be a combinatorially-aware RPE. Then $\psi$-RPE-GTs are at least as strong as MPNNs.
\end{restatable}

\section{Conclusion and Open Questions}

We have established a framework for comparing different types of positional encodings.
Future work can continue this line of work of comparing different specific positional encodings, e.g.~shortest-path vs.~resistance distance or resistance distance vs.~SPE.
Furthermore, our theoretical results provide some suggestions for designing new positional encodings, e.g., one may not want to convert RPEs into APEs as done in the literature.
Finally, it would be interesting to shift the focus of research on positional encodings for graph transformers from coarse-grained distinguishing-power results (what pairs of graphs can GTs distinguish) to more fine-grained results on expressive power, for example, approximation results~\citep{azizian2021expressive} or results using RPE-augWL-inspired distances to study graph transformers \citep{chen2022weisfeiler,chen2023weisfeiler,boker2024fine}.


\section*{Acknowledgements}

This work was supported by the NSF (Grants CCF-2217058 for GM, ZW, and YW;
CCF-2112665 and CCF-2310411 for YW; and CCF-2311180 and CCF-1941086 for MB and AN).
MB would like to thank YW for supporting a visit to UCSD where this research was initiated.

\section*{Impact Statement}
This paper presents work whose goal is to advance the field of Machine Learning. There are many potential societal consequences of potentially more effective graph learning models due to our work, none of which we feel must be specifically highlighted here.




\bibliography{example_paper}
\bibliographystyle{icml2024}

\newpage
\appendix
\onecolumn

\section{Details from \Cref{sec:equivalence between PEs}}

\subsection{Example of RPE vs RPE-augWL}
\label{sec:rpe_vs_rpe_wl_example}

\begin{figure}[htbp!]
    \centering
    \subfigure{
        \centering
        \includegraphics[height=1in]{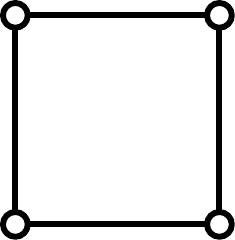}
    }
    \subfigure{
        \centering
        \includegraphics[height=1in]{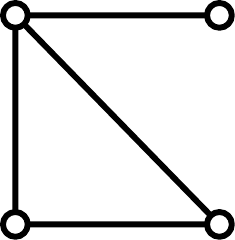}
    }
    \caption{Left: $G$. Right: $H$}
    \label{fig:RPE-WL-example}
\end{figure}

In this section, we show that RPE-augWL can distinguish graphs that are indistinguishable by just their RPE. By~\Cref{lem:RPE_transformers_equal_w_WL}, this also means the graphs are distinguishable by some RPE transformer. Consider the graphs $G$ and $H$ in \Cref{fig:RPE-WL-example} and the RPE $A$ that assigns each graph its adjacency matrix. Both graphs have 4 vertices and 4 edges, so their adjacency matrices $A_G$ and $A_H$ are indistinguishable. However, after one iteration of $A$-WL, all vertices of $G$ will have the same color, while the degree 1 vertex in $H$ will have a different color from all vertices in $G$. Therefore, $G$ and $H$ are distinguishable by $A$-WL.

\subsection{Equivalence of RPE-augWL and RPE-2-WL: Proof of~\Cref{thm:wl-equals-2-wl}}
\label{sec:wl-equals-2-wl}

Our proof of~\Cref{thm:wl-equals-2-wl} is based on the following lemma.

\begin{restatable}{lemma}{wlequalstwowllemma}
    \label{lem:wl-equals-2-wl-lemma}
        Let $\psi$ be a pseudo-symmetric RPE. Let $(G,X_G)$ and $(H,X_H)$ be featured graphs. Let $u,v\in V_G$ and $x,y\in V_H$. Then $\chi^{(t)}_{2,\psi_G}(u,v) = \chi^{(t)}_{2,\psi_H}(x,y)$ iff $\psi_{G}(u,v) = \psi_H(x,y)$, $\chi_{\psi_G}^{(t)}(u) = \chi_{\psi_H}^{(t)}(x)$, and $\chi_{\psi_G}^{(t)}(v) = \chi_{\psi_H}^{(t)}(y)$.
\end{restatable}

\begin{proof}[Proof of \Cref{lem:wl-equals-2-wl-lemma}]
    We prove this by induction on $t$. For $t=0$, this is follows directly from definition as $\chi^{(t)}_{\psi_G}(v)= X_G(v)$ and $\chi^{(0)}_{2,\psi_{G}}(u,v) = (X_G(u), X_G(v),\psi_{G}(u,v) )$.
    \par
    Now inductively suppose this is true for some value $t$; we will show this is true for $t+1$.
    \par
    First assume that $\chi^{(t+1)}_{2,\psi_G}(u,v) = \chi^{(t+1)}_{2,\psi_H}(x,y)$. Then, it follows from the definition that
    \begin{enumerate}
        \item $\chi^{(t)}_{2,\psi_G}(u,v) = \chi^{(t)}_{2,\psi_H}(x,y)$.
        \item $\multiset{ \chi^{(t)}_{2,\psi_G}(u,w) : w\in V_G } = \multiset{ \chi^{(t)}_{2,\psi_H}(x,z) : z\in V_H }$;
        \item $\multiset{ \chi^{(t)}_{2,\psi_G}(w,v) : w\in V_G } = \multiset{ \chi^{(t)}_{2,\psi_H}(z,y) : z\in V_H }$.
    \end{enumerate}
    By induction, Item 1 implies that
    \begin{enumerate}
        \item[4.] $\psi_G(u,v)=\psi_H(x,y)$, and by pseudo-symmetry, $\psi_G(v,u)=\psi_H(y,x)$;
        \item[5.] $\chi_{\psi_G}^{(t)}(u) = \chi_{\psi_H}^{(t)}(x)$ and $\chi_{\psi_G}^{(t)}(v) = \chi_{\psi_H}^{(t)}(y)$;
    \end{enumerate}
    By item 2, there is a bijection $\sigma:V_G\to V_H$ such that $\chi^{(t)}_{2,\psi_G}(u,w) = \chi^{(t)}_{2,\psi_H}(x,\sigma(w))$. The inductive hypothesis implies that $\psi_G(u,w) = \psi_H(x,\sigma(w))$ and $\chi^{(t)}_{\psi_G}(w) = \chi_{\psi_H}^{(t)}(\sigma(w))$. Therefore, the $\psi$-WL colors of $u$ and $x$ at $t+1$ are the same as
    \begin{align*}
        &\chi^{(t+1)}_{\psi_G}(u) = (\chi^{(t)}_{\psi_G}(u),\multiset{\, (\chi_{\psi_G}^{(t)}(w), {\psi_G}(u,w)) : w\in V_G \,}) \\
        &= (\chi^{(t)}_{\psi_H}(x),\multiset{\, (\chi_{\psi_H}^{(t)}(z), {\psi_H}(x,z)) : z\in V_H \,}) = \chi^{(t+1)}_{{\psi_H}}(x).
    \end{align*}
    By item 3, there is a bijection $\sigma':V_G\to V_H$ such that $\chi^{(t)}_{2,\psi_G}(w,v) = \chi^{(t)}_{2,\psi_H}(\sigma'(w),y)$. The inductive hypothesis implies that $\psi_G(w,v) = \psi_H(\sigma'(w),y)$ and $\chi^{(t)}_{\psi_G}(w) = \chi_{\psi_H}^{(t)}(\sigma'(w))$. By pseudo-symmetry, one also has that $\psi_G(v,w) = \psi_H(y,\sigma'(w))$. Therefore, the $\psi$-WL colors of $v$ and $y$ at $t+1$ are the same as
    \begin{align*}
        &\chi^{(t+1)}_{\psi_G}(v) = (\chi^{(t)}_{\psi_G}(v),\multiset{\, (\chi_{\psi_G}^{(t)}(w), {\psi_G}(v,w)) : w\in V_G \,}) \\
        &= (\chi^{(t)}_{\psi_H}(y),\multiset{\, (\chi_{\psi_H}^{(t)}(y), {\psi_H}(y,z)) : z\in V_H \,}) = \chi^{(t+1)}_{{\psi_H}}(y).
    \end{align*}
    \indent Conversely, assume that $\psi_G(u,v) = \psi_H(x,y)$, $\chi_{\psi_G}^{(t+1)}(u) = \chi_{\psi_H}^{(t+1)}(x)$, and $\chi_{\psi_G}^{(t+1)}(v) = \chi_{\psi_H}^{(t+1)}(y)$. By definition of the $\psi$-WL, we have that $\chi_{\psi_G}^{(t)}(u) = \chi_{\psi_H}^{(t)}(x)$ and $\chi_{\psi_G}^{(t)}(v) = \chi_{\psi_H}^{(t)}(y)$. Hence, by induction, we have that $\chi^{(t)}_{2,\psi_G}(u,v) = \chi^{(t)}_{2,\psi_H}(x,y)$.
    Moreover, there exist bijections $\sigma_{ux}:V_G\to V_H$ and $\sigma_{vy}:V_G\to V_H$ such that
    \begin{enumerate}
        \item for any $w\in V_G$, $\chi_{\psi_G}^{(t)}(w) = \chi_{\psi_H}^{(t)}(\sigma_{ux}(w))$ and $\psi_G(u,w)=\psi_H(x,\sigma_{ux}(w))$;
        \item for any $w\in V_G$, $\chi_{\psi_G}^{(t)}(w) = \chi_{\psi_H}^{(t)}(\sigma_{vy}(w))$ and $\psi_G(v,w)=\psi_H(y,\sigma_{vy}(w))$, and by pseudo-symmetry, we also have that $\psi_G(w,v)=\psi_H(\sigma_{vy}(w),y)$.
    \end{enumerate}
    For any given $w\in V_G$, by induction, one has that $\chi_{2,\psi_G}^{(t)}(u,w) = \chi_{2,\psi_H}^{(t)}(x,\sigma_{ux}(w))$ and $\chi_{2,\psi_G}^{(t)}(w,v) = \chi_{2,\psi_H}^{(t)}(\sigma_{vy}(x),y)$. Therefore, $\multiset{ \chi^{(t)}_{2,\psi_G}(u,w) : w\in V_G }=\multiset{ \chi^{(t)}_{2,\psi_H}(x,z) : z\in V_H }$ and $\multiset{ \chi^{(t)}_{2,\psi_G}(w,v) : w\in V_G }=\multiset{ \chi^{(t)}_{2,\psi_H}(z,y) : z\in V_H }$.  This implies that $\chi^{(t+1)}_{2,\psi_G}(u,v) = \chi^{(t+1)}_{2,\psi_H}(x,y)$ which concludes the proof.
\end{proof}

Now we are ready to prove our main theorem in this section.
\wlequalstwowl*


\begin{proof}[Proof of \Cref{thm:wl-equals-2-wl}]
    For the forward direction, we will show that for all $t\geq 0$, if $\multiset{ \chi_{\psi_{G}}^{(t+1)}(u) : u\in V_G } = \multiset{ \chi_{\psi_{H}}^{(t+1)}(v) : v\in V_H }$, then $\multiset{ \chi_{2,\psi_{G}}^{(t)}(u,v) : u,v\in V_G } = \multiset{ \chi_{2,\psi_{H}}^{(t)}(x,y) : x,y\in V_H }$. This implies that if $G$ and $H$ are indistinguishable by the $\psi$-WL, they are indistinguishable by the $\psi$-2-WL.
    \par
    We first prove the base case of $t=0$. As $\multiset{ \chi_{\psi_{G}}^{(1)}(u) : u\in V_G } = \multiset{ \chi_{\psi_{H}}^{(1)}(v) : v\in V_H }$, then there is a bijection $\sigma:V_G\to V_H$ such that $\chi_{\psi_{G}}^{(1)}(u) = \chi_{\psi_{G}}^{(1)}(\sigma(u))$. Moreover, by the definition of $\chi^{(1)}_\psi$, for each $u\in V$, there is a bijection $\sigma_{u}:V_{G}\to V_{H}$ such that $\psi_{G}(u,v) = \psi_{H}(\sigma(u),\sigma_{u}(v))$. Therefore, we can define a bijection $\Phi:V_{G}\times V_{G}\to V_{H}\times V_{H}$ defined $\Phi(u,v) = (\sigma(u), \sigma_{u}(v))$ such that $\psi_{G}(u,v) = \psi_{H}(\Phi(u,v))$. Therefore, $\multiset{ \chi_{2,\psi_{G}}^{(0)}(u,v) : u,v\in V_G } = \multiset{ \chi_{2,\psi_{H}}^{(0)}(x,y) : x,y\in V_H }$.
    \par
    We now prove that for each $t\geq 1$, if $\multiset{ \chi_{\psi_{G}}^{(t+1)}(u) : u\in V_G } = \multiset{ \chi_{\psi_{H}}^{(t+1)}(v) : v\in V_H }$, then $\multiset{ \chi_{2,\psi_{G}}^{(t)}(u,v) : u,v\in V_G } = \multiset{ \chi_{2,\psi_{H}}^{(t)}(x,y) : x,y\in V_H }$.
    Assume that both graphs are assigned $k^{(t+1)}$ colors by $\chi^{(t+1)}$. More specifically, assume there are partitions $V_G=V_{G,1}^{(t+1)}\cup\cdots\cup V_{G,k^{(t+1)}}^{(t+1)}$ and $V_H=V_{H,1}^{(t+1)}\cup\cdots\cup V_{H,k^{(t+1)}}^{(t+1)}$ such that 
    \begin{enumerate}
        \item $|V^{(t+1)}_{G,i}|=|V^{(t+1)}_{H,i}|$;
        \item for any $u\in V^{(t+1)}_{G,i}$ and $v\in V^{(t+1)}_{G,j}$, $\chi_{\psi_G}^{(t+1)}(u) = \chi_{\psi_G}^{(t+1)}(v)$ iff $i=j$;
        \item for any $u\in V_{G,i}^{(t+1)}$ and $x\in V_{H,i}^{(t+1)}$, $\chi_{\psi_G}^{(t+1)}(u) = \chi_{\psi_H}^{(t+1)}(x)$.
    \end{enumerate}
    We can define analogous partitions for $\chi^{(t)}_\psi$.

    \begin{figure}[htbp!]
        \centering
        \includegraphics[height=1.5in]{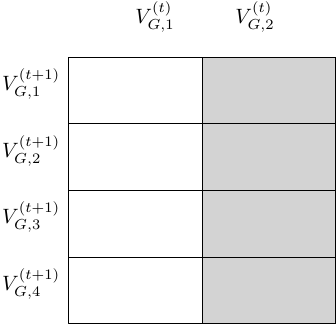}
        \caption{Decomposition of $V_G\times V_G$.}
        \label{fig: decomposition}
    \end{figure}
    \par
    These partitions induce a partition on the set of pairs of vertices $V_G\times V_G=\cup_{1\leq i\leq k^{(t+1)}}\cup_{1\leq j\leq k^{(t)}} V_{G,i}^{(t+1)}\times V_{G,j}^{(t)}$ and $V_H\times V_H=\cup_{1\leq i\leq k^{(t+1)}}\cup_{1\leq j\leq k^{(t)}}V_{H,i}^{(t+1)}\times V_{H,j}^{(t)}$; see \Cref{fig: decomposition} for an illustration of the decomposition.

    We show that $\chi_{2,\psi}^{(t)}$ has the same multiset of colors on $V_{G,i}^{(t+1)}\times V_{G,j}^{(t)}$ and $V_{H,i}^{(t+1)}\times V_{H,j}^{(t)}$. By the definition of the partitions, there exists a bijection $\sigma^{(t+1)}_i:V_{G,i}^{(t+1)}\to V_{H,i}^{(t+1)}$ such that $\chi_{\psi_G}^{(t+1)}(u)=\chi_{\psi_H}^{(t+1)}(\sigma^{(t+1)}_i(u))$ for any $u\in V_{G,i}^{(t+1)}$. By definition of $\psi$-WL, we have that $\chi_{\psi_G}^{(t)}(u)=\chi_{\psi_H}^{(t)}(\sigma^{(t+1)}_i(u))$ for any $u\in V_{G,i}^{(t+1)}$. Furthermore, for any $u\in V_{G,i}^{(t+1)}$,
    \begin{align*}
       \multiset{\, (\chi^{(t)}_{\psi_G}(w), \psi_G(u,w)) : w\in V_G \,}=\multiset{\, (\chi^{(t)}_{\psi_H}(z), \psi_H(\sigma^{(t+1)}_i(u),z)) : z\in V_H \,}.
    \end{align*}

    Now for each $1\leq j\leq k^{(t)}$, consider the sub-multiset $\multiset{\, (\chi^{(t)}_{\psi_G}(w), \psi_G(u,w)) : w\in V_{G,j}^{(t)} \,}$ of $\multiset{\, (\chi^{(t)}_{\psi_G}(w), \psi_G(u,w)) : w\in V_G \,}$,  then we have that
    \[\multiset{\, (\chi^{(t)}_{\psi_G}(w), \psi_G(u,w)) : w\in V_{G,j}^{(t)} \,}=\multiset{\, (\chi^{(t)}_{\psi_H}(z), \psi_H(\sigma_i^{(t+1)}(u),z)) : z\in V_{H,j}^{(t)} \,}.\]
    This is because when $j\neq b$, one has that $\chi_{\psi_G}^{(t)}(w)\neq \chi_{\psi_H}^{(t)}(z)$ for $w\in V_{G,j}^{(t)}$ and $z\in V_{H,b}^{(t)}$.


    Therefore,
    \begin{align*}
        &\multiset{ (\psi_G(u,w), \chi^{(t)}_{\psi_G}(u),\chi^{(t)}_{\psi_G}(w)):\,w\in V_{G,j}^{(t)} }\\
        =&\multiset{ (\psi_H(\sigma_i^{(t+1)}(u),z), \chi^{(t)}_{\psi_H}(\sigma_i^{(t+1)}(u)),\chi^{(t)}_{\psi_H}(z)):\,z\in V_{H,j}^{(t)} }.
    \end{align*}
    Therefore, by \Cref{lem:wl-equals-2-wl-lemma}, one has that
    \[\multiset{ \chi_{2,\psi_G}^{(t)}(u,w):\,w\in V_{G,j}^{(t)} }=\multiset{ \chi_{2,\psi_H}^{(t)}(\sigma_i^{(t+1)}(u),z):\,z\in V_{H,j}^{(t)} }.\]
    Since $\sigma_i^{(t+1)}$ is bijective, we conclude that
    \[
        \multiset{ \chi_{2,\psi_G}^{(t)}(u,v):\,u\in V_{G,i}^{(t+1)},v\in V_{G,j}^{(t)} }=\multiset{ \chi_{2,\psi_H}^{(t)}(x,y):\,x\in V_{H,i}^{(t+1)},y\in V_{H,j}^{(t)} }.
    \]
    Taking the multiset union over $1\leq i\leq k^{(t+1)}$ and $1\leq j\leq k^{(t)}$, this implies that $\multiset{ \chi_{2,\psi_{G}}^{(t)}(u,v) : u,v\in V_G } = \multiset{ \chi_{2,\psi_{H}}^{(t)}(x,y) : x,y\in V_H }$.
    \par
    Conversely, we will show that for $t\geq 0$, if $\multiset{ \chi_{2,\psi_G}^{(t)}(u,v):u,v\in V_G }=\multiset{ \chi_{2,\psi_H}^{(t)}(x,y):x,y\in V_H }$, then $\multiset{ \chi^{(t)}_{\psi_G}(u):u\in V_G }=\multiset{ \chi^{(t)}_{\psi_H}(x):x\in V_H }$.
    \par
    For the base case of $t=0$, we have that $\multiset{ (X_G(u),X_G(v),\psi_G(u,v)):u,v\in V_G }=\multiset{ (X_H(x),X_H(y),\psi_H(x,y)):x,y\in V_H }$. Therefore, $\multiset{ (X_G(u),X_G(v)):u,v\in V_G }=\multiset{ (X_H(x),X_H(y)):x,y\in V_H }$ and hence $\multiset{ X_G(u)=:u\in V_G }^2=\multiset{ X_H(x):x\in V_H }^2$. Now we apply the following fact:
    \begin{fact}\label{fact:multiset-square}
        Let $A$ and $B$ be two finite multisets. If $A^2=B^2$, then $A=B$.
    \end{fact}
    Then, we have that $\multiset{ \chi^{(0)}_{\psi_G}(u)=X_G(u):u\in V_G }=\multiset{ \chi^{(0)}_{\psi_H}(x)=X_H(x):x\in V_H }$.

    Now assume that for $k\geq 1$, we have that
    \[
        \multiset{ \chi_{2,\psi_G}^{(t)}(u,v):u,v\in V_G }=\multiset{ \chi_{2,\psi_H}^{(t)}(x,y):x,y\in V_H }.
    \]
    \Cref{lem:wl-equals-2-wl-lemma} implies that
    \[\multiset{ (\psi_G(u,v),  \chi^{(t)}_{\psi_G}(u),\chi^{(t)}_{\psi_G}(v)):u,v\in V_G }=\multiset{ (\psi_H(x,y), \chi^{(t)}_{\psi_H}(x),\chi^{(t)}_{\psi_H}(y)):x,y\in V_H }.\]
    Dropping the $\phi$ term, one has that
    \[\multiset{ (\chi^{(t)}_{\psi_G}(u),\chi^{(t)}_{\psi_G}(v)):u,v\in V_G }=\multiset{ (\chi^{(t)}_{\psi_H}(x),\chi^{(t)}_{\psi_H}(y)):x,y\in V_H }.\]
    This implies that the following Cartesian products are equal:
    \[\multiset{ \chi^{(t)}_{\psi_G}(u):u\in V_G }^2=\multiset{ \chi^{(t)}_{\psi_H}(x):x\in V_H }^2.\]

    Then, by \Cref{fact:multiset-square} again, we have that
    \[
        \multiset{ \chi^{(t)}_{\psi_G}(u):u\in V_G }=\multiset{ \chi^{(t)}_{\psi_H}(x):x\in V_H }.
    \]
    This means that $G$ and $H$ are indistinguishable by $\psi$-WL.

    We now end this proof by proving \Cref{fact:multiset-square}.
    \begin{proof}[Proof of \Cref{fact:multiset-square}]
        We let $C:=A\cup B$ denote the set of all distinct elements in $A$ and $B$. Let $C=\{c_1,\ldots,c_n\}$ and assume that $A$ contains $r^A_i$ many $c_i$ and $B$ contains $r^B_i$ many $c_i$ for each $i=1,\ldots,n$. Then, $A^2$ contains $r^A_i r^A_j$ many $(c_i,c_j)$ and $B^2$ contains $r^B_i r^B_j$ many $(c_i,c_j)$. Therefore, $A^2=B^2$ implies that $r^A_i r^A_j=r^B_i r^B_j$ for all $i,j=1,\ldots,n$. For each $i=1,\ldots,n$, we let $j=i$ and hence $(r^A_i)^2 =(r^B_i)^2$. Since we are dealing with nonnegative integers, we must have that $r^A_i=r^B_i$. Therefore, $A=B$.
    \end{proof}
\end{proof}

\subsection{RPE-2-WL vs 2-EGN}\label{sec:egn}

Equivariant or invariant graph networks (EGNs and IGNs)~\cite{maron2019invariant} are another important type of graph neural networks. These models will be important when trying to understand the distinguishing power of positional encodings later.

\begin{definition}[EGNs]
    A \textbf{$k$\textsuperscript{th} order Equivariant Graph Network ($k$-EGN)} is a function $F:\R^{n^k\times d_0}\to\R^{n^s\times d_o}$ for $s\leq k$ of the form
    \[
        F= h\circ m\circ L^{(T)}\circ\sigma\circ\cdots\circ\sigma\circ L^{(1)}
    \]
    where $h$ and $L^{(t)}:\R^{n^{k_{t-1}}\times d_{t-1}}\to \R^{n^{k_t}\times d_t}$ are equivariant linear layers for each $t=1,\ldots,T$ such that $k_t\leq k$, $k_0=k$ and $k_L=s$, $\sigma$ is an activation function, and $m$ is a pointwise MLP. When $s=0$, the resulting $k$-EGN is also called a \textbf{$k$\textsuperscript{th} order Invariant Graph Network ($k$-IGN)}.
\end{definition}

\textbf{Note}: Some implementations of 2-IGNs and 2-EGNs (e.g.~\citep{maron2019invariant}) stack the adjacency matrix to the input $W$. In this paper, unless otherwise stated, the adjacency matrix is \textbf{not} stacked to the input.
\par
EGNs are permutation equivariant in the sense that permutation of the order of input tensors will result in a corresponding permutation of the order of output tensors. Likewise, IGNs are permutation invariant as any permutation of the input tensors results in the same output.
\par
We are interested in 2-EGNs and 2-IGNs since they are closely related to the 2-WL graph isomorphism test. Let $\psi$ be an RPE valued in $\R^k$ and let $g:\R^{n^2\times k}\to\R^{n\times l}$ be any 2-EGN. Recall that for any graph $G$ with $n$ vertices, the RPE $\psi_G$ can be represented as a tensor in $\R^{n^2\times k}$. If $X_G\in\R^{n\times l}$ are node features for $G$, let $\row(X_G),\col(X_G)\in\R^{n^2\times l}$ be the tensors where each row or column respectively are the node feature $X_G$, e.g. $\row(X_G)[i,:,:]= X_G$ and $\col(X_G)[:,j,:] = X_G$ for all $1\leq i,j\leq n$.
\par
We now show that 2-EGNs with the input $[\psi_G,\row(X_G),\col(X_G)]\in\R^{n^2\times(k+2l)}$ have the same distinguishing powers as the $\psi$-2-WL test.

\begin{restatable}[Equivalence of RPE-2-WL and 2-EGN]{proposition}{twowltwoign}\label{thm:2-WL-2-IGN}
    Let $\psi$ be a \emph{diagonally-aware} RPE valued in $\R^k$. Let $\mathcal{G}$ be any finite set of featured graphs. Then there exists a 2-EGN $g:\R^{n^2\times k}\to\R^{n\times l}$ such that any $(G,X_G),(H,X_H)\in \mathcal{G}$ can be distinguished by $g$ with the input $[\psi,\row(X),\col(X)]$ iff $G$ and $H$ can be distinguished by the $\psi$-2-WL test.
\end{restatable}

This result is a generalization of existing results showing that 2-IGNs of the form $g:\R^{n^2\times k}\to\R^m$ are equivalent to the 2-WL test~\citep{maron2019provably, chen2020can}.

\begin{proof}[Proof of \Cref{thm:2-WL-2-IGN}]

We will prove the forward and backward direction of this statement separately in the following lemmas. However, we emphasize that only~\Cref{lem:2_wl_stronger_than_2_IGN} needs the assumption of diagonal-awareness. Moreover, we emphasize that the ``proofs'' of both of these lemmas are just slightly adapting existing results connecting 2-IGNs to the 2-WL test.

\begin{lemma}[2-EGNs are stronger than RPE-2-WL]
\label{lem:2_ign_stronger_than_2_wl}
Let $\psi$ be an RPE valued in $\R^k$.
Let $\mathcal{G}$ be a finite set of graphs. Then, there exists a 2-EGN $g:\R^{n^2\times (k+2l)}\to\R^{n\times m}$ such that if any $(G,X_G),(H,X_H)\in\mathcal{G}$ can be distinguished by $\psi$-2-WL test then $\multiset{g([\psi_G,\row(X_G),\col(X_G)])(u):u\in V_G}\neq\multiset{g([\psi_H,\row(X_H),\col(X_H)])(v):v\in V_H}$.
\end{lemma}
\begin{proof}
    This follows from the proof by~\citet[Theorem 1]{maron2019provably} showing that there exists a $2$-EGN that can distinguish any two graphs from a finite set of graphs well as the classical 2-WL algorithm. (While their statement of Theorem 1 is only for two graphs, their proof explains why this can be generalized to a finite set of graphs.) This lemma shows how this result can be generalized to the RPE-2-WL test.
    \par
    Their proof is broken into two steps. First, the authors show that given a tensor $W\in\R^{n^{2}\times(e+1)}$ of a diagonal feature matrix stacked with the adjacency matrix, there is a 2-equivariant linear layer $f:\R^{n^{2}\times(e+1)}\to \R^{n^{2}\times 4\times(e+2)}$ such that, for each pair $(u,v)$, the vector $f(W)[u,v,:]$ corresponds to the isomorphism type of $(u,v)$, which is the initial color of $(u,v)$ in the 2-WL test. Second, they show that there is a $d$ layers 2-EGN $f=L^{(d)}\circ\sigma\circ\cdots\circ\sigma\circ L^{(1)}:\R^{n^{2}\times k}\to \R^{n^2\times l}$ that can implement $d$ iterations of the 2-WL algorithm; that is to say, $f(A_G)(u,v) = f(A_H)(u,v)$ iff $\wlcolor{ }{d}(u,v) = \wlcolor{ }{d}(x,y)$. Finally, they show that there exists an MLP $m$, together with a final layer $h$ of summing over all entries, such that $g=h\circ m\circ f$ implements the 2-WL test, i.e. $g(A_G) = g(A_H)$ iff $\chi(G) = \chi(H)$.
    \par
    The difference between the 2-WL test and the RPE-2-WL test lies in the initialization step. In the classical 2-WL test, a pair $(u,v)$ is colored with its isomorphism type, while in the RPE-2-WL test, the initial color is $(X_G(u), X_G(v), \psi_G(u,v))$. If the input to the 2-EGN is the RPE matrix concatenated with the node features (i.e. the initial coloring), we can skip the initialization step and just perform the update steps of the RPE-2-WL through a 2-EGN $f=L^{(d)}\circ\sigma\circ\cdots\circ\sigma\circ L^{(1)}:\R^{n^{2}\times k}\to \R^{n^2\times l}$. Finally, instead of taking the sum of all entries, we can apply the operation of row sum (which is equivariant) to the output of $f$ to obtain our target 2-EGN $g:\R^{n^2\times k}\to\R^{n\times l}$.
\end{proof}

\begin{lemma}[RPE-2-WL is stronger than 2-EGNs]
\label{lem:2_wl_stronger_than_2_IGN}
Let $\psi$ be a \emph{diagonally-aware} RPE. Let $(G,X_G)$ and $(H,X_H)$ be graphs. If $G$ and $H$ cannot be distinguished by $\psi$-2-WL, then $(G,X_G)$ and $(H,X_H)$ cannot be distinguished by any 2-EGN with input $[\psi,\row(X),\col(X)]$.
\end{lemma}
\begin{proof}
    This follows directly from the proof showing that if two graphs cannot be distinguished by the 2-WL, they cannot be distinguished by 2-EGNs~\citep[Lemma 4.2]{chen2020can}. The only reason we cannot apply this proof directly is that it relies on the fact that the initial colors to the 2-WL---and therefore all subsequent colors---are diagonally-aware~\citep[Lemma D.1]{chen2020can}, which is not necessarily the case for the initial colors of RPE-2-WL for general RPEs. However, if the RPE is diagonally-aware, then our lemma can be proven following the proof of \citet[Lemma 4.2]{chen2020can}.
\end{proof}

\end{proof}

\subparagraph{Importance of diagonal-awareness.}
In the proof of \Cref{thm:2-WL-2-IGN} (specifically in \Cref{lem:2_wl_stronger_than_2_IGN}), we show that RPE-2-WL is at least as strong as 2-EGNs if the RPE is diagonally-aware.  We emphasize that if we drop the assumption of diagonal-awareness, the lemma does not hold. Therefore, 2-EGN can be strictly stronger than th RPE-2-WL test for non-diagonally-aware RPEs, as shown by the following simple example.
\par
Consider the two graphs on two vertices $G$ and $H$ where $G$ has an edge and $H$ does not. (In this example, the graphs have no node features.) Consider the (artificial) RPE $\psi$ that assigns
$$
\psi_{G} = \begin{bmatrix}
    1 & 0 \\
    0 & 1
\end{bmatrix},\,
\psi_{H} = \begin{bmatrix}
    0 & 1 \\
    1 & 0
\end{bmatrix}.
$$
Note that $\psi$ is \textit{not} diagonally-aware, as the values 0 and 1 appear on both the diagonal and off-diagonal elements.
\par
The $\psi$-2-WL colorings for $t=0$ are
$$
\chi_{2,\psi}^{(0)}(G) = \begin{bmatrix}
    1 & 0 \\
    0 & 1
\end{bmatrix},\,
\chi_{2,\psi}^{(0)}(H) = \begin{bmatrix}
    0 & 1 \\
    1 & 0
\end{bmatrix}.
$$
These are, in fact, already stable $\psi$-2-WL colorings, as another iteration of $\psi$-2-WL gives the same partition. This means further iterations of $\psi$-2-WL won't be able to distinguish these two graphs. For example,
$$
\wlcolor{2,\psi}{1}(G)=
\begin{bmatrix}
    (1, (\multiset{0,1},  \multiset{0,1})) & (0, (\multiset{0,1},\multiset{0,1})) \\
    (0, (\multiset{0,1},  \multiset{0,1})) & (1, (\multiset{0,1},  \multiset{0,1}))
\end{bmatrix}
$$
$$
\wlcolor{2,\psi}{1}(H) =
\begin{bmatrix}
    (0, (\multiset{0,1},  \multiset{0,1})) & (1, (\multiset{0,1},\multiset{0,1})) \\
    (1, (\multiset{0,1},  \multiset{0,1})) & (0, (\multiset{0,1},  \multiset{0,1}))
\end{bmatrix},\,
$$
As $\wlcolor{2,\psi}{0}(G)$ and $\wlcolor{2,\psi}{0}(H)$ have the same multiset of colors $\multiset{0,0,1,1}$, they are indistinguishable by $\psi$-2-WL. However, there is a 2-IGN that distinguishes these graphs. Specifically, one invariant linear matrix is $B:\R^{n\times n}\to\R$ that maps $B(W) = \sum_{v\in V} W_{vv}$. (In the language of the original IGN paper~\citep{maron2019invariant}, this corresponds to the partition $\gamma=\multiset{1,2}$.) Then $B(\psi_{G}) = 2\neq B(\psi_{H})=0$, so this 2-IGN distinguishes $\psi_{G}$ and $\psi_{H}$.

\subsection{Equivalence of APEs and APE-GTs: Proof of \Cref{lem:APE_transformer_equal_APE}}
\label{sec:ape_ape_gt}

\apeapegt*

\begin{proof}
    Let $T$ be any transformer, and let $(G,X_G)$ and $(H,X_H)$ indistinguishable by $\phi$. The initial input to $T$ are $\hat{X_{G}} = [X_G|\phi_{G}]$ and $\hat{X}_{H} = [X_H|\phi_{H}]$. As $G$ and $H$ are indistinguishable by $\phi$, then there is a permutation $P$ such that $P\hat{X}_G = \hat{X}_H$; however, as transformers are permutation-equivariant (\Cref{thm:transformers_permutation_equivariant}), then $T(\hat{X}_H) = T(P\hat{X}_G) = PT(\hat{X}_G)$, so the theorem follows.
    \par
    Now we prove the inverse. Consider the transformer $T$ with all weight matrices set to the zero matrix. The transformer $T$ is exactly the identity map. Thus, if $G$ and $H$ are distinguishable by $\phi$, then they are distinguishable by $T$.
\end{proof}

\subsection{Isomorphic implies WL Indistinguishable: Proof of \Cref{prop:isomorphic_implies_same_wl_colors}}\label{sec:proof-isomorphic-implies-same-wl-colors}

\isomorphicimpliessamewlcolors*

\begin{proof}
    Let $(G,X_G)$ and $(H,X_H)$ be feature isomorphic graphs connected by isomorphism $\sigma:V_G\to V_H$. Let $\psi$ be an RPE. We will prove the stronger statement that $\wlcolor{\psi_G}{t}(u)=\wlcolor{\psi_H}{t}(\sigma(u))$ for all $v\in V_G$ and $t\geq 0$. For the base case of $t=0$, this is trivial as $\wlcolor{\psi_G}{0}(u) = X_G(u) = X_H(\sigma(u)) = \wlcolor{\psi_H}{0}(\sigma(u))$ for all $v\in V_G$ by the definition of feature isomorphism.
    \par
    Now suppose the statement is true for $t$, and we will show it is true for $t+1$. Let $u\in V_G$. Consider the colors $\wlcolor{\psi_G}{t+1}(u) = (\wlcolor{\psi_G}{t}(u), \multiset{(\wlcolor{\psi_G}{t}(v), \psi_{G}(u,v)) : v\in V_G})$ and $\wlcolor{\psi_H}{t+1}(\sigma(u)) = (\wlcolor{\psi_H}{t}(\sigma(u)), \multiset{(\wlcolor{\psi_H}{t}(v), \psi_H(\sigma(u), v)) : v\in V_H})$. By induction, we know that the first coordinates $\wlcolor{\psi_G}{t}(u) = \wlcolor{\psi_H}{t}(\sigma(u))$. Moreover, by the inductive hypothesis and the definition of RPE, we know that $(\wlcolor{\psi_G}{t}(v), \psi_{G}(u,v)) = (\wlcolor{\psi_H}{t}(\sigma(v)), \psi_{H}(\sigma(u),\sigma(v)))$, which implies the second coordinates are equals. Therefore, the theorem holds.
    \par
    A similar argument shows that $\wlcolor{2,\psi_G}{t}(u,v)=\wlcolor{2,\psi_H}{t}(\sigma(u),\sigma(v))$ for all $u,v\in V_G$ and $t\geq 0$.
\end{proof}


\subsection{Equivalence of RPE-augWL and RPE-GTs: Proof of \Cref{lem:RPE_transformers_equal_w_WL}}\label{sec:proof-RPE-transformers-equal-w-WL}

\rpetransformersequalwl*
\begin{proof}
     We will denote the input at the $l$-th layer of the RPE-GT for graph $G$ as $X_G^{(l)}\in\R^{|V_G|\times\ell}$, and for each node $u\in V_G$, its feature is denoted by $X_G^{(l)}(u)$. Note $X_G^{(0)}(u)=X_G(u)$. We adopt similar notations for $H$.
    \par
    For the forward direction, it suffices to show that if $G$ and $H$ are indistinguishable by the $\psi$-WL test, they are indistinguishable by any $\psi$-RPE-GT. We will prove the stronger statement that if the multiset of colors $\multiset{\wlcolor{\psi_G}{t}(u) : v\in V_G} = \multiset{\wlcolor{\psi_H}{t}(x) : x\in V_H}$, then any bijection $\sigma^{(t)}:V_G\to V_H$ such that $\wlcolor{\psi_G}{t}(u)=\wlcolor{\psi_H}{t}(\sigma^{(t)}(u))$ will satisfy the following property: $X_G^{(t)}(u)=X_H^{(t)}(\sigma^{(t)}(u))$ for all $u\in V_G$, and hence $\multiset{X_G^{(t)}(u) : u\in V_G}=\multiset{X_H^{(t)}(x) : x\in V_H}$. We will prove this by induction on $t$.
    \par
    The key observation is that the only part of the transformer architecture that is not applied separately to the node features $X^{(t)}(u)$ is the multiplication by the attention matrix $A^{(l,h)}(X^{(l)})X^{(l)}$, meaning for all other parts of the transformer, if the inputs $X_G^{(t)}(u)=X_H^{(t)}(x)$ are equal, then the outputs {at $u$ and $x$} are equal.
    \par
    For the base case of $t=0$, the statement holds as the initial node features equals the $\psi$-WL colors for $t=0$, i.e.~ $\multiset{\wlcolor{\psi_G}{0}(v):v\in V_G} = \multiset{X_G^{(0)}(v) : v\in V_G}$.
    \par
    Now suppose the theorem inductively holds for some $t$. Assume that $\multiset{\wlcolor{\psi_G}{t+1}(u) : v\in V_G} = \multiset{\wlcolor{\psi_H}{t+1}(x) : x\in V_H}$ and let $\sigma^{(t+1)}:V_G\to V_H$ denote any bijection such that $\wlcolor{\psi_G}{t+1}(u)=\wlcolor{\psi_H}{t+1}(\sigma^{(t+1)}(u))$ for any $u\in V_G$. By the definition of $\chi_{\psi}$ and the inductive hypothesis, we have that $\multiset{\wlcolor{\psi_G}{t}(u) : v\in V_G} = \multiset{\wlcolor{\psi_H}{t}(x) : x\in V_H}$, $\wlcolor{\psi_G}{t}(u)=\wlcolor{\psi_H}{t+1}(\sigma^{(t)}(u))$ and $X_G^{(t)}(u)=X_H^{(t)}(\sigma^{(t+1)}(u))$ for all $u\in V_G$.
    \par
    Now for any $v\in V_G$ and $y\in V_H$ such that $X_G^{(t)}(v)=X_H^{(t)}(y)$, we have that
    \begin{equation}\label{eq:transformer equation}
        X_G^{(t)}(u)W_Q(X_G^{(t)}(v)W_K)^{T}=X_H^{(t)}(\sigma^{(t+1)}(u))W_Q(X_H^{(t)}(y)W_K)^{T}
    \end{equation}
    for any matrices $W_Q$ and $W_K$.

    Now, for any $u\in V_G$, $\wlcolor{\psi_G}{t+1}(u)=\wlcolor{\psi_H}{t+1}(\sigma^{(t+1)}(u))$ implies that
    $$\multiset{(\wlcolor{\psi_G}{t}(v), \psi_{G}(u,v)) : v\in V_G} = \multiset{(\wlcolor{\psi_H}{t}(y), \psi_{H}(\sigma^{(t+1)}(u),y)) : y\in V_H}.$$
    Hence,
    $$\multiset{(X_G^{(t)}(v), \psi_{G}(u,v)) : v\in V_G} = \multiset{(X_H^{(t)}(y), \psi_{H}(\sigma^{(t+1)}(u),y)) : y\in V_H}.$$
    Note that for any given value $a$,
    \begin{equation}\label{eq:multiset equation}
        \multiset{\psi_{G}(u,v) : v\in V_G, X_G^{(t)}(v)=a} = \multiset{\psi_{H}(\sigma^{(t+1)}(u),y) : y\in V_H, X_H^{(t)}(y)=a}.
    \end{equation}
    Then, we have that
    \begin{align*}
        &[A^{(t,h)}(X_G^{(t)})X_G^{(t)}](u) = \sum_vA^{(t,h)}(X_G^{(t)})(u,v)X_G^{(t)}(v)\\
        &=\sum_vf_1(\psi_G(u,v))\mathrm{softmax}(X_G^{(t)}(u)W_Q(X_G^{(t)}(\cdot)W_K)^{T}/\sqrt{d_h}+f_2(\psi_G(u,\cdot)))(v)X_G^{(t)}(v)\\
        &=\sum_yf_1(\psi_H(\sigma^{(t+1)}(u),y))\mathrm{softmax}(X_H^{(t)}(\sigma^{(t+1)}(u))W_Q(X_H^{(t)}(\cdot)W_K)^{T}/\sqrt{d_h}+f_2(\psi_H(\sigma^{(t+1)}(u),\cdot)))(y)X_H^{(t)}(y)\\
        &=[A^{(t,h)}(X_H^{(t)})X_H^{(t)}](\sigma^{(t+1)}(u)).
    \end{align*}
    Here the third equation follows from \Cref{eq:transformer equation} and \Cref{eq:multiset equation}.

    It is then easy to check that  $X^{(t+1)}_G(u) = X^{(t+1)}_H(\sigma^{(t+1)}(u))$ for any $u\in V_G$ and this concludes the proof.

    \par
    For the other direction,  the proof  follows from the one of \citet[Theorem 4]{zhang2023rethinking}; More precisely, \citet{zhang2023rethinking} are dealing with the case when the RPE $\psi$ is taking value from a finite set $D$ for graphs with bounded sizes. Furthermore, their RPE $\psi$ is a distance (or a stack of distance functions) so that $\psi_G(u,v)=0$ iff $u=v\in V_G$ for any graph $G$. In this way, elements in $D$ can be written as $0=r_0<r_1<\cdots r_H$. Note that, for one step of $\psi$-WL, the second argument can be decomposed by
    $$\multiset{(\chi^{(t)}_{\psi_G}(u),\, \psi_G(v,u)) : u\in V_G }=\bigcup_{r_i}\multiset{({\chi}^{(t)}_{\psi_G}(u),r_i):\, \psi_G(v,u)=r_i, u\in V_G }.$$
    \par
    Then, \citet{zhang2023rethinking} explicitly constructed $H+1$ different RPE attention heads $A_{\mathrm{RPE}}^{h}$ which can simulate $\multiset{({\chi}^{(t)}_{\psi_G}(u),r_i):\, \psi_G(v,u)=r_h,u\in V_G }$ in an injective manner for each $h=0,\ldots,H$. In this way, they argue that by concatenating these heads, one can simulate their GD-WL test where each step is updated by $\widehat{\chi}^{(t+1)}_{\psi_G}(v) = \multiset{(\widehat{\chi}^{(t)}_{\psi_G}(u),\, \psi_G(v,u)) : u\in V_G }$. But as mentioned already in \Cref{sec:distinguishing-power}, this GD-WL test is equivalent to the $\psi$-WL test when $\psi$ is a graph distance (or a stack of graph distances), so their one layer RPE transformer can simulate one step of $\psi$-WL in an injective manner.
    \par
    Recall that GD-WL is equivalent to $\psi$-WL due to the fact that $\psi$ takes value 0 on diagonal entries and $\psi$ takes non-zero values on off-diagonal entries. This is the reason why one can drop the first argument in $(\chi^{(t)}_{\psi_G}(v),\multiset{(\chi^{(t)}_{\psi_G}(u),\, \psi_G(v,u)) : u\in V })$ without losing distinguishing power. Hence, the same argument in the proof of \citep[Theorem 4]{zhang2023rethinking} can be applied to the setting of diagonally-aware RPEs due to the following simple result.

    \begin{lemma}
        When $\psi$ is a diagonally-aware RPE, then $\psi$-WL is equivalent to the following variant of $\psi$-WL:
        \begin{gather*}
            \widehat{\chi}_{\psi_G}^{(0)}(v) = X_G(v),\\
            \widehat{\chi}_{\psi_G}^{(t+1)}(v) = \multiset{(\widehat{\chi}^{(t)}_{\psi_G}(u),\, \psi_G(v,u)) : u\in V }.
            \end{gather*}
    \end{lemma}

    Now, to prove our theorem, we simply let $D:=\{\psi_G(u,v):u,v\in V_G\}\cup\{\psi_H(x,y):x,y\in V_H\}$ and apply the same argument as in the proof of \citep[Theorem 4]{zhang2023rethinking} to conclude the proof.
\end{proof}

\subsection{Mapping APEs to RPEs: Proof of \Cref{lem:APE-to-RPE-WL,thm:main-APE-to-RPE}}\label{sec:proof-main-APE-to-RPE}


We first recall a key lemma from \citet{xu2018powerful}. For any set $S$, we let $\mathrm{Mul}(S)$ denote the collection of all finite multi-subsets of $S$ and recall that $\mathrm{Mul}_2(S)$ denotes the collection of all multi-subsets of size 2 in $S$.

\begin{lemma}[{\citep[Lemma 5]{xu2018powerful}}]\label{lem:countable-for all multiset}
    If $S$ is countable, then there exists a function $f:S\to \R$ such that the map $h_f:\mathrm{Mul}(S)\to \R$ defined by sending $h_f(X)=\sum_{x\in X}f(x)$ is injective.
\end{lemma}

As a direct consequence, we immediately have the following corollary.

\begin{corollary}\label{lem:countable}
    If $S$ is countable, then there exists a function $f:S\to \R$ such that the map $h_f:\mathrm{Mul}_2(S)\to \R$ defined by sending $h_f(\multiset{x,y})=f(x)+f(y)$ is injective.
\end{corollary}

\apetorpewl*

\begin{proof}
    Let $S=\cup_{G\in\mathcal{G}}\im\phi_{G}$ where $\mathcal{G}$ is the set of all finite graphs. We first note that $S$ is countable as the set of all finite graphs is countable. Let $f$ be any function such that $h_f$ that is injective on $S$; such a function must exist by~\Cref{lem:countable} as $S$ is countable. Assume that $G$ and $H$ are indistinguishable by the $\psi^{f}$-2-WL test. This, in particular, implies that the multisets of initial colors are the same:
    \[
        \multiset{(X_G(u), X_G(v),\psi^f_G(v,w)):u,v\in V_G }=  \multiset{ (X_H(x), X_H(y),\psi^f_H(x,y)):x,y\in V_H }.
    \]
    By the definition of $\psi^{f}$ and the fact that $h_f$ is injective, one has that
    \[
        \multiset{ (X_G(u), X_G(v),\phi_G(u),\phi_G(v)):u,v\in V_G }= \multiset{ (X_H(x), X_H(y),\phi_H(x),\phi_H(y)):x,y\in V_H }.
    \]
    Hence, the following two Cartesian products of multisets are equal:
    \[
        \multiset{ (X_G(u),\phi_G(u)):u\in V_G }^2= \multiset{ (X_H(x),\phi_H(x)):x\in V_H }^2.
    \]
    By \Cref{fact:multiset-square}, we have that $\multiset{(X_G(u),\phi_G(u)):u\in V_G}=\multiset{ (X_H(x),\phi_H(x)):x\in V_H }$.
    \par
    Conversely, assume that $(G,X_G)$ and $(H,X_H)$ are indistinguishable by $\phi$. Then $\multiset{(X_G(u),\phi_G(u):u\in V_G}=\multiset{ (X_H(x),\phi_H(x)):x\in V_H }$. Let $\sigma:V_{G}\to V_{H}$ be a bijection such that $\phi_{G}(v) = \phi_{H}(\sigma(v))$ and $X_G(v)=X_H(\sigma(v))$ for all $v\in V_{G}$. We will prove that $(G,X_G)$ and $(H,X_H)$ are indistinguishable by $\psi^{f}$-WL for any function $f$. In fact, we will prove the stronger statement that $\wlcolor{\psi^{f}_G}{t}(v) = \wlcolor{\psi^{f}_H}{t}(\sigma(v))$ for all $t\geq 0$. This is stronger than proving they are $\psi^{f}$-WL indistinguishable as we use the \textit{same} bijection for all iterations $t$.
    \par
    We prove this by induction on the iteration $t$. For $t=0$, this is true as the initial $\psi^{f}$-WL colors $\wlcolor{\psi^{f}_H}{t}(v)$ are just the node features $X_G(v)$, which are implied by the definition of indistinguishability by an APE. Now suppose this is true for some iteration $t$; we will show it is true for $t+1$. Let $v\in V_{G}$. The inductive hypothesis implies that the first element of the colors are equal $\wlcolor{\psi^{f}_G}{t}(v) = \wlcolor{\psi^{f}_H}{t}(\sigma(v))$. Moreover, if we consider the second element of the colors, we find that
    \begin{align*}
        \multiset{ (\wlcolor{\psi^{f}_G}{t}(w), \psi_{G}^{f}(v, w)) : w\in V_{G} } =& \multiset{ (\wlcolor{\psi^{f}_G}{t}(w), f(\phi_{G}(v))+ f(\phi_{G}(w))) : w\in V_{G} } \tag{defintion of $\psi_{G}^{f}$} \\
        =& \multiset{ (\wlcolor{\psi^{f}_G}{t}(w), f(\phi_{H}(\sigma(v)))+ f(\phi_{H}(\sigma(w)))) : w\in V_{G} } \tag{defintion of $\sigma$} \\
        =& \multiset{ (\wlcolor{\psi^{f}_H}{t}(\sigma(w)), f(\phi_{H}(\sigma(v))+f(\phi_{H}(\sigma(w))))) : w\in V_{G} } \tag{inductive hypothesis} \\
        =& \multiset{ (\wlcolor{\psi^{f}_H}{t}(\sigma(w)), \psi_{H}^{f}(\sigma(v),\sigma(w))) : w\in V_{G} } \tag{defintion of $\psi_{H}^{f}$} \\
        =& \multiset{ (\wlcolor{\psi^{f}_H}{t}(w), \psi_{H}^{f}(\sigma(v), w)) : w\in V_{H} }. \tag{$\sigma$ is a bijection}
    \end{align*}
    Therefore, $\wlcolor{\psi^{f}_G}{t+1}(v) = \wlcolor{\psi^{f}_H}{t+1}(\sigma(v))$ as claimed.
\end{proof}

\mainapetorpe*

\begin{proof}
    By \Cref{lem:APE_transformer_equal_APE}, two graphs $(G,X_G)$ and $(H,X_H)$ are indistinguishable by all APE transformers with APE $\phi$ iff they are indistinguishable by $\phi$. By \Cref{lem:APE-to-RPE-WL}, this is equivalent to being indistinguishable by $\psi^{f}$-2-WL. Hence, by \Cref{lem:RPE_transformers_equal_w_WL}, this is equivalent to being indistinguishable by any RPE transformer with RPE $\psi^{f}$.
\end{proof}

\subsection{Mapping RPE to APE: Proof of \Cref{thm:main-RPE-to-APE}}
\label{apx:rpe_to_ape}

\rpetoapewl*

\begin{proof}
    Let $(G,X_G)$ and $(H,X_H)$ be featured graphs that are indistinguishable by $\psi$-2-WL.
    By~\Cref{lem:2_wl_stronger_than_2_IGN}, this means that for any $2$-EGN $g$ of the appropriate dimension, $g([\psi_G,\row(X_G),\col(X_G)]) = g([\psi_G,\row(X_G),\col(X_G)])$. In particular, for any 2-EGN $h$ of appropriate dimension, there is a 2-EGN $g_{h}$ such that $g_h([\psi_{G},X_GX_G^{T}]) = [h(\psi_{G}), X_G^{2}]$; the 2-EGN $g_{h}$ runs $h$ on the first columns of $[\psi_G,\row(X_G),\col(X_G)]$ and then copies the mean of the rows of $\row(X_{G})$ on the remaining columns. The output of $g_h$ is exactly the APE $\psi^{h}_{G}$ concatenated to the node features $X_G$.
    \par
    For the backward direction of the case of graphs without node features, this follows directly from~\Cref{lem:2_ign_stronger_than_2_wl}.
\end{proof}

\thmmainrpetoape*

\begin{proof}
    By \Cref{lem:RPE_transformers_equal_w_WL} and \Cref{thm:wl-equals-2-wl}, $G$ and $H$ are indistinguisable by an RPE transformer with RPE $\psi$ iff $G$ and $H$ are indistinguisable by the $\psi$-2-WL test. By \Cref{lem:IGN}, this is equivalent to $G$ and $H$ being indistinguishable by $\phi^g$. By \Cref{lem:APE_transformer_equal_APE}, this is equivalent to $G$ and $H$ being indistinguishable by any APE transformer with APE $\phi^{g}$.
\end{proof}

\begin{figure}[htbp!]
    \centering
    \subfigure{
        \centering
        \includegraphics[height=1.5in]{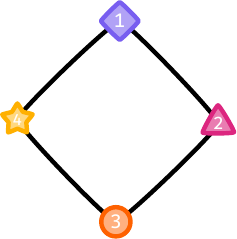}
    }
    \subfigure{
        \centering
        \includegraphics[height=1.5in]{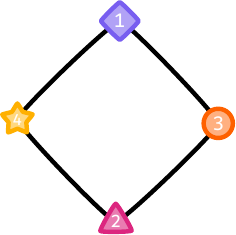}
    }
    \caption{Left: $(G,X_G)$. Right: $(H,X_H)$.}
    \label{fig:node_feature_counterexample}
\end{figure}

We now show that RPEs are strictly stronger than APEs derived from RPEs for graphs with node features.

\rpetoapecounterexample*

\begin{proof}
    The graph $G$ and $H$ are both be the cycle graph on 4-vertices $C_4$. The RPE $\psi$ is the adjacency matrix. The node features are $X_G=[1,2,3,4]$ and $X_H=[1,3,2,4]$. See~\Cref{fig:node_feature_counterexample}.
    \par
    We first prove these graphs are distinguishable by $\psi$-WL. If we consider the first $\psi$-WL colors $\wlcolor{\psi}{1}$, we see that $G$ has a node of color $(1,\multiset{(2,1),(3,0),(4,1)})$; however, $H$ has no node of that color, as the only node with feature $1$ has color $(1,\multiset{(3,1),(2,0),(4,1)})$.
    \par
    We next prove that are indistinguishable by the APE $\phi^{g}$ for any 2-EGN $g$. The first thing to observe is that the APE $\phi^{g}$ will assign the same value to each node in $G$ and $H$; this follows from the definition of APE and the fact that for any two nodes $u,v\in V_G$, there is a graph isomorphism $\sigma:V_G\to V_G$ such that $\sigma(u)=v$; likewise for any two nodes $x,y\in V_H$ or nodes $u\in V_G$ and $x\in V_H$. Therefore, as $(G,X_G)$ and $(H,X_H)$ have the same multiset of node features and $\phi^{g}$ has uniform entries, they cannot be distinguished by $\phi^{g}$.
\end{proof}

\section{Details from \Cref{sec:comparing_different_pes}}

\subsection{A Lemma for RPE-augWL Composition.}
\label{sec:useful_properties_of_wl_tests}

In this section, we will repeatedly use the following useful property of composing the RPE-augWL test. Let $\psi$ be an RPE valued in $\R^d$, and let $f:\R^{d}\to\R$ be a function. Consider the composition $f\circ \psi$ defined by $(f\circ\psi)_G(u,v):=f(\psi_G(u,v))$ for any graph $G$. One naturally wonders about the distinguishing power of $\psi$-WL versus $f\circ\psi$-WL. In fact, $\psi$-WL will always be at least as strong as $f\circ\psi$-WL, and under mild conditions on the input, there will exist a function $f$ such that $\psi$-WL and $f\circ\psi$-WL are equally strong.

\begin{proposition}\label{prop:concatenate rpe}
    Let $\psi$ be an RPE. Let $G$ and $H$ be graphs. Let $u\in V_G$ and $v\in V_H$.
    \begin{enumerate}
        \item If $\chi^{(t)}_{\psi_{G}}(u) = \chi^{(t)}_{\psi_{H}}(v)$, then for any function $f:\R^{d}\to\R$, $\chi^{(t)}_{f\circ\psi_G}(u) = \chi^{(t)}_{f\circ\psi_{H}}(v)$. In particular, $\psi$-WL is as strong as $f\circ\psi$-WL.
        \item For any set $\Omega\subset\R^{d}$ and any function $f:\R^{d}\to\R$ injective on $\Omega$, if $\im\psi_G\cup\im\psi_H\subset\Omega$ and $\chi^{(t)}_{\psi_{G}}(u) \neq \chi^{(t)}_{\psi_{H}}(v)$, then $\chi^{(t)}_{f\circ\psi_G}(u) \neq \chi^{(t)}_{f\circ\psi_{H}}(v)$. In particular, $f\circ\psi$-WL is as strong as $\psi$-WL.
    \end{enumerate}
\end{proposition}

\begin{proof}
For item 1, it follows from
\citep[Theorem 3]{zhu2023structural}.

For item 2, we choose the function $f$ to be any function that is injective on $\Omega$; such a function exists as $\Omega$ are finite. The lemma follows by induction of the iteration $t$.
\end{proof}
\begin{remark}
    It is quite common that RPEs will be drawn from a finite set $\Omega\subset\R^{d}$. In particular, if we consider any RPE defined on unweighted graphs with $n$ vertices that is purely a function of its topology, e.g., the shortest-path distance or Laplacian eigenspace projection, then because there are only finitely many unweighted graphs on $n$ vertices, then the image of the PE in $\R^{d}$ will be finite.
\end{remark}

\subsection{Proof of \Cref{thm:spectral distance = kernel} and \Cref{thm:resistance_wl_equals_pinv_wl}}
\label{sec:gram_wl_and_distance_wl}

\paragraph*{Euclidean Distance and Gram Matrices.}
Let $P=\{p_1,\ldots,p_n\}\subset\R^{d}$ be a point cloud. The \textit{\textbf{Gram matrix}} of $P$ is denoted $G_{P}\in\R^{n\times n}$ where $(G_{P})_{ij} = p_i^{T}p_j$. The \textit{\textbf{distance matrix}} of $P$ is denoted $D_{P}\in\R^{n\times n}$ where $(D_{P})_{ij} = \| p_i-p_j \|$. A point cloud $P$ is \textit{\textbf{centered}} if $\sum_{i=1}^{n} p_i=0$.

\begin{restatable}{theorem}{gramwlequalsdistancewl}
\label{thm:gram_wl_equals_distance_wl}
    Assume that for any graph $G$, there is an embedding of the nodes $P_G:V\to\R^{d}$ such that $\im P_G$ is centered.
    Let $\psi_{d}$ and $\psi_{g}$ be RPEs that are the distance and Gram matrix of the point set $\im P_{G}$, respectively. Then, $\psi_{d}$-WL is at least as strong as are as $\psi_{g}$-WL. Conversely, $\psi_g$-WL with diagonal augmentation is at least as strong as $\psi_d$.
\end{restatable}

The proof is based on the following key lemma. We define the \textit{\textbf{squared distance matrix}} of $P$ to be the matrix $D^{2}_{P}\in\R^{n\times n}$ where $(D^{2}_{P})_{ij} = \| p_i-p_j \|^{2}$.

\begin{lemma}
\label{lem:gram_wl_equals_squared_distance_wl}
    Let $\psi_{d}$ and $\psi_{g}$ be RPEs such that, for any graph $G$, $\psi_{d,G}$ and $\psi_{g,G}$ are the squared distance and Gram matrix of a centered point set $P_{G}\subset\R^{d}$. Then $\psi_{d}$-WL is at least as strong as are as $\psi_{g}$-WL. Conversely, $\psi_g$-WL with diagonal augmentation is at least as strong as $\psi_d$.
\end{lemma}

\begin{proof}[Proof of \Cref{thm:gram_wl_equals_distance_wl} assuming \Cref{lem:gram_wl_equals_squared_distance_wl}]
    As the square root function is injective, then by \Cref{prop:concatenate rpe}, $G$ and $H$ are indistinguishable by $\psi_d$-WL if and only if $D_P^{2}$ and $D_{Q}^{2}$ are indistinguishable by $\psi_g$-WL.
\end{proof}

Now we proceed to prove \Cref{lem:gram_wl_equals_squared_distance_wl}.
\begin{lemma}
\label{lem:diagonal_gram_matrix}
    Let $P=\{p_1,\ldots,p_n\}\subset\R^{d}$ be a centered point cloud. For any $1\leq k\leq n$, $(G_P)_{ii} = \frac{1}{n}\sum_{i=1}^{n}(D_P^{2})_{ik} - \frac{1}{2n^2}\sum_{i=1}^{n}\sum_{j=1}^{n} (D_P^{2})_{ij}$.
\end{lemma}
\begin{proof}
    Observe that $(D_p^{2})_{ij} = p_i^{T}p_i + p_j^{T}p_j - 2p_i^{T}p_{j}$. Therefore, the first term simplifies to
    \begin{align*}
        \frac{1}{n}\sum_{i=1}^{n}(D_P^{2})_{ik} =& p_k^{T}p_{k} + \frac{1}{n}\sum_{i=1}^{n} p_{i}^{T}p_i - \frac{2}{n}\sum_{i=1}^{n}p_i^{T}p_{k} \\
        =& p_k^{T}p_{k} + \frac{1}{n}\sum_{i=1}^{n} p_{i}^{T}p_i \tag{as $\sum_{i=1}^{n} p_i=0$}.
    \end{align*}
    The second term simplifies to:
    \begin{align*}
        \frac{1}{2n^2}\sum_{i=1}^{n}\sum_{j=1}^{n} (D_P^{2})_{ij} =& \frac{1}{2n^2}\sum_{i=1}^{n}\sum_{j=1}^{n} \left(p_i^{T}p_i + p_j^{T}p_j - 2p_i^{T}p_{j}\right)\\
        =& \frac{1}{2n^2}\sum_{i=1}^{n}\left( np_i^{T}p_i + \sum_{j=1}^{n}p_j^{T}p_j - \sum_{j=1}^{n}2p_i^{T}p_{j}\right)\\
        =& \frac{1}{2n^2}\sum_{i=1}^{n}\left( np_i^{T}p_i + \sum_{j=1}^{n}p_j^{T}p_j\right) \tag{as $\sum_{i=1}^{n} p_i=0$}\\
        =& \frac{1}{2n^2}\sum_{i=1}^{n} 2np_i^{T}p_i\\
        =& \frac{1}{n}\sum_{i=1}^{n} p_i^{T}p_i.
    \end{align*}
    Therefore,
    \begin{align*}
        \frac{1}{n}\sum_{i=1}^{n}(D_P^{2})_{ik} - \frac{1}{2n^2}\sum_{i=1}^{n}\sum_{j=1}^{n} (D_P^{2})_{ij} =& p_{k}^{T}p_k + \frac{1}{n}\sum_{i=1}^{n} p_{i}^{T}p_i - \frac{1}{n}\sum_{i=1}^{n} p_{i}^{T}p_i \\
        =& p_k^{T}p_k \\
        =& (G_P)_{kk}. \qedhere
    \end{align*}
\end{proof}

\begin{lemma}
\label{lem:two_iterations_squared_distance_wl}
    Let $\psi_{d}$ and $\psi_{g}$ be RPEs such that, for any graph $G$, $\psi_{d,G}$ and $\psi_{g,G}$ are the squared distance and Gram matrix of a centered point set $P_{G}\subset\R^{d}$. Let $(G,X_G)$ and $(H,X_H)$ be featured graphs. If $\wlcolor{\psi_{d,G}}{2}(u) = \wlcolor{\psi_{d,H}}{2}(x)$, then $(G_{P_G})_{uu} = (G_{P_{H}})_{xx}$.
\end{lemma}
\begin{proof}
    We will prove this using \Cref{lem:diagonal_gram_matrix}. First, the fact that $\wlcolor{\psi_{d,G}}{1}(u) = \wlcolor{\psi_{d,H}}{1}(x)$ implies $\multiset{(D^{2}_{P_G})_{uv} : v\in V_G} = \multiset{(D^{2}_{P_H})_{xy} : y\in V_H}$, which in particular implies that $\frac{1}{n}\sum_{v\in V_{G}}(D_{P_G}^{2})_{uv} = \frac{1}{n}\sum_{y\in V_H}(D_{P_H}^{2})_{xy}$. Second, as $\wlcolor{\psi_{d,G}}{2}(u) = \wlcolor{\psi_{d,H}}{2}(x)$, then as the second coordinate of these colors are equal, there is a bijection $\sigma:V_G\to V_H$ such that $\wlcolor{\psi_{d,G}}{1}(v) = \wlcolor{\psi_{d,H}}{1}(\sigma(v))$ for all $v\in V_G$. In particular, this implies that $\multiset{(D_{P_G}^{2})_{vw} : w\in V_G} = \multiset{(D_{P_H}^{2})_{\sigma(v)z} : x\in V_H}$ for all $v\in V_G$. Therefore, $P_{G}$ and $P_{H}$ have the same multiset of distances, and $\frac{1}{2n^2}\sum_{v\in V_G}\sum_{w\in V_G} (D_{P_G}^{2})_{vw}=\frac{1}{2n^2}\sum_{y\in V_H}\sum_{z\in V_H} (D_{P_H}^{2})_{yz}$. These two observations and \Cref{lem:diagonal_gram_matrix} imply the lemma.
\end{proof}

\begin{proposition}[Forward Direction of \Cref{lem:gram_wl_equals_squared_distance_wl}]
\label{thm:squared_distance_wl_stronger_than_gram_wl}
    Let $\psi_{d}$ and $\psi_{g}$ be RPEs such that, for any graph $G$, $\psi_{d,G}$ and $\psi_{g,G}$ are the squared distance and Gram matrix of a centered point set $P_{G}\subset\R^{d}$. Let $(G,X_G)$ and $(H,X_H)$ be featured graphs. Let $u\in V_G$ and $x\in V_H$. If $\wlcolor{\psi_{d,G}}{t+2}(u) = \wlcolor{\psi_{d,H}}{t+2}(x)$, then $\wlcolor{\psi_{g,G}}{t}(u) = \wlcolor{\psi_{g,H}}{t}(x)$. Moreover, if $(G,X_G)$ and $(H,X_H)$ are indistinguishable by $\psi_{d}$-WL, then $(G,X_G)$ and $(H,X_H)$ are indistinguishable by $\psi_g$-WL.
\end{proposition}
\begin{proof}
    We prove this by induction on $t$. For $t=0$, this is obvious as $\wlcolor{\psi_{d,G}}{2}(u) = \wlcolor{\psi_{d,H}}{2}(x)$ implies that $\wlcolor{\psi_{g,G}}{0}(u) = X_G(u) = X_H(v) = \wlcolor{\psi_{g,H}}{0}(x)$
    \par
    Now suppose the proposition is true for a natural number $t$, and let $u\in V_G$ and $x\in V_H$ such that $\wlcolor{\psi_{d,G}}{t+3}(u) = \wlcolor{\psi_{d,H}}{t+3}(x)$. By induction, we know that $\wlcolor{\psi_{d,G}}{t+2}(u) = \wlcolor{\psi_{d,H}}{t+2}(x)$ implies $\wlcolor{\psi_{g,G}}{t}(u) = \wlcolor{\psi_{g,H}}{t}(x)$, so the first coordinates of $\wlcolor{\psi_{g,G}}{t+1}(u)$ and $\wlcolor{\psi_{g,H}}{t+1}(x)$ are equal.
    \par
    We now show the second coordinates are equal. The fact that $\wlcolor{\psi_{d,G}}{t+3}(u) = \wlcolor{\psi_{d,H}}{t+3}(x)$ implies there is a bijection $\sigma:V_G\to V_H$ such that $(\wlcolor{\psi_{d,G}}{t+2}(v),\, (D_{P_G}^2)_{uv}) = (\wlcolor{\psi_{d,H}}{t+2}(\sigma(v)),\, (D_{P_H}^2)_{x\sigma(v)})$ for all $v\in V_G$. As $\wlcolor{\psi_{d,G}}{t+2}(v)=\wlcolor{\psi_{d,H}}{t+2}(\sigma(v))$ implies $(G_{P_G})_{vv} = (G_{P_H})_{\sigma(v)\sigma(v)}$ by \Cref{lem:two_iterations_squared_distance_wl}, then we know that
    $$
        (G_{P_G})_{uv} = \frac{1}{2}\left((G_{P_G})_{vv} + (G_{P_G})_{uu} - (D^{2}_{P_G})_{uv}\right) = \frac{1}{2}\left((G_{P_H})_{\sigma(v)\sigma(v)} + (G_{P_H})_{xx} - (D_{P_H}^{2})_{x\sigma(v)}\right) = (G_{P_Q})_{x\sigma(v)}.
    $$
    Moreover, $\wlcolor{\psi_{d,G}}{t+2}(v)=\wlcolor{\psi_{d,H}}{t+2}(\sigma(v))$ implies $\wlcolor{\psi_{g,G}}{t}(v)=\wlcolor{\psi_{g,H}}{t}(\sigma(v))$ by the inductive hypothesis. Therefore, as $\sigma$ is a bijection, the second coordinates of $\wlcolor{\psi_{g,G}}{t+1}(u)$ and $\wlcolor{\psi_{g,H}}{t+1}(x)$ are equal:
    \begin{equation*}
        \multiset{(\wlcolor{\psi_{g,G}}{t}(v),\, (G_{P_G})_{uv}) : v\in V_G} = \multiset{(\wlcolor{\psi_{g,H}}{t}(y),\, (G_{P_H})_{xy}) : y\in V_H}. \qedhere
    \end{equation*}
\end{proof}

\begin{proposition}[Backward Direction of \Cref{lem:gram_wl_equals_squared_distance_wl}]
\label{prop:gram_wl_stronger_than_squared_distance_wl}
    Let $\psi_{d}$ and $\psi_{g}$ be RPEs such that, for any graph $G$, $\psi_{d,G}$ and $\psi_{g,G}$ are the squared distance and Gram matrix of a centered point set $P_{G}\subset\R^{d}$. Let $\psi_{g}'=(\psi_g,I)$ be the diagonal augmentation of $\psi_g$.  Let $(G,X_G)$ and $(H,X_H)$ with featured graphs. Let $u\in V_G$ and $x\in V_H$. If $\wlcolor{\psi_{g,G}'}{t+1}(u) = \wlcolor{\psi_{g,H}'}{t+1}(x)$, then $\wlcolor{\psi_{d,G}}{t}(u) = \wlcolor{\psi_{d,H}}{t}(x)$. In particular, $\psi_g'$-WL with diagonal augmentation is at least as strong as $\psi_d$-WL.
\end{proposition}
\begin{proof}
    We prove this by induction on $t$. For $t=0$, this is obvious as $\wlcolor{\psi_{g,G}'}{1}(u) = \wlcolor{\psi_{g,H}'}{1}(x)$ implies $\wlcolor{\psi_{g,G}'}{0}(u) = X_G(u) = X_H(v) = \wlcolor{\psi_{g,H}'}{0}(x)$ and equivalently $\wlcolor{\psi_{d,G}}{0}(u) = X_G(u) = X_H(v) = \wlcolor{\psi_{d,H}}{0}(x)$.
    \par
    Now suppose the proposition is true for a natural number $t$, and suppose that $\wlcolor{\psi_{g,G}'}{t+2}(u) = \wlcolor{\psi_{g,H}'}{t+2}(x)$. By induction, we know that $\wlcolor{\psi_{g,G}'}{t+1}(u) = \wlcolor{\psi_{g,H}'}{t+1}(x)$ implies $\wlcolor{\psi_{d,G}}{t}(u) = \wlcolor{\psi_{d,H}}{t}(x)$, so the first coordinates of $\wlcolor{\psi_{d,G}}{t+1}(u)$ and $\wlcolor{\psi_{d,H}}{t+1}(x)$ are equal.
    \par
    We now show the second coordinates are equal. As $\wlcolor{\psi_{g,G}'}{t+2}(u) = \wlcolor{\psi_{g,H}'}{t+2}(x)$, then there is a bijection $\sigma:V_G\to V_H$ such that $(\wlcolor{\psi_{g,G}'}{t+1}(v),(G_{P_G},I)_{uv}) = (\wlcolor{\psi_{g,H}'}{t+1}(\sigma(v)),(G_{P_G},I)_{x\sigma(v)})$ for all $v\in V_G$. The diagonal augmentation implies that $(G_{P_G}, I)_{uu} = (G_{P_H}, I)_{xx}$ as a diagonal element of $(G_{P_G},I)$ can only equal a diagonal element of $(G_{P_H},I)$. By a similar argument, the fact that $\wlcolor{\psi_{g,G}'}{t+1}(v)) = \wlcolor{\psi_{g,H}'}{t+1}(\sigma(v))$ implies that $(G_{P_G},I)_{vv} = (G_{P_H},I)_{\sigma(v)\sigma(v)}$. Therefore,
    \begin{align*}
        (D_{P_G}^{2})_{uv} = & (G_{P_G})_{uu} + (G_{P_G})_{vv} - 2(G_{P_G})_{uv} \\
        = & (G_{P_H})_{xx} + (G_{P_H})_{\sigma(v)\sigma(v)} - 2(G_{P_H})_{x\sigma(v)} = (D_{P_H}^2)_{x\sigma(v)}.
    \end{align*}
    Therefore, the second coordinates of $\wlcolor{\psi_{d,G}}{t+1}(u)$ and $\wlcolor{\psi_{d,H}}{t+1}(x)$ are equal, i.e.,
    \begin{equation*}
        \multiset{(\wlcolor{\psi_{d,G}}{t}(v),\, (D_{P_G}^{2})_{uv}) : v\in V_G} = \multiset{(\wlcolor{\psi_{d,H}}{t}(v),\, (D_{P_H}^{2})_{uv}) : v\in V_H}. \qedhere
    \end{equation*}

\end{proof}

Now we can prove \Cref{thm:spectral distance = kernel} below:

\spectraldistanceequalskernel*

\begin{proof}[Proof of \Cref{thm:spectral distance = kernel}]
    Note that for given graph $G$ with $n$ points, any spectral decomposition of the Laplacian $L=\sum_{i=2}^n\lambda_iz_iz_i^T$, and any function $f:\R^{+}\to\R^{+}$, we can define the point cloud $P_G=\{ \sqrt{f(L)} 1_v :  v\in V \}$, where $\sqrt{f(L)} = \sum_{i=2}^{n} \sqrt{f(\lambda_i)} z_iz_i^{T}$ The spectral kernel $K_G^f$ is the Gram matrix of this point cloud and the spectral distance $D^f_G$ is the distance matrix of this point cloud. (While different choice of bases for the eigenspace result in different point clouds $P_G$, these point clouds will always have the same Gram and distance matrices regardless of choice of basis.) This point cloud is centered as $\sum_{v\in V} \sqrt{f(L)} 1_v = \sum_{v\in V}\sum_{i=2}^{n}\sqrt{f(L)} \lambda_i z_iz_i^{T}1_v=0$ by the fact that each eigenvector $z_i$ is orthogonal to the all-1s vector, which is the eigenvector of $L$ corresponding to $\lambda_1=0$. The result for RPE-augWL with RPEs $K_{G}^{f}$ and $D^{f}_{G}$ holds by directly applying \Cref{thm:gram_wl_equals_distance_wl}.
\end{proof}

\paragraph*{RD-WL and $\boldsymbol{L^\dagger}$-WL}
\label{sec:resistance_wl_and_laplacian_wl}

In this section, we prove \Cref{thm:resistance_wl_equals_pinv_wl} by showing that it is unnecessary to perform diagonal augmentation on $L^{\dagger}$ in order for $L^{\dagger}$-WL to be as strong as RD-WL. In particular, we show that vertices $u$ and $v$ with the same $L^{\dagger}$-WL color after one iteration have the same diagonal entry $L^{\dagger}(u,u) = L^{\dagger}(v,v)$. This is the point in the proof of \Cref{prop:gram_wl_stronger_than_squared_distance_wl} where diagonal augmentation is used, so because this holds for $L^{\dagger}$-WL, we can drop the diagonal augmentation.
\par
For background, the \textit{\textbf{pseudoinverse}} of the Laplacian is $L^{\dagger}=\sum_{i=2}^{n} \frac{1}{\lambda_i} z_iz_i^{T}$, and the \textit{\textbf{resistance distance}} between nodes $u$ and $v$ in a graph is $RD(u,v) = (1_u-1_v)^{T}L^{\dagger}(1_u-1_v)$, where $1_u$ is the vector whose $u$-th entry is 1 and other entries are 0. The resistance distance is a \textbf{squared} spectral distance, but remarkably, it is also a metric (see for example \citep[Section 12.8]{spielman2019sagt}), which is not true in general for the square of distances.
\par
In particular, we prove the following lemma:
\begin{lemma}
    \label{lem:pseudo inverse diagonal}
            Let $G$ and $H$ be graphs. Let $u\in V_{G}$ and $v\in V_{H}$ such that $\chi^{(1)}_{L^{\dagger}_G}(u) = \chi^{(1)}_{L^{\dagger}_H}(v)$. Then $L^{\dagger}_G(u,u) = L^{\dagger}_H(v,v)$.
    \end{lemma}
The proof is based on several auxiliary results.

\begin{lemma}\label{lem:diagonal_pseudoinverse_laplacian_ineq}
    Let $G$ be a connected graph with more than 1 vertex and let $L$ denote its graph Laplacian. Then, for any $u\neq v\in V$, we have that $L^{\dagger}(v,v)>L^{\dagger}(u,v)$.
\end{lemma}

\begin{proof}
    By the fact that $RD(u,v)=L^{\dagger}(u,u)+L^{\dagger}(v,v)-2L^{\dagger}(u,v)$, we have that
    \[
        L^{\dagger}(v,v)-L^{\dagger}(u,v)=\frac{1}{2}(RD(u,v)+L^{\dagger}(v,v)-L^{\dagger}(u,u))
    \]
    Therefore, to prove that $L^{\dagger}(v,v) > L^{+}_{uv}$, it suffices to prove that
    \[
        RD(u,v)> L^{\dagger}(u,u)-L^{\dagger}(v,v).
    \]
    By \Cref{lem:diagonal_gram_matrix}, we have that
    \begin{align*}
    L^{\dagger}(u,u)-L^{\dagger}(v,v)&=\frac{1}{n}\left(\sum_{w\in V}RD(u,w)-\sum_{w\in V}RD(v,w)\right)\\
    &= \frac{1}{n}\sum_{w\in V}RD(u,w)-RD(v,w) \\
    &= \frac{1}{n}\sum_{w\in V\setminus\{u,v\}}RD(u,w)-RD(v,w) \tag{as $RD(v,u)=0$} \\
    &\leq \frac{1}{n}\sum_{w\in V\setminus\{u,v\}}RD(u,v) \tag{Triangle inequality.} \\
    &=\frac{n-2}{n}RD(u,v)<RD(u,v).  \qedhere
    \end{align*}
\end{proof}

\begin{lemma}\label{lem:diagonal_pseudoinverse_laplacian_general}
    Let $G$ be a graph and let $L$ denote its graph Laplacian. Let $u\in V$. Then either
    \begin{enumerate}
        \item $L^{\dagger}(u,u)>L^{\dagger}(u,v)$ for all $v\in V$, or
        \item $L^{\dagger}(u,v)=0$ for all $v\in V$.
    \end{enumerate}
\end{lemma}
\begin{proof}
    Let $G_1,\cdots,G_k$ denote the connected components of $G$. Then, we have that $L=L_1\oplus \cdots\oplus L_k$ and thus $L^\dagger=L_1^\dagger\oplus \cdots\oplus L_k^\dagger$. For each connected component $G_l$, if it contains more than 1 vertex, then by \Cref{lem:diagonal_pseudoinverse_laplacian_ineq},  $L_l^\dagger$ satisfies the condition that the diagonal entries are strictly larger than all other entries in the corresponding rows. If $G_l$ contains only 1 vertex, then $L_l^\dagger=0$ and hence the corresponding row and column in $L^\dagger$ is always 0. This concludes the proof.
\end{proof}

Now, we finish the proof of \Cref{lem:pseudo inverse diagonal}.

\begin{proof}[Proof of \Cref{lem:pseudo inverse diagonal}]
    By the definition of the $L^\dagger$-WL test, $\chi^{(1)}_{L^{\dagger}_G}(u) = \chi^{(1)}_{L^{\dagger}_H}(v)$ implies that $\multiset{L^{\dagger}_G(u,w) : w\in V_{G}} = \multiset{L^{\dagger}_H(v,x) : x\in V_{H}}$. By \Cref{lem:diagonal_pseudoinverse_laplacian_general}, we know that either the multisets $\multiset{L^{\dagger}_G(u,w) : w\in V_{G}}=\multiset{L^{\dagger}_H(v,x) : x\in V_{H}}$ are all 0s, in which case $L_G^{\dagger}(u,u)= 0 = L_H^{\dagger}(v,v)$, or they have the same maximal element, which is $L_G^{\dagger}(u,u)=L_H^{\dagger}(v,v)$.
\end{proof}

\subsection{SPE APE-GT is Stronger than RD-RPE-GT}
\label{sec:spe_and_resistance}

\paragraph{Background.} First, we provide the definition of SPE \cite{huang2023stability}. Denote the eigenvalues and eigenvectors of the Laplacian as $\lambda_{1}\leq\ldots\leq\lambda_{|V|}$ and $z_1,\ldots,z_{|V|}$. Let $\Lambda := [\lambda_{1},\ldots,\lambda_{|V|}]\in\R^{|V|}$ and $V:= [z_1|\cdots|z_{|V|}]\in\R^{|V|\times|V|}$, so $L=V\diag(\Lambda)V^{T}$.
\par
\textit{\textbf{Stable and expressive positional encoding (SPE)}} is the APE
$$
\SPE(V,\Lambda) = g(V\diag(f(\Lambda))V^{T}),
$$
where $f:\R^{n}\to\R^{n}$ is a permutation-equivariant function like an elementwise MLP or DeepSets~\cite{zaheer2017deep} and $g:\R^{n\times n}\to\R^{n}$ is any graph neural network. In our context, we assume that $f$ is a DeepSets network and $g$ is a 2-EGN.

By removing the final layer of 2-EGN from SPE, one obtains a natural RPE that we call the \textit{\textbf{relative-SPE (RSPE)}}:
$
\mathrm{RSPE}(V,\Lambda) = V\diag(f(\Lambda))V^{T}.
$

\begin{lemma}\label{lem:rspe_to_spe}
    If a chosen RSPE is diagonally-aware (or diagonally augmented), then there exists a 2-EGN $g$ such that RSPE-RPE-GT is equivalent to SPE-APE-GT in terms of distinguishing power.
\end{lemma}
\begin{proof}
    This follows from \Cref{thm:main-RPE-to-APE}.
\end{proof}

Notice that when $f$ is an elementwise function, RSPE reduces to a spectral kernel and we hence have the following result; however, RSPE is more general than spectral kernels as $f(\Lambda)_i$ may be a function of all the eigenvalues, not just the $i$th eigenvalue. Accordingly, we have the following lemma.

\begin{lemma}\label{lem:SPE spectral kernel}
    RSPE can compute any spectral kernel.
\end{lemma}


Based on the lemma, we are able to prove our main result in this section.

\spestrongerthanresistance*

\begin{proof}
    The precise statement of the result is that there exists a choice of RSPE with the diagonal augmentation such that its corresponding SPE-APE-GT is at least as strong as RD-RPE-GT.

     By \Cref{lem:rspe_to_spe}, we know that there exists a 2-EGN $g$ such that RSPE-RPE-GT is equivalent to SPE-APE-GT in terms of distinguishing power. Therefore, it suffices to show that RSPE-RPE-GT is at least as strong as RD-RPE-GT.
    Since the pseudoinverse of the Laplacian $L^{\dagger}$ is a spectral kernel, by \Cref{lem:SPE spectral kernel} RSPE can be chosen so that its corresponding RSPE-WL is at least as strong as $L^\dagger$-WL. Since the extra diagonal augmentation of the two RPEs will not change the relative distinguishing power of their RPE-augWL test, by \Cref{lem:RPE_transformers_equal_w_WL} and the fact that RD is diagonally-aware, we conclude the proof.
\end{proof}

\subsection{Powers of Matrices and Spectral Kernels: Proofs from \Cref{sec:powers_of_matrices}}
\label{sec:proof-powers-of-laplacian-equal-spectral-kernel}

We want to show that the multidimensional RPE $(I,L,\ldots,L^{k})$ is equivalent to spectral kernels in distinguishing power. To prove this, we will use a basic result from the interpolation literature showing that any function on $k$ points can be fit exactly by a $(k-1)$-degree polynomial $p$.

\begin{lemma}
\label{lem:polynomials_regression_guarantee}
    Let $f:\R\to\R$ be any function and let $\{x_1,\ldots,x_n\}\subset\R$. Then there exists an $(n-1)$-degree polynomial $p:\R\to\R$ such that $p(x_i)=f(x_i)$ for all $1\leq i\leq n$.
\end{lemma}

\powersoflaplacianspectralkernel*

\begin{proof}
    Let $G$ and $H$ be graphs such that $K^{f}$-WL distinguishes $G$ and $H$. Let $\{\lambda_{G,1},\ldots,\lambda_{G,n}\}$ and  $\{\lambda_{H,1},\ldots,\lambda_{H,n}\}$ be the spectra of the Laplacians $L_G$ and $L_H$ respectively, and denote their union $\{\lambda^{G}_1,\ldots,\lambda^{G}_n\}\cup\{\lambda^{G}_1,\ldots,\lambda^{G}_n\}=\{\lambda_0,\ldots,\lambda_{2n-1}\}$ (without loss of generality, we assume these elements are distinct). Let $p(x)=\sum_{i=0}^{2n-1}c_ix^{i}$ be the polynomial guaranteed by \Cref{lem:polynomials_regression_guarantee} for the point cloud $\{\lambda_0,\ldots,\lambda_{2n-1}\}$ and for the function $f$ in defining $K^f$. Let $p_{l}:\R^{2n}\to\R$ be the linear function with the same coefficients as $p$, i.e. $p_{l}(x) = \sum_{i=0}^{2n-1} c_ix_i$. By linearity, if we apply $p_l$ elementwise on the tensor $(I,L_G,\ldots,L_G^{2n-1})$, we find that it equals $K^f$ for $G$:
    $$
        p_{l}(I,L_{G},\ldots,L_{G}^{2n-1}) = \sum_{i=0}^{2n-1}c_iL_{G}^{i} =  \sum_{i=0}^{2n-1}\sum_{j=2}^{n} c_i\lambda_{G,j}^{i} z_{j}z_{j}^{T} =\sum_{j=2}^{n} p(\lambda_{G,j}) z_{j}z_{j}^{T} = \sum_{j=2}^{n} f(\lambda_{G,j}) z_{j}z_{j}^{T} = K^{f}_{G}
    $$
    The same calculation holds for $H$.
    Therefore, as $K^{f}$-WL distinguishes $G$ and $H$, then $p_l\circ(I,L,\ldots,L^{2n-1})$-WL distinguishes $G$ and $H$.
    \par
    Moreover,~\Cref{prop:concatenate rpe} implies $(I,L,\ldots,L^{2n-1})$-WL distinguishes $G$ and $H$. As the above argument can be repeated (using different polynomials $p$) to show that $(I,L,\ldots,L^{2n-1})$ can distinguish any pair of graphs that $K^{f}$ can, then $(I,L,\ldots,L^{2n-1})$-WL is as strong as $K^{f}$-WL on graphs with $n$ nodes.
\end{proof}

The following theorem uses normalized versions of the adjacency matrix and Laplacian. Recall that $D$ denotes the degree matrix, $A$ denotes the adjacency matrix, and $L$ denotes the graph Laplacian. The symmetrically-normalized adjacency matrix is $\widehat{A} = D^{-1/2}AD^{-1/2}$, and its symmetrically-normalized  Laplacian is $\widehat{L} = D^{-1/2}AD^{-1/2} = I-\widehat{A}$.
We let $\mu_1\leq\cdots\leq \mu_n $ denote eigenvalues of $\widehat{A}$ and let $\nu_1\leq\cdots\leq \nu_n$ denote eigenvalues of $\widehat{L}$

Note that $\nu_{i} = 1-\mu_{i}$. Moreover, $\widehat{A}$ and $\widehat{L}$ they have the same set of eigenvectors $\{x_1,\ldots,x_n\}$ (up to choice of orthonormal basis). For a function $f:\R^{+}\to\R^{+}$, we define the \textit{\textbf{normalized spectral kernel}} as $\widehat{K}^{f} = \sum_{i=2}^{n} f(\nu_i) x_ix_i^T$.

\stackofadjacencyspectralkernel*

\begin{proof}
    This theorem follows as the eigenvalues $\mu_{i}$ of $\widehat{A}$ and $\nu_{i}$ of $\widehat{L}$ can be matched as $\mu_{i} = 1-\nu_{i}$. The rest of the proof is the same as \Cref{thm:powers_of_laplacian_equal_spectral_kernel} by using polynomials to fit the function $\widehat{f}(x) = f(1-x)$.
\end{proof}

\heatkernelsgeneralkernels*

\begin{proof}
    This theorem follows as the eigenvalues of $H^{(t)}$ are $e^{-t\lambda_i} = (e^{-\lambda_i})^{t}$, so the proof is the same as the proof of \Cref{thm:powers_of_laplacian_equal_spectral_kernel} by using polynomials to fit the function $\widehat{f}(x) = f(-\log(x))$.
\end{proof}

\subsection{Common Matrices: Proof of \Cref{prop:common matrices}}

\commonmatrices*

\begin{proof}
    By \Cref{lem:wl-equals-tradition-wl} below, we know that $A$-WL is equivalent to the WL test. Hence, we only need to show that all other RPE-augWL tests are equivalent to $A$-WL.

    Choose any featured graphs $(G,X_G)$ and $(H,X_H)$. Pick $u\in V_G$ and $v\in V_H$.
    First of all, we observe the following fact.
    \begin{fact}\label{fact:degree}
        $\chi_{A_G}^{(1)}(u)=\chi_{A_H}^{(1)}(v)$ simplies that $\deg_G(u)=\deg_H(v)$.
    \end{fact}

    \paragraph{$\widetilde{A}$:} We inductively prove that $\chi_{A_G}^{(t)}(u)=\chi_{A_H}^{(t)}(v)$ iff $\chi_{\widetilde{A}_G}^{(t)}(u)=\chi_{\widetilde{A}_H}^{(t)}(v)$. The case when $t=0$ trivially holds. Assume that the statement holds for some $t\geq 0$. Then, we first assume that $\chi_{A_G}^{(t+1)}(u)=\chi_{A_H}^{(t+1)}(v)$. This implies that $\chi_{A_G}^{(1)}(u)=\chi_{A_H}^{(1)}(v)$. By \Cref{fact:degree}, we have that $\deg_G(u)=\deg_H(v)$. By induction, we have that $\wlcolor{\widetilde{A}_G}{t}(u)=\wlcolor{\widetilde{A}_H}{t}(v)$ and
    \[\multiset{(\wlcolor{\widetilde{A}_G}{t}(x),\widetilde{A}_G(u,x)):x\in V_G}=\multiset{(\wlcolor{\widetilde{A}_H}{t}(y),\widetilde{A}_H(v,y)):y\in V_H}.\]
    Hence, we have that
    \[\multiset{(\wlcolor{\widetilde{A}_G}{t}(x),A_G(u,x)/\deg_G(u)):x\in V_G}=\multiset{(\wlcolor{\widetilde{A}_H}{t}(y),A_H(v,y)/\deg_H(v)):y\in V_H},\]
   and thus $\chi_{\widetilde{A}_G}^{(t+1)}(u)=\chi_{\widetilde{A}_H}^{(t+1)}(v)$.

   Conversely, assume that $\chi_{\widetilde{A}_G}^{(t+1)}(u)=\chi_{\widetilde{A}_H}^{(t+1)}(v)$. This implies that $\chi_{{\widetilde{A}_G}}^{(1)}(u)=\chi_{{\widetilde{A}_H}}^{(1)}(v)$. Similarly to \Cref{fact:degree}, we have that $\deg_G(u)=\deg_H(v)$. Then, the above argument can be repeated to show that $\chi_{{A_G}}^{(t+1)}(u)=\chi_{{A_H}}^{(t+1)}(v)$.

   Therefore, $\widetilde{A}$-WL is equivalent to $A$-WL.

    \paragraph{$\widehat{A}$:} We prove that $\widehat{A}$-WL is equivalent to $\widetilde{A}$-WL and hence $A$-WL by showing that for any $u\in V_G$ and $v\in V_H$, one has for all $t\geq 0$ that
    \[\chi_{\widehat{A}_G}^{(t)}(u)=\chi_{\widehat{A}_H}^{(t)}(v)\,\text{ implies }\,\chi_{\widetilde{A}_G}^{(t)}(u)=\chi_{\widetilde{A}_H}^{(t)}(v)\text{  and } \chi_{\widetilde{A}_G}^{(t+1)}(u)=\chi_{\widetilde{A}_H}^{(t+1)}(v)\,\text{ implies }\,\chi_{\widehat{A}_G}^{(t)}(u)=\chi_{\widehat{A}_H}^{(t)}(v)\]

    The case when $t=0$ holds trivially.

    Now assume that the statement holds for any $t\geq 1$. $\chi_{\widehat{A}_G}^{(t+1)}(u)=\chi_{\widehat{A}_H}^{(t+1)}(v)$ implies that
    \begin{align*}
        &(\chi_{\widehat{A}_G}^{(t)}(u),\multiset{(\chi_{\widehat{A}_G}^{(t)}(x),A_G(u,x)/\sqrt{\deg_G(u)\deg_G(x)}):x\in V_G})\\
        &=(\chi_{\widehat{A}_H}^{(t)}(v),\multiset{(\chi_{\widehat{A}_H}^{(t)}(y),A_H(v,y)/\sqrt{\deg_H(v)\deg_H(y)}):y\in V_H}).
    \end{align*}
    By counting the number of non-zero elements on both sides, we have that $\deg_G(u)=\deg_H(v)$. Hence,
    \begin{align*}
        &(\chi_{\widehat{A}_G}^{(t)}(u),\multiset{(\chi_{\widehat{A}_G}^{(t)}(x),\deg_G(u),A_G(u,x)/\sqrt{\deg_G(x)}):x\in V_G})\\
        &=(\chi_{\widehat{A}_H}^{(t)}(v),\multiset{(\chi_{\widehat{A}_H}^{(t)}(y),\deg_H(v),A_H(v,y)/\sqrt{\deg_H(y)}):y\in V_H})
    \end{align*}
    and thus
    \begin{align*}
        &(\chi_{\widehat{A}_G}^{(t)}(u),\multiset{(\chi_{\widehat{A}_G}^{(t)}(x),\deg_G(u),A_G(u,x)):x\in V_G})\\
        &=(\chi_{\widehat{A}_H}^{(t)}(v),\multiset{(\chi_{\widehat{A}_H}^{(t)}(y),\deg_H(v),A_H(v,y)):y\in V_H}).
    \end{align*}
    Finally, by the inductive hypothesis we have that
    \begin{align*}
        \wlcolor{\widetilde{A}_G}{t+1}(u)&=(\chi_{\widetilde{A}_G}^{(t)}(u),\multiset{(\chi_{\widetilde{A}_G}^{(t)}(x),\widetilde{A}_G(u,x)):x\in V_G})\\
        &=(\chi_{\widetilde{A}_H}^{(t)}(v),\multiset{(\chi_{\widetilde{A}_H}^{(t)}(y),\widetilde{A}_H(v,y)):y\in V_H})=\wlcolor{\widetilde{A}_H}{t+1}(v).
    \end{align*}
    For the other direction, assume that $\chi_{\widetilde{A}_G}^{(t+2)}(u)=\chi_{\widetilde{A}_H}^{(t+2)}(v)$. This implies that
    \begin{align*}
        &(\chi_{\widetilde{A}_G}^{(t+1)}(u),\multiset{(\chi_{\widetilde{A}_G}^{(t+1)}(x),\widetilde{A}_G(u,x)):x\in V_G})\\
        &=(\chi_{\widetilde{A}_H}^{(t+1)}(v),\multiset{(\chi_{\widetilde{A}_H}^{(t+1)}(y),\widetilde{A}_H(v,y)):y\in V_H}).
    \end{align*}
    It is obvious that $\deg_G(u)=\deg_H(v)$.
    Let $\sigma:V_G\to V_H$ denote a bijection such that $(\chi_{\widetilde{A}_G}^{(t+1)}(x),\widetilde{A}_G(u,x))=(\chi_{\widetilde{A}_H}^{(t+1)}(\sigma(x)),\widetilde{A}_H(v,\sigma(x)))$ for all $x\in V_G$. Then, we have that $\deg_G(x)=\deg_H(\sigma(x))$. Hence, $(\chi_{\widetilde{A}_G}^{(t+1)}(x),\widehat{A}_G(u,x))=(\chi_{\widetilde{A}_H}^{(t+1)}(\sigma(x)),\widehat{A}_H(v,\sigma(x)))$. By the inductive hypothesis, we have that $(\chi_{\widehat{A}_G}^{(t)}(x),\widehat{A}_G(u,x))=(\chi_{\widehat{A}_H}^{(t)}(\sigma(x)),\widehat{A}_H(v,\sigma(x)))$. Therefore,
    \begin{align*}
        \chi_{\widehat{A}_G}^{(t+1)}(u)&=(\chi_{\widehat{A}_G}^{(t)}(u),\multiset{(\chi_{\widehat{A}_G}^{(t)}(x),\widehat{A}_G(u,x)):x\in V_G})\\
        &=(\chi_{\widehat{A}_H}^{(t)}(v),\multiset{(\chi_{\widehat{A}_H}^{(t)}(y),\widehat{A}_H(v,y)):y\in V_H})=\chi_{\widehat{A}_H}^{(t+1)}(v).
    \end{align*}



    $L,\widehat{L},\widetilde{L}$:
    Note that the graph Laplacian $L$ is combinatorially-aware. Hence by \Cref{thm:combinatorially-aware-rpe-stronger-than-WL}, $L$-WL is at least as strong as the WL test and hence $A$-WL test. For the other direction, we prove the stronger result that $\chi_{A_G}^{(t)}(u)=\chi_{A_H}^{(t)}(v)$ implies that $\chi_{L_G}^{(t)}(u)=\chi_{L_H}^{(t)}(v)$ for any $u\in V_G$ and $v\in V_H$.

    The statement trivially holds when $t=0$. Assume that the statement holds for some $t\geq 0$. Then, we assume that $\chi_{A_G}^{(t+1)}(u)=\chi_{A_H}^{(t+1)}(v)$. This implies that $\chi_{A_G}^{(1)}(u)=\chi_{A_H}^{(1)}(v)$. By \Cref{fact:degree}, we have that $\deg_G(u)=\deg_H(v)$. By definition, $\chi_{A_G}^{(t)}(u)=\chi_{A_H}^{(t)}(v)$ and hence by induction hypothesis, $\chi_{L_G}^{(t)}(u)=\chi_{L_H}^{(t)}(v)$. Now, from the fact that $\multiset{(\chi_{A_G}^{(t)}(x),A_G(u,x)):x\in V_G}=\multiset{(\chi_{A_H}^{(t)}(y),A_H(v,y)):y\in V_H}$, we obtain a bijection $\sigma:V_G\to V_H$ such that $(\chi_{A_G}^{(t)}(x),A_G(u,x))=(\chi_{A_H}^{(t)}(\sigma(x)),A_H(v,\sigma(x)))$ for all $x\in V_G$. We assume that $\sigma(u)=v$ (otherwise we can modify $\sigma$ so that the assumption holds). Then, by inductive hypothesis we have that $\chi_{L_G}^{(t)}(x)=\chi_{L_H}^{(t)}(\sigma(x))$ for all $x\in V_G$. Furthermore, from the fact that $A_G(u,x)=A_H(v,\sigma(x))$, $\sigma(u)=v$ and $\deg_G(u)=\deg_H(v)$, we have that $L_G(u,x)=L_H(v,\sigma(x))$ for all $x\in V_G$. Hence, $\multiset{(\chi_{L_G}^{(t)}(x),L_G(u,x)):x\in V_G}=\multiset{(\chi_{L_H}^{(t)}(y),L_H(v,y)):y\in V_H}$. Therefore, $\chi_{L_G}^{(t+1)}(u)=\chi_{L_H}^{(t+1)}(v)$. This concludes the proof.

    Therefore,  $L$-WL is equivalent to $A$-WL. Similarly, $\widehat{L}$-WL is equivalent to $\widehat{A}$-WL and $\widetilde{L}$-WL is equivalent to $\widetilde{A}$-WL. Hence, all of these RPE-augWL tests are equivalent to the WL test.
\end{proof}

\paragraph{Magnetic Laplacian}

In this section, we prove \Cref{prop:magnetic_laplacian} showing that the magnetic Laplacian is as most as strong as an RPE as the concatenation of the RPEs $({D^*},A,A^T)$.

\begin{lemma}
    The magnetic Laplacian $L^\alpha=f^\alpha({D^*},A,A^T)$ where $f^\alpha:\R^3\to\R$ is defined
    \[f^\alpha(x,y,z):=x-\exp(\iota2\pi\alpha \cdot\mathrm{sgn}(z-y))\cdot \frac{y+z}{2}.\]
    and is applied to each channel of $(D^{*}, A,A^{T})$.
\end{lemma}

\magneticlaplacians*

\begin{proof}
    As the magnetic Laplacian $L^{\alpha}$ is an elementwise function of the concatenation $(D^{*}, A, A^T)$, then by~\Cref{prop:concatenate rpe} Part 1, $(D^{*},A,A^{T})$-WL is at least as strong as $L^{\alpha}$-WL.
\end{proof}

\subsection{Combinatorial-Awareness: Proof of \Cref{thm:combinatorially-aware-rpe-stronger-than-WL}}
\label{sec:wl_and_combinatorially_aware_rpes}



\combinatoriallyawarerpewl*

For any two featured graphs
$(G,X_G)$ and $(H,X_H)$, we prove the following stronger lemma that immediately implies~\Cref{thm:combinatorially-aware-rpe-stronger-than-WL}.

\begin{lemma}
    Let $\psi$ be a combinatorially-aware RPE. Let $u\in G$ and $v\in H$. If $\wlcolor{\psi_G}{t}(u)=\wlcolor{\psi_H}{t}(v)$, then $\wlcolor{G}{t}(u) = \wlcolor{H}{t}(v)$.
\end{lemma}
\begin{proof}
    We prove this by induction on the iteration. The case $t=0$ holds trivially by definition of the two WL tests.
    \par
    Now inductively suppose the theorem is true for iteration $t$. Let $u\in G$ and $v\in H$ such that $\wlcolor{\psi_G}{t+1}(u)=\wlcolor{\psi_H}{t+1}(v)$.
    \par
    First, by definition, $\wlcolor{\psi_G}{t}(u) = \wlcolor{\psi_H}{t}(v)$. This and the inductive hypothesis imply $\wlcolor{G}{t}(u) = \wlcolor{H}{t}(v)$.
    \par
    Next, by definition, there is a bijection $\sigma:V_G\to V_H$ such that $(\wlcolor{\psi_G}{t}(w),\psi_{G}(u,w)) = (\wlcolor{\psi_H}{t}(\sigma(w)), \psi_{H}(v, \sigma(w)))$ for all $w\in V_G$. This implies $\psi_{G}(u,w) = \psi_{H}(v,\sigma(w))$ if and only if both $\{u,w\}$ and $\{v,\sigma(w)\}$ are edges or non-edges; therefore, $\multiset{ \wlcolor{\psi_G}{t}(w) : \{u,w\}\in E_G } = \multiset{ \wlcolor{\psi_H}{t}(x) : \{v,x\}\in E_H }$. This and the inductive hypothesis imply $\multiset{ \wlcolor{G}{t}(w) : \{u,w\}\in E_G } = \multiset{ \wlcolor{H}{t}(x) : \{v,x\}\in E_H }$. This and the observation from the previous paragraph imply $\wlcolor{G}{t+1}(u) = \wlcolor{H}{t+1}(v)$.
\end{proof}

Note that the adjacency matrix $A$ is trivially combinatorially-aware and hence $A$-WL is at least as strong as the WL. In fact, the other direction also holds and hence we have that the $A$-WL and WL are equally strong.

\begin{proposition}[RPE-augWL generalizes WL]\label{lem:wl-equals-tradition-wl}
    Consider the RPE given by the adjacency matrix $A$. Given two graphs $G$ and $H$, then $\chi(G)=\chi(H)$ iff $\chi_A(G)=\chi_A(H)$.
\end{proposition}
\begin{proof}
We only need to prove the direction that the WL is at least as strong as the $A$-WL. Let $(G,X_G)$ and $(H,X_H)$ be two featured graphs. We prove the following stronger statement:
\begin{lemma}\label{lem:WL = A-WL}
    Assume that $\multiset{\wlcolor{G}{t}(u):u\in V_G}=\multiset{\wlcolor{H}{t}(v):v\in V_H}$. Let $u\in G$ and $v\in H$. If $\wlcolor{G}{t}(u) = \wlcolor{H}{t}(v)$, then $\wlcolor{A_G}{t}(u)=\wlcolor{A_H}{t}(v)$ and hence $\multiset{\wlcolor{A_G}{t}(u):u\in V_G}=\multiset{\wlcolor{A_H}{t}(v):v\in V_H}$.
\end{lemma}
\begin{proof}[Proof of \Cref{lem:WL = A-WL}]
    We prove this by induction on the iteration. The case $t=0$ holds trivially by definition of the two WL tests.
    \par
    Now inductively suppose the theorem is true for iteration $t$. Let $u\in G$ and $v\in H$ such that $\wlcolor{G}{t+1}(u)=\wlcolor{H}{t+1}(v)$. By definition, $\wlcolor{G}{t}(u)=\wlcolor{H}{t}(v)$ and $\{\!\!\{ \wlcolor{G}{t}(w) : \{u,w\}\in E_G \}\!\!\} = \{\!\!\{ \wlcolor{H}{t}(x) : \{v,x\}\in E_H \}\!\!\}$. By induction assumption, we have that $\wlcolor{G}{t}(u)=\wlcolor{H}{t}(v)$. Furthermore, there is a bijection $\sigma:N_G(u)\to N_H(v)$ such that $\wlcolor{G}{t}(w)=\wlcolor{H}{t}(\sigma(w))$ for all $w\in N_G(u)$. Now, since $\multiset{\wlcolor{G}{t}(x):x\in V_G}=\multiset{\wlcolor{H}{t}(y):y\in V_H}$, we have that $\multiset{\wlcolor{G}{t}(x):x\in V_G\backslash N_G(u)}=\multiset{\wlcolor{H}{t}(y):y\in V_H\backslash N_H(v)}$. Hence there exists a bijection $\tau:V_G\backslash N_G(u)\to V_H\backslash N_H(v)$ such that $\wlcolor{G}{t}(x)=\wlcolor{H}{t}(\tau(x))$ for all $x\in V_G\backslash N_G(u)$. Now, we define a bijection $\Phi:V_G\to V_H$ as the union of $\sigma$ and $\tau$. In this way, we have that for any $x\in V_G$, $A_G({u,x})=A_H(v,\Phi(x))$ and $\wlcolor{A}{t}(x)=\wlcolor{A}{t}(\Phi(x))$. Hence, we have that
    \[(\wlcolor{A_G}{t}(u),\multiset{(\wlcolor{A_G}{t}(x), A_G(u,x)):x\in V_G})=(\wlcolor{A_H}{t}(v),\multiset{(\wlcolor{A_H}{t}(y), A_H(v,y)):y\in V_H}).\]
    This implies that $\wlcolor{A_G}{t+1}(u)=\wlcolor{A_H}{t+1}(v)$ and we hence concludes the proof.
\end{proof}

This concludes the proof.
\end{proof}

\subsection{Resistance Distance is not Combinatorially-Aware}\label{sec:RD not ca}

In \Cref{fig:RDnotca}, we show a graph $G$ so that $\text{RD}(1,3)=\text{RD}(4,5)=1$ but $\{4,5\}\in E_G$ and $\{1,3\}\notin E_G$. This proves  that the resistance distance is not combinatorially-aware.
\begin{figure}[htbp!]
    \centering
    \includegraphics[scale=0.8]{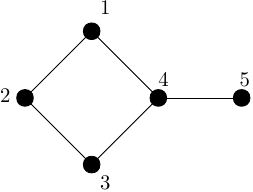}
    \caption{$\text{RD}(1,3)=\text{RD}(4,5)=1$. }
    \label{fig:RDnotca}
\end{figure}

\subsection{Resistance Distance and Block Cut-Edge Component Trees}
\label{sec:effective_resistance_and_cut_edges}

For a connected graph $G$, a \textit{\textbf{cut edge}} is an edge $e\in E_G$ such that $G\setminus\{e\}$ is disconnected. An \textit{\textbf{(edge-)biconnected component}} of $G$ is a maximal subset $U\subset V_G$ such that the subgraph of $G$ induced by $U$ is connected and has no cut edges. The \textit{\textbf{block cut-edge tree}} of $G$ is the tree whose vertices are the edge-biconnected components $BCC(G)$ of $G$ and its edges are $\{ \{C_1,C_2\} : C_1\neq C_2\in BCC(G),\, \exists \{u,v\}\in E_G \text{ such that } u\in C_1\,\&\, v\in C_2 \}$, i.e. the edges of the block cut-edge tree correspond to cut edges in $G$ that connect two edge-biconnected components.
\par
Similarly, a \textit{\textbf{cut vertex}} is a vertex $v\in V_G$ such that $G\setminus\{v\}$ is disconnected.
\par
In this section, we prove the following theorem.

\begin{theorem}\label{thm:cut edge}
    Let $G$ and $H$ be unfeatured graphs.
    \begin{enumerate}
        \item Let $\{u,v\}\in E_{G}$ and $\{x,y\}\in E_H$. If $\multiset{\chi_{RD_G}^{(t)}(u), \chi_{RD_G}^{(t)}(v)} = \multiset{\chi_{RD_H}^{(t)}(x), \chi_{RD_H}^{(t)}(y)}$ for all $t\geq 0$, then $\{u,v\}$ is a cut edge if and only if $\{x,y\}$ is a cut edge.
        \item If $\chi_{RD}(G) = \chi_{RD}(H)$, then the block cut-edge trees of $G$ and $H$ are isomorphic.
    \end{enumerate}
\end{theorem}

\paragraph{Background} The following are well-known facts about effective resistance that will be useful for proving this theorem.

\begin{lemma}
\label{lem:cut_edge_one_eff}
    Let $G$ be a graph with $\{u, v\}\in E_G$.  Then $RD_G(u,v)\leq 1$. Moreover, $RD_G(u,v) = 1$ if and only if $\{u,v\}$ is a cut edge.
\end{lemma}

\begin{lemma}
    Let $G$ be a graph with $u,v,w\in V$. Then $RD_{G}(u,v) + RD_{G}(v,w) = RD_{G}(u,w)$ if and only if $v$ is a cut vertex such that $u$ and $w$ are disconnected in $G\setminus\{v\}$.
\end{lemma}

\paragraph{Notation} For any pair of vertices $u$ and $v$ of a graph $G$, their resistance distance is denoted by $RD_G(u,v)$; however, we may drop the subscript $G$ when it is clear from the context.

Recall that the notation $\chi_{RD}(G) = \chi_{RD}(H)$ means that two graphs $G$ and $H$ have the same multisets of the RD-WL colorings of their vertices for all $t\geq 0$, i.e.
$\multiset{\chi_{RD_G}^{(t)}(v) : v\in V_G} = \multiset{\chi_{RD_H}^{(t)}(v) : v\in V_H}$,
where $\chi_{RD_G}$ and $\chi_{RD_H}$ are RD-WL colorings of $G$ and $H$, respectively.
Similarly, we define $\chi_{RD_G}(u)$ as the ordered tuple $(\wlcolor{RD_G}{0}(u), \wlcolor{RD_G}{1}(u),\ldots)$. In other words, $\chi_{RD_G}(u) = \chi_{RD_H}(v)$ if and only if $\chi_{RD_G}^{(t)}(u) = \chi_{RD_H}^{(t)}(v)$ for all $t\geq 0$. We call these the \textit{\textbf{colors}} of the graphs and nodes, respectively. We prove the following lemma about this new notation.

\begin{lemma}
\label{lem:stable_coloring_lemma}
    Let $u\in V_G$ and $u'\in V_H$ be such that $\chi_{RD_G}(u) = \chi_{RD_H}(u')$. Then $\multiset{(\chi_{RD_G}(v), RD_G(u,v)) : v\in V_G} = \multiset{(\chi_{RD_H}(v'), RD_H(u',v')) : v' \in V_H}$.
\end{lemma}
\begin{proof}
    If not, then there is some iteration $t$ such that $\multiset{(\chi^{(t)}_{RD_G}(v), RD_G(u,v)) : v\in V_G} \neq \multiset{(\chi^{(t)}_{RD_H}(v'), RD_H(u',v')) : v\in V_H}$. In this case, clearly $\chi^{(t+1)}_{RD_G}(u) \neq \chi^{(t+1)}_{RD_H}(u')$.
\end{proof}

\begin{lemma}
\label{lem:identical_dist_list}
Let $C$ be any set of colors. Let $u\in V_G$ and $u'\in V_H$ be such that $\chi_{RD_G}(u) = \chi_{RD_H}(u')$.  Let $D(u)$ (respectively, $D(u')$) be the multisets of the resistance distances of all vertices of $G$ (respectively, $H$) with color $C$ from $u$ (respectively, $u'$), i.e.
$D(u) = \multiset{RD_G(u, v) : v\in V_G,\,\chi_{RD_G}(v)\in C}$ and
$D(u') = \multiset{RD_H(u', v') : v'\in V_H,\, \chi_{RD_H}(v)\in C}$.
Then $D(u) = D(u')$.  In particular, the mean, minimum, and maximum of the multisets $D(u)$ and $D(u')$ are equal.
\end{lemma}
\begin{proof}
If these sets are not the same, then by~\Cref{lem:stable_coloring_lemma}, it cannot be the case that $\chi_{RD_G}(u) = \chi_{RD_H}(u')$.
\end{proof}

\begin{lemma}
\label{lem:one_singleton_side}
Let $\{u,v\}$ be a cut edge in $G$, and let $S_u$ and $S_v$ be the vertex sets of the connected components of $G\setminus\{u,v\}$ such that $u\in S_u$ and $v\in S_v$. Let $C_{uv}\subset V$ be the subset of all vertices that have RD-WL-color $\chi_{RD_G}(u)$ or $\chi_{RD_H}(v)$. Then $S_u\cap C_{uv} = \{u\}$ or $S_v\cap C_{uv} = \{v\}$.
\end{lemma}
\begin{proof}
Let $x$ be the farthest vertex of $C_{uv}$ (with respect to the resistance distance) from the set $\{u, v\}$ in $G$. Assume $x\neq u$ and $x\neq v$, otherwise the statement of the lemma trivially holds.  Without loss of generality assume $x\in S_u$.  We show that $C_{uv}\cap S_v = \{v\}$.

\begin{figure}[h]
    \centering
    \includegraphics[height=1.0in]{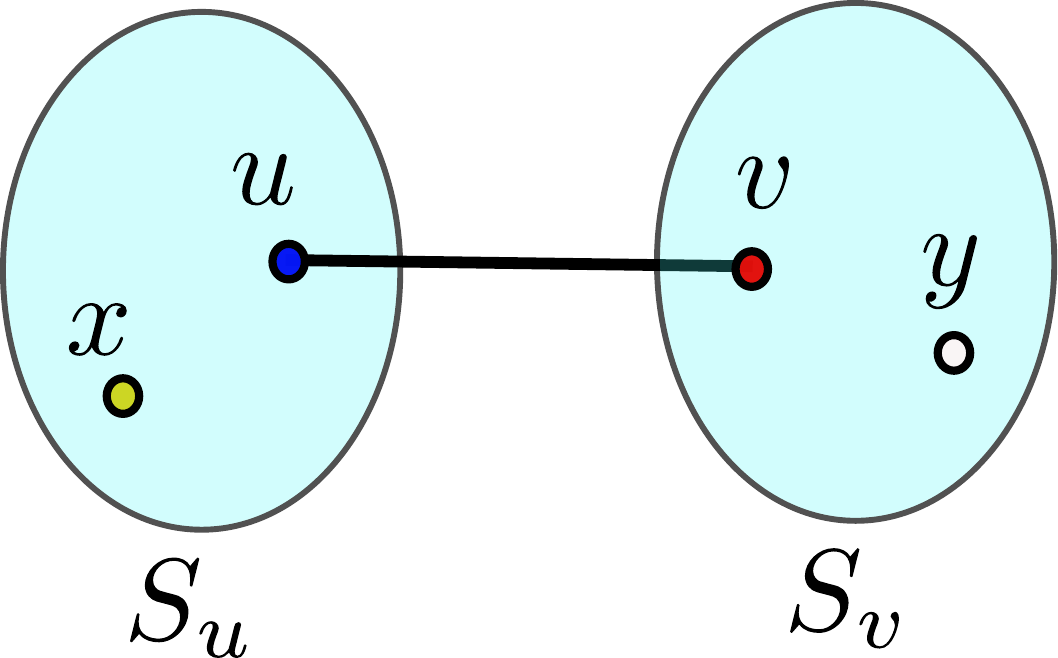}
\end{figure}

Let $y\in S_v$, and $y\neq v$.
Let $m_y$, $m_u$ an $m_v$ be the maximum (resistance) distance of $y$, $u$ and $v$ from any vertex in $C_{uv}$.
Note $m_u = RD(u,x)$, and $m_v = RD(v, x) = 1 + RD(u,x)$ because of the choice of $x$.
Further, we have $m_y \geq RD(y, v) + RD(v, x) = RD(y, v) + m_v > m_v$, and that $m_y \geq RD(y, u) + RD(u, x) = RD(y,u) + m_u = RD(y,v) + 1 + m_u > m_u$ as $\{u,v\}$ is a cut edge.
Thus, $m_y \neq m_u$ and $m_y \neq m_v$. In particular, the list of distances of both $v$ and $u$ to vertices in $C_{u,v}$ are different from the list of distances of $y$ to these vertices.
Hence, by \Cref{lem:identical_dist_list}, the color $\chi_{RD_G}(y)$ is different from both of the colors $\chi_{RD_G}(u)$ and $\chi_{RD_G}(v)$.
\end{proof}

\begin{theorem}
\label{thm:same_color_cut_edge}
Let $G$ and $H$ be two graphs with the same RD-WL colorings.
Let $\{u,v\}\in E_G$ be a cut edge, and let $\{u',v'\}\in E_H$ be such that $\chi_{RD_G}(u) = \chi_{RD_H}(u')$ and $\chi_{RD_G}(v) = \chi_{RD_H}(v')$. Then, $\{u',v'\}$ is a cut edge in $H$.
\end{theorem}
\begin{proof}
Let $S_u$ and $S_v$ be the vertex sets of the connected components of $G\backslash \{u,v\}$ such that $u\in S_u$ and $v\in S_v$.
Also, let $C^G_{uv}$ (respectively, $C^H_{u'v'}$) be the set of all vertices that are colored $\chi_{RD_G}(u)$ or $\chi_{RD_G}(v)$ in $G$ (respectively, in $H$).
By \Cref{lem:one_singleton_side}, and without loss of generality, we assume $C^G_{uv} \cap S_v = \{v\}$.  Therefore, the closest vertex of $C^G_{uv}$ to $v$ is $u$, which is at distance 1 to $v$ since $\{u,v\}$ is a cut edge; the vertex $u$ is the closest to $v$ as there are no other vertices besides $v$ in $C^G_{uv} \cap S_v$, and any vertex in $C^G_{uv} \cap S_u$ is at distance at least 1 from $v$ via the series formula of the resistance distance.
\par
Since $\chi_{RD_G}(v) = \chi_{RD_H}(v')$, the closest vertex in $C^H_{u'v'}$ to $v'$ has distance 1 from $v'$---otherwise $\chi_{RD_H}(v')$ would be different from $\chi_{RD_G}(v)$ by \Cref{lem:identical_dist_list}. Moreover, $RD_H(u',v')\leq 1$ as $\{u',v'\}\in E_H$. Thus, $RD_H(u',v')=1$ and, by \Cref{lem:cut_edge_one_eff}, $\{u',v'\}$ is a cut edge.
\end{proof}

\begin{lemma}
\label{lem:left_right_separation}
Let $G$ and $H$ be two graphs with the same RD-WL colorings.
Let $\{u,v\}$ be a cut edge in $G$ and $S_u$, $S_v$ be the partitions of $G\backslash \{u,v\}$.
Let $\{u',v'\}$ be a cut edge in $H$ such that $\chi_{RD_G}(u) = \chi_{RD_H}(u')$ and $\chi_{RD_G}(v) = \chi_{RD_H}(v')$, and let $S'_{u'}$, $S'_{v'}$ be partitions of $H\backslash \{u',v'\}$.
Let $C$ be any subset of colors used in these equivalent colorings.
Let $D_v = \multiset{d_1, \ldots, d_k}$ be the list of distances of vertices with color $C$ in $S_v$ from $v$
and let $L_u = \multiset{\ell_1, \ldots, \ell_t}$ be the list of distances of vertices with color $C$ in $S_u$ from $u$.
Similarly, let $D'_{v'} = \multiset{d'_1, \ldots, d'_{k'}}$ be the list of distances of vertices with color $C$ in $S'_{v'}$ from $v'$
and let $L'_{u'} = \multiset{\ell'_1, \ldots, \ell'_{t'}}$ be the list of distances of vertices with color $C$ in $S'_{u'}$ from $u'$.
Then, we have $D_v = D'_{v'}$ and $L_u = L'_{u'}$. In particular, $|D_v| = |D'_{v'}|$ and $|L_u| = |L'_{u'}|$.
\begin{figure}[h]
    \centering
    \includegraphics[height=1.0in]{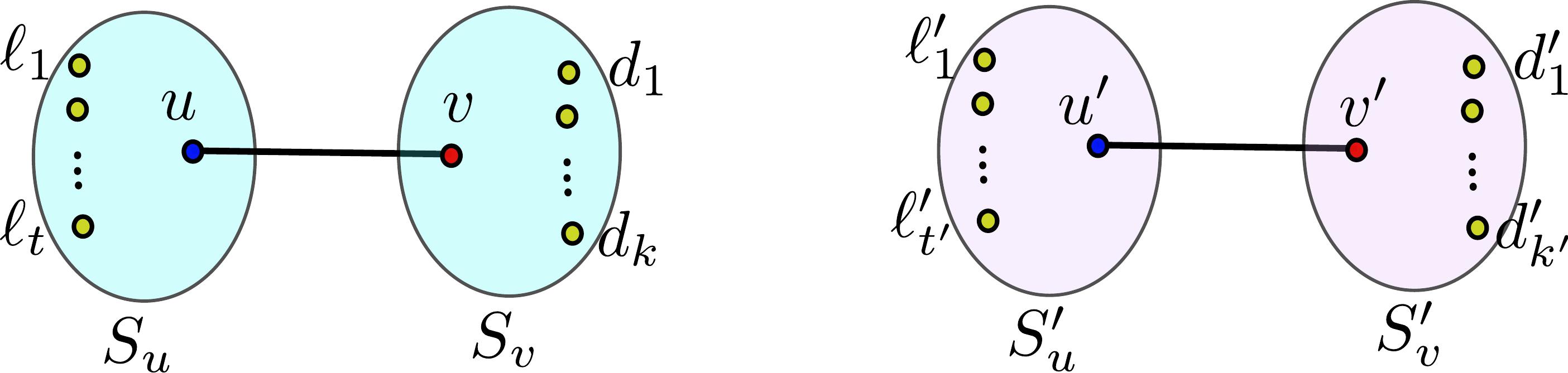}
\end{figure}
\end{lemma}
\begin{proof}
Without loss of generality, by permuting indices, let
$
D_I = D_v\cap D'_{v'} = \multiset{d_1, \ldots, d_p} = \multiset{d'_1, \ldots, d'_p}
$, and
$
L_I = L_u\cap L'_{u'} = \multiset{\ell_1, \ldots, \ell_q} = \multiset{\ell'_1, \ldots, \ell'_q}
$.  Hence,
\[
D'_{v'} = \multiset{
    d_1, \ldots, d_p, d'_{p+1}, \ldots, d'_{k'}
},
\]
and
\[
L'_{u'} = \multiset{
    \ell_1, \ldots, \ell_q, \ell'_{q+1}, \ldots, \ell'_{t'}
}.
\]
By the definition of $D_I$,
\begin{equation}
\label{eqn:disjoint_1}
\multiset{
    d'_{p+1}, \ldots, d'_{k'}
} \cap \multiset{
    d_{p+1}, \ldots, d_{k}
} = \emptyset,
\end{equation}
and
\begin{equation}
\label{eqn:disjoint_2}
\multiset{
    \ell'_{q+1}, \ldots, \ell'_{t'}
} \cap \multiset{
    \ell_{q+1}, \ldots, \ell_{t}
} = \emptyset.
\end{equation}
Since $\chi_{RD_G}(v) = \chi_{RD_H}(v')$, by Lemma \ref{lem:identical_dist_list}, the list of distances of vertices with color $C$ from $v$ in $G$ is the same as the list of vertices with color $C$ from $v'$ in $H$. As $\{u,v\}$ and $\{u',v'\}$ are cut edges, then $D_v \cup (L_u+1) = D'_{v'} \cup (L'_{u'} + 1)$, where $+1$ is applied to all elements of a multiset.  Specifically,
\begin{align*}
\multiset{
    d_1, \ldots, d_k, \ell_1+1, \ldots, \ell_t+1
} &= \multiset{
    d'_1, \ldots, d'_{k'}, \ell'_1+1, \ldots, \ell'_{t'}+1
} \\
&= \multiset{
    d_1, \ldots, d_p, d'_{p+1}, \ldots d'_{k'}, \ell_1+1, \ldots, \ell_q+1, \ell'_{q+1}+1\ldots \ell'_{t'}+1
}
\end{align*}
Therefore, removing $D_I$ and $L_I$ from these multisets, we conclude
\begin{align*}
\multiset{
    d_{p+1}, \ldots, d_k, \ell_{q+1}+1, \ldots, \ell_t+1
} = \multiset{
    d'_{p+1}, \ldots d'_{k'}, \ell'_{q+1}+1\ldots \ell'_{t'}+1
}
\end{align*}
Equations (\ref{eqn:disjoint_1}) and (\ref{eqn:disjoint_2}) imply
\begin{align*}
\multiset{
    d_{p+1}, \ldots, d_k
} = \multiset{
    \ell'_{q+1}+1\ldots \ell'_{t'}+1
},
\end{align*}
and
\begin{align}
\label{eqn:ellplusone_is_dprim}
\multiset{
    \ell_{q+1}+1, \ldots, \ell_t+1
} = \multiset{
    d'_{p+1}, \ldots d'_{k'}
}.
\end{align}

On the other hand, since the list of distances of vertices with color $C$ from $u$ in $G$ is the same as the list of vertices with color $C$ from $u'$ in $H$, i.e.~$(D_v+1) \cup L_u = (D'_{v'}+1) \cup L'_{u'}$.  Specifically,
\begin{align*}
\multiset{
    d_1+1, \ldots, d_k+1, \ell_1, \ldots, \ell_t
} &= \multiset{
    d'_1+1, \ldots, d'_{k'}+1, \ell'_1, \ldots, \ell'_{t'}
} \\
&= \multiset{
    d_1+1, \ldots, d_p+1, d'_{p+1}+1, \ldots d'_{k'}+1, \ell_1, \ldots, \ell_q, \ell'_{q+1}\ldots \ell'_{t'}
}.
\end{align*}
Therefore,
\begin{align*}
\multiset{
    d_{p+1}+1, \ldots, d_k+1, \ell_{q+1}, \ldots, \ell_t
} = \multiset{
    d'_{p+1}+1, \ldots d'_{k'}+1, \ell'_{q+1}\ldots \ell'_{t'}
}.
\end{align*}
Again, Equations (\ref{eqn:disjoint_1}) and (\ref{eqn:disjoint_2}), imply
\begin{align*}
\multiset{
    d_{p+1}+1, \ldots, d_k+1
} = \multiset{
    \ell'_{q+1}\ldots \ell'_{t'}
},
\end{align*}
and
\begin{align}
\label{eqn:ellminusone_is_dprim}
\multiset{
    \ell_{q+1}, \ldots, \ell_t
} = \multiset{
    d'_{p+1}+1, \ldots d'_{k'}+1
}.
\end{align}

Combining Equations (\ref{eqn:ellplusone_is_dprim}) and (\ref{eqn:ellminusone_is_dprim}), we conclude
\[
\multiset{
    \ell_{q+1}, \ldots, \ell_t
} = \multiset{
    d'_{p+1}+1, \ldots d'_{k'}+1
} = \multiset{
    \ell_{q+1}+2, \ldots, \ell_t+2
}.
\]
Thus,
$
\multiset{
    \ell_{q+1}, \ldots, \ell_t
}
$
is empty, implying $L_u = L'_{u'}$, and $D_v = D'_{v'}$, as desired.
\end{proof}

\begin{lemma}
\label{lem:preserve_cut_edge}
Let $G$ and $H$ be two graphs with the same RD-WL colorings.
Let (i) $\{u,v\}$ be a cut edge in $G$, (ii) $u'\in V_H$ such that $\chi_{RD_G}(u) = \chi_{RD_H}(u')$, (iii)
$v'\in V_H$ such that $\chi_{RD_H}(v') = \chi_{RD_G}(v)$, and $RD_H(u',v') = 1$.  Then, $\{u',v'\}$ must be a cut edge in $H$.
\end{lemma}

Note that the lemma does not trivially follow from \Cref{lem:cut_edge_one_eff} since $\{u',v'\}$ is not assumed to be an edge a priori.
\begin{proof}
Let $C_{uv}$ be the set of all vertices that are colored $\chi_{RD_G}(u)$ or $\chi_{RD_G}(v)$ in $G$.
By \Cref{lem:one_singleton_side}, $S_u\cap C_{uv} = \{u\}$ or $S_v\cap C_{uv} = \{v\}$.  Assume that $S_v\cap C_{uv} = \{v\}$; the other case is similar.

Suppose, to derive a contradiction, that $\{u',v'\}$ is not a cut edge.
Let $w'$ be a neighbor of $v'$ that resides in the same connected component as $u'$ in $H\backslash\{v'\}$. (In particular, if $v'$ is not a cut vertex $w'$ can be any neighbor of $v')$. Such a $w'$ exists and is distinct from $u'$, otherwise $\{u',v'\}$ would be a cut edge. We know that $RD_H(w', u') < RD_H(w', v')+RD_H(v',u') = RD_H(w', v')+1$; the triangle inequality implies the inequality, and it is a strict inequality because, in the case of equality, $v'$ would a cut vertex separating $w'$ and $u'$, contradicting the choice of $w'$. Let $\ell = RD_H(w', v')$, and note that $\ell < 1$, as otherwise $\{w',v'\}$ is a cut edge between $v'$ and $u'$ implying that $RD_H(u', v') > 1$.

Since $\chi_{RD_G}(v) = \chi_{RD_H}(v')$, there must exist a vertex $w\in V_G$ such that $\chi_{RD_G}(w) = \chi_{RD_H}(w')$ and $RD_G(v, w) = \ell < 1$.  Since all vertices in $S_u$ are at a distance at least one from $v$, we have $w \in S_v$.  Let $L_{w}$ be the list of distances of vertices in $C_{uv}$ from $w$.  Since $S_v\cap C_{uv} = \{v\}$, the smallest distance in $L_{w}$ is $\ell$ and the second smallest distance in $\ell+1$, to $v$ and $u$, respectively.

Let $C_{u'v'}$ be the set of all vertices in $H$ that are colored $\chi_{RD_H}(u')=\chi_{RD_G}(u)$ or $\chi_{RD_H}(v') = \chi_{RD_G}(v)$.  Then, let $L_{w'}$ be the list of all distances of $w'$ to vertices in $C_{u'v'}$.  Since $\chi_{RD_H}(w') = \chi_{RD_G}(w)$, then $L_{w'} = L_w$ by~\Cref{lem:identical_dist_list}.  In particular, the smallest two numbers in $L_{w'}$ are $\ell$ and $\ell+1$.  We know $RD_H(w', v') = \ell$, and by the triangle inequality $RD_H(w', u') \leq RD_H(w', v') + 1 = \ell+ 1$.  Since, the second smallest number in the list is $\ell+1$, and not smaller, the inequality must be tight, i.e. $RD_H(w', u') = RD_H(w', v') + 1 = RD_H(w', v') + RD_H(v', u')$.  Therefore, $v'$ is a cut vertex separating $u'$ and $w'$, contradicting the choice of $w'$.
\end{proof}

\begin{lemma}
\label{lem:same_outgoing_edges}
Let $G$ and $H$ be two graphs with the same RD-WL colorings.
Let $\{u,v\}$ be a cut edge in $G$ and
$\{u',v'\}$ a cut edge in $H$ such that $\chi_{RD_G}(u) = \chi_{RD_H}(u')$ and $\chi_{RD_G}(v) = \chi_{RD_H}(v')$.
Furthermore, let $B_v$ (resp.~$B'_{v'}$) be the biconnected component of $v$ (resp.~$v'$) in $G$ (resp.~$H$). Finally, let $K$ (resp.~$K'$) be the set of all cut edges except $\{u,v\}$ (resp.~$\{u',v'\}$) with one endpoint in $B_v$ (resp.~$\{u',v'\}$). Then there exists a bijection between $f:K\rightarrow K'$, such that for any $\{x,y\}\in K$ and $\{w,z\} = f(\{x,y\})$, if $x$ in on $v$'s side in $G\backslash \{x,y\}$ and $w$ is on $v'$'s side in $H\backslash \{w,z\}$, then $\chi_{RD_G}(x) = \chi_{RD_H}(w)$ and $\chi_{RD_G}(y) = \chi_{RD_H}(z)$.
\end{lemma}
\begin{proof}
Let $S_v$ (resp.~$S'_{v'}$) be the vertex set of the connected component of $G\backslash \{u,v\}$ (resp.~$H\backslash \{u',v'\}$) that contains $v$ (resp.~$v'$).
\par
Let $X = \{x_1, \ldots, x_k\}$ be the set of endpoints of $K$ that are in $B_v$ ordered such that $RD(v, x_1) \leq \ldots \leq RD(v, x_k)$. (Note that each $x_i$ may be connected to multiple cut edges of $K$.)
\begin{figure}[h]
    \centering
    \includegraphics[height=1.0in]{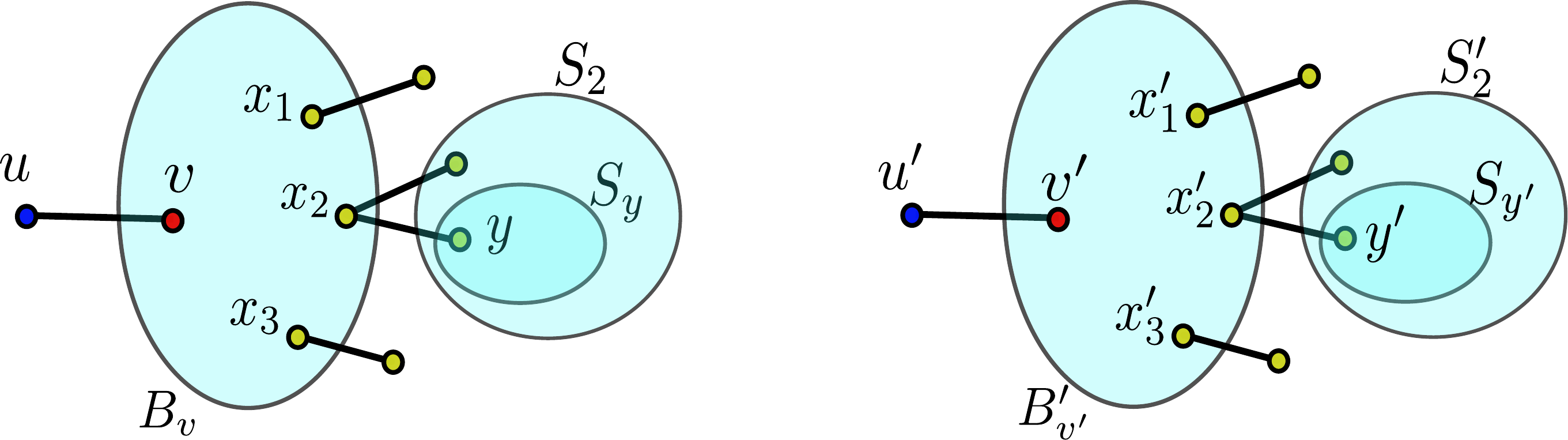}
\end{figure}

By \Cref{lem:left_right_separation}, there exists $X'=\{x'_1, \ldots, x'_k\}\subset S'_{v'}$ such that for every $i\in [k]$
\begin{enumerate}
    \item [(1)] $\chi_{RD_G}(x_i) = \chi_{RD_H}(x'_i)$ and
    \item [(2)] $RD_G(v, x_i) = RD_H(v', x'_i)$.
\end{enumerate}
(Let $C_X$ be the set of all vertices that are colored by the colors $\{\chi_{RD_G}(x_1), \ldots, \chi_{RD_G}(x_k)\}$ in $S_v$. Note, $C_X$ is a (possibly proper) superset of $\{x_1, \ldots, x_k\}$.  Therefore, by \Cref{lem:left_right_separation}, there exists a set $C'_X \subseteq S_{v'}$, such that there is a one-to-one correspondence between $C_X$ and $C'_X$ with corresponding pairs have the same color and distance to $v$ and $v'$. The existence of $X'$ with the aforementioned properties follows as $X\subseteq C_X$.)

For each $i\in [k]$, let $Y_i$ be the set of all vertices that are adjacent to $x_i$ with a cut edge in $K$. Since $\chi_{RD_G}(x_i) = \chi_{RD_H}(x'_i)$, by \Cref{lem:preserve_cut_edge}, for each $i\in [k]$, there exists $Y'_i$ such that (1) for all $y'\in Y'_i$, $\{x'_i,y'\}$ is a cut edge, and (2) the multisets $\multiset{\chi_{RD_G}(y)}_{y\in Y_i} = \multiset{\chi_{RD_H}(y')}_{y'\in Y'_i}$.

Next, we show for every $i\in [k]$ and every $y'\in Y'_i$, $RD_H(v', x'_i) < RD_H(v', y'_i)$, or equivalently, $v'$ is in the connected component of $x'_i$ in $H\backslash \{x'_i ,y'\}$. We prove this by induction on $i$.

For each $i\in[k]$, and each $y\in Y_i$, let $S_y$ be the set of all vertices on the $y$ side in $G\backslash \{x_i,y\}$, and let $S_i =  \bigcup_{y\in Y_i}{S_y}$.
For all $h<i$, and each $y'\in Y'_h$, let $S'_{y'}$ be the set of all vertices on the $y'$ side in $H\backslash \{x'_h,y'\}$, and let $S'_h = \bigcup_{y'\in Y'_h}{S'_{y'}}$.
Note $S_v\backslash\bigcup_{h<i}{S_h}$ contains $v$ and, by the induction hypothesis, $S'_{v'}\backslash\bigcup_{h<i}{S'_h}$ contains $v'$.

By \Cref{lem:left_right_separation} and that $RD_G(v, x_h) = RD_H(v', x'_h)$, $S_h$ and $S'_h$ induce the same set of color/distance pairs from $v$ and $v'$, respectively; i.e.~there is a one-to-one map $f$ from $S_h$ to $S'_h$ such that for each pair $(q, f(q))$, (i) $q$ and $f(q)$ have the same color, and (ii) $RD_G(v, q) = RD_H(v', f(q))$. Thus, $B_{v, i}= S_v\backslash\bigcup_{h<i}{S_h}$ and $B'_{v, i} = S'_{v'}\backslash\bigcup_{h<i}{S'_h}$ induce the same set of color/distance pairs from $v$ and $v'$, respectively.
Let $C$ be the set of colors of vertices that are incident to cut edges.
The $i$th smallest distance with color from $C$ in $B_{v, i}$ corresponds to $x_i$.  We know $x'_i$ has the same color as $x_i$, and $RD_G(v, x_i) = RD_H(v', x'_i)$. If there exists $y'\in Y'_i$ such that the $RD_H(v', y') < RD_H(v', x_i)$ then the $i$th smallest number in $B'_{v, i}$ would be smaller than $RD_G(v, x_i)$ which is a contradiction.  Hence, $RD_H(v', x'_i) < RD_H(v', y')$ for all $y'\in Y'_i$, as desired.

Therefore, $B'_{v'} = S'_{v'}\backslash\bigcup_{i\in[k]}{S'_i}$ is connected, and it has the same set of colors as $B_v = S_v\backslash\bigcup_{i\in[k]}{S_i}$.  So, it must be the biconnected component that contains $v'$.
\end{proof}

\begin{theorem}
\label{thm:same_color_same_cut_edge_tree}
    If two graphs $G$ and $H$ have the same RD-WL colorings, then their block cut-edge trees are isomorphic.
\end{theorem}
\begin{proof}
    Let $\{u,v\}$ be a cut edge in $G$ such that $S_u$ is a biconnected component.
    Let $\{u',v'\}$ be an edge in $H$ with $\chi_{RD_G}(u) = \chi_{RD_H}(u')$ and $\chi_{RD_G}(v) = \chi_{RD_H}(v')$.
    By \Cref{lem:same_outgoing_edges}, $S'_{u'}$ is a biconnected component of $H$.
    Root the block cut-edge tree of $G$ at $S_u$ and the block cut-edge tree of $H$ at $S'_{u'}$ so all cut edges are directed away from $S_u$ and $S'_{u'}$.
    \par
    For a directed cut edge $u\rightarrow v$, define $T_G[v]$ the be the block cut-edge subtree rooted at the biconnected component of $v$.
    We use induction on the size of $T_G[v]$ to show that for two same colored edges with the same direction $u\rightarrow v$ and $u'\rightarrow v'$, $T_G[v]$ and $T_H[v']$ are isomorphic.
    \par
    Let $B_v$ and $B'_{v'}$ be the biconnected components of $v$ in $G$ and $v'$ in $H$ respectively. For the base case that $B_v=S_v$, then by \Cref{lem:same_outgoing_edges}, $B'_{v'}=S'_{v'}$ and the statement holds. Alternatively, also by \Cref{lem:same_outgoing_edges}, the outgoing directed edges of $B_v$ and $B'_{v'}$ are in one-to-one correspondence with matching head and tail colors.  Let these sets of edges be $K$ and $K'$, respectively, and let $f$ be the bijection between them that preserves colors.  For any $x\rightarrow y\in K$, and its image $f(x)\rightarrow f(y) = w\rightarrow z\in K'$, by the induction hypothesis, the block cut-edge subtrees $T_G[y]$ and $T_H[z]$ are isomorphic. Therefore, $T_G[v]$ and $T_H[v']$ are isomorphic.
\end{proof}


\section{Experiments}\label{sec:experiments}

In this section, we carry out experiments to validate our two main results \Cref{thm:main-APE-to-RPE} and \Cref{thm:main-RPE-to-APE}. Our code is adapted from the GraphGPS module~\cite{rampavsek2022recipe} and subsequent fork from~\citet{muller2023attending}.

\subsection{Graph Isomorphism: CSL}\label{sec:CSL}

\begin{table}[htbp!]
\centering
\caption{Test performance on the CSL dataset of different APEs. Results are shown in the form of mean $\pm$ standard deviation. Experiments are averaged over 5 runs.}
\begin{tabular}{|c | c c|}
    \hline
    \multicolumn{3}{|c|}{CSL Dataset: Classification Accuracy ($\uparrow$)} \\
    \hline
    & HKdiagSE & RWSE \\
    \hline
    APE-GT & $100\pm0.00$ & $100\pm0.00$  \\
    DeepSet RPE-GT & $100\pm0.00$ & $100\pm0.00$ \\
    \hline
\end{tabular}
\label{table:graph_classification_results-ape-to-rpe}
\end{table}

\begin{table}[htbp!]
\centering
\caption{Test performance on the CSL dataset of different RPEs. Results are shown in the form of mean $\pm$ standard deviation. Experiments are averaged over 5 runs. The result for GIN was taken from \citet{muller2023attending}.}
\begin{tabular}{|c | c c c c c|}
    \hline
    \multicolumn{6}{|c|}{CSL Dataset: Classification Accuracy ($\uparrow$)} \\
    \hline
    & RSPE & Powers of Adjacency & RD & SPD & Adjacency \\
    \hline
    RPE-GT & $100\pm0.00$ & $100\pm0.00$ & $100\pm0.00$ & $90\pm0.00$ & $10\pm0.00$\\
    EGN APE-GT & $100\pm0.00$ & $100\pm0.00$ & $100\pm0.00$ & $90\pm0.00$ & $10\pm0.00$\\
    \hline
     GIN & \multicolumn{5}{|c|}{ $10\pm0.00$} \\
    \hline
\end{tabular}
\label{table:graph_classification_results}
\end{table}

We consider the graph classification task on the Circular Skip Links (CSL) dataset \cite{dwivedi2023benchmarking}. The goal of the CSL dataset is to classify graphs according to their isomorphism type. This is classified as a ``Hard'' task by \citet{muller2023attending}.
\par
To validate \Cref{thm:main-APE-to-RPE}, we consider two different APEs: Random Walk Structural Encoding (RWSE) and Heat Kernel Diagonal Structural Encoding (HKdiagSE) from the GraphGPS Library~\cite{rampavsek2022recipe}. For both APEs, we used times $\{1,...,20\}$. We test these APEs with APE-GTs and their corresponding RPEs by applying a DeepSet and the result GTs are denoted DeepSet RPE-GTs.
\par
To validate \Cref{thm:main-RPE-to-APE}, we consider four different RPEs: RSPE~\cite{huang2023stability}, a stack of 20 powers of the adjacency matrix (inspired by RRWP~\cite{ma2023inductive}), the resistance distance (RD), the shortest-path distance (SPD), and the adjacency matrix. We then consider their corresponding APEs by applying a 2-EGN and the result GTs are denoted EGN APE-GTs. We ran the RPE-GT and APE-GT for 1000 epochs each. Our results are averaged over 5 random seeds. See \Cref{table:graph_classification_results-ape-to-rpe} and \Cref{table:graph_classification_results} for the results.
\par
As our theoretical results (\Cref{thm:main-APE-to-RPE,thm:main-RPE-to-APE}) predict, RPE-GTs and EGN APE-GTs have the same classification accuracy when using the same RPE. Likewise, APE-GT and DeepSets RPE-GT have the same classification accuracy when using the same APE. This validates our main theorems.
\par
Moreover, note that the accuracy of Adjacency-RPE-GT and GIN are equal, which agrees with~\Cref{cor:combinatorially-aware-rpe-stronger-than-WL}.

\subsection{Graph Isomorphism: BREC}
\label{sec:experiments_brec}

\begin{table}[htbp!]
\centering
\caption{Performance of different RPEs on the BREC dataset. Basic, Regular, Extension, and CFI are subsets of the BREC dataset. Results are shown in the form of mean $\pm$ standard deviation. Experiments are averaged over 5 runs.}
\begin{tabular}{|c|c c c c|c|}
    \hline
    \multicolumn{6}{|c|}{BREC Dataset: Pair-Distinguishing Accuracy ($\uparrow$)} \\
    \hline
    & Basic & Regular & Extension & CFI & Total \\
    \hline
    \multicolumn{6}{|c|}{RPE: Resistance Distance (RD)} \\
    \hline
    RPE-GT & $100.00 \pm 0.00$ & $35.71\pm 0.00$ & $100.00\pm 0.00$ & $7.60\pm 1.14$ & $54.40 \pm 0.29$ \\
    EGN APE-GT & $100.00 \pm 0.00$ & $35.71\pm 0.00$ & $100.00\pm 0.00$ & $7.00\pm 1.23$ & $54.25 \pm 0.31$  \\
    \hline
    \multicolumn{6}{|c|}{RPE: Shortest-Path Distance (SPD)} \\
    \hline
    RPE-GT & $26.67 \pm 0.00$ & $9.29\pm 0.00$ & $40.8\pm 0.45$ & $6.8\pm 0.11$ & $19.15\pm 0.22$ \\
    EGN APE-GT & $26.67 \pm 0.00$ & $9.42\pm 0.28$ & $41.00\pm 0.00$ & $11.80\pm 0.45$ & $20.50 \pm 0.18$  \\
    \hline
    \multicolumn{6}{|c|}{RPE: Stable Positional Encoding (SPE)} \\
    \hline
    RPE-GT & $97.00\pm 2.17$ & $34.71\pm 0.57$ & $96.00\pm1.87$ & $3.2\pm 4.47$ & $51.5\pm 0.73$ \\
    EGN APE-GT & $0.97\pm 2.17$ & $33.14\pm 0.97$ & $98.80\pm 0.84$ & $8.80\pm1.30$ & $53.05\pm0.54$ \\
    \hline
    \multicolumn{6}{|c|}{RPE: Stack of Powers of Adjacency Matrix} \\
    \hline
    RPE-GT & $100.00\pm 0.00$ & $35.00\pm 00$ & $94.6\pm 0.00$ & $9.80\pm 0.84$ & $53.55\pm 0.28$ \\
    EGN APE-GT & $83.00\pm 1.39$ & $32.86\pm0.00$ & $88.20\pm 1.79$ & $3.00\pm 00$ &  $46.75\pm 0.47$ \\
    \hline
    \multicolumn{6}{|c|}{RPE: Adjacency Matrix} \\
    \hline
    RPE-GT & $0.00\pm 0.00$ & $0.00\pm 0.00$ & $0.00\pm 0.00$ & $0.00\pm 0.00$ & $0.00\pm 0.00$ \\
    EGN APE-GT & $0.00\pm 0.00$ & $0.00\pm 0.00$ & $0.00\pm 0.00$ & $0.00\pm 0.00$ & $0.00\pm 0.00$  \\
    \hline
\end{tabular}
\label{table:brec_results}
\end{table}

\begin{table}[htbp!]
\centering
\caption{Performance of different APEs on the BREC dataset. Basic, Regular, Extension, and CFI are subsets of the BREC dataset. Results are shown in the form of mean $\pm$ standard deviation. Experiments are averaged over 5 runs.}
\begin{tabular}{|c|c c c c|c|}
    \hline
    \multicolumn{6}{|c|}{BREC Dataset: Pair-Distinguishing Accuracy ($\uparrow$)} \\
    \hline
    & Basic & Regular & Extension & CFI & Total \\
    \hline
    \multicolumn{6}{|c|}{APE: Heat Kernel Diagonal Structural Encoding (HKdiagSE)} \\
    \hline
    APE-GT & $56.00\pm 4.50$ & $26.00\pm 0.57$ & $38.00\pm 3.16$ & $3.00\pm 00$ & $27.75\pm 1.43$  \\
    DeepSets RPE-GT & $59.00\pm 4.01$ & $18.86\pm 1.84$  & $59.60\pm 6.11$ & $1.2\pm 1.30$ & $30.65\pm 1.57$ \\
    \hline
    \multicolumn{6}{|c|}{APE: Random Walk Structural Encoding (RWSE)} \\
    \hline
    APE-GT & $84.33\pm 3.65$ & $31.29\pm 0.83$ & $62.40\pm 4.39$ & $1.00\pm 0.00$ & $39.45\pm 1.10$ \\
    DeepSets RPE-GT & $69.67\pm 5.32$ & $21.29\pm 2.73$ & $63.80\pm 2.39$ & $0.4\pm0.55$ & $33.95\pm 1.81$ \\
    \hline
\end{tabular}
\label{table:brec_ape_results}
\end{table}

We consider the graph isomorphism benchmark dataset BREC~\citep{wang2024brec}. The BREC dataset tests a graph neural network architecture's ability to distinguish pairs of non-isomorphic graphs. Unlike the CSL dataset which poses the graph isomorphism problem as a graph classification problem, the BREC dataset considers pairs of non-isomorphic graphs and tests whether a graph neural networks can learn different representations for the two graphs. The graph neural network is trained via contrastive learning on sets of graphs pairs with randomly permuted node indices. The graph neural network is then evaluated using the $T^{2}$ test on the learned features for each of the two graphs. The exact training and evaluation procedures can be found in the original BREC dataset paper~\citep{wang2024brec}.
\par
We compared several RPEs on the BREC dataset. We found that RD, SPE, and a stack of 20 powers of the adjacency matrix all perform similarly well and perform better than SPD in terms of total accuracy. Moreover, the adjacency matrix achieves $0\%$ accuracy, which matches the theoretical results of~\Cref{prop:common matrices} as all pairs of graphs in the BREC dataset are WL indistinguishable.
\par
We also compared two APEs on the BREC dataset: RWSE and HKdiagSE from the GraphGPS library~\cite{rampavsek2022recipe}. For both APEs, we used times $\{1,...,20\}$. We found that RWSE slightly outperformed HKdiagSE.
\par
Finally, for the most part, both RPEs and APEs achieved similar performance for their different architectures, i.e. RPE-GT and EGN APE-GT or APE-GTs and DeepSets RPE-GT. This matches our theoretical results (\Cref{thm:main-APE-to-RPE,thm:main-RPE-to-APE}).
\par
However, one interesting thing we found is that there are several graph pairs in the CFI subset that SPD-GTs were able to learn to distinguish while RD-GTs did not. This suggests that SPD-WL and RD-WL may be incomparable. \citet{zhang2023rethinking} show there are graphs that RD-WL can distinguish that SPD-WL cannot; this suggests the converse is true too, although it is not a definitive proof of this fact and should not be interpreted as such. A question for future research would be to prove that these pairs of graphs are indistinguishable by RD-WL or to find a RD-GT that could distinguish these pairs of graphs.
\par
Hyperparameter search was performed for each pair of PE and architecture.

\subsection{Graph Regression: ZINC}\label{sec:ZINC}

\begin{table}[htbp!]
\centering
\caption{Test performance on the Small ZINC dataset. Results are shown in the form of mean $\pm$ standard deviation. Experiments are averaged over 3 runs.}
\begin{tabular}{|c | c | c  c c|}
    \hline
    \multicolumn{5}{|c|}{ZINC Dataset: MAE ($\downarrow$)} \\
    \hline
    & \# Parameters  & RD & SPD & SPE  \\
    \hline
    RPE-GT & 573922  & $0.096\pm 0.002$ & $0.130\pm0.010$ & $0.132\pm0.007$ \\
    EGN APE-GT & 520514 & $0.196\pm 0.004$  & $0.217\pm0.011$ & $0.200\pm 0.007$ \\
    \hline
    & \# Parameters  & RD+EF & SPD+EF & SPE+EF  \\
    \hline
    RPE-GT & 696226 & $0.069\pm 0.004$ & $0.092\pm 0.010$ & $0.092\pm 0.001$ \\
    EGN APE-GT & 664818 & $0.110\pm 0.003$ & $0.132\pm 0.004$ & $0.126\pm 0.003$ \\
    \hline
\end{tabular}
\label{table:graph_classification_results_ZINC}
\end{table}

Although our theoretical results are for the distinguishing power of graph transformers and apply most directly to graph classification/isomorphism tasks, we would also like to compare RPEs and APEs on real-world tasks beyond those. In this experiment, we compare RPE-GTs and EGN APE-GTS for graph regression on the small ZINC dataset containing 12k graphs \cite{dwivedi2023benchmarking}. We consider the RPEs resistance distance (RD), shortest-path distance (SPD), and SPE~\cite{huang2023stability}. See \Cref{table:graph_classification_results_ZINC} for the results.
\par
We test the models both with and without the use of edge features of graphs in the ZINC dataset. Edge features are distinct from RPEs as they are not dependent on the graph structure and are associated to edges and not a pairs of vertices. For the ZINC dataset, the features are one-hot encodings of the bond type (e.g. single, double, triple). When we include edge features, we concatenate them to the RPE and pass the concatenated tensor to the model in place on the RPE. This use of edge features is not captured by our theoretical results, and it is not clear if this is the best or most natural way of including edge features in a graph transformer. Other papers have proposed different ways of incorporating edge features into a graph transformer~\citep{ma2023inductive, jin2023edgeformers}.
\par
From \Cref{table:graph_classification_results_ZINC}, we see that the performance of EGN APE-GT is worse than the corresponding RPE-GT. This could be partially explained by \Cref{thm:main-RPE-to-APE}, which states that RPE-GTs have stronger distinguishing power  than EGN APE-GTs for graphs with node features (as the graphs in ZINC do). Additionally, we suspect that the training of the extra EGN layer introduces additional difficulty in the training process, leading to worse performance of APE-GT compared to RPE-GT. This result aligns with our expectations, suggesting that instead of converting an RPE into an APE, it may be better to use the RPE directly in a RPE-GT.

\paragraph{Hyperparameters \& experimental setup.}

Our results are averaged over 3 random seeds. We ran the RPE-GT and APE-GT for 1000 epochs each.
\par
Following \citet{zhang2023rethinking}, the functions $f_1$ and $f_2$ in the RPE-GT head are embeddings into Gaussian kernels followed by an MLP.
\par
Hyperparameter search was performed for RPE-GTs and EGN APE-GTs. We chose the final hyperparameters for each architecture to have similar numbers of parameters.
Hyperparameters involved in our models are listed in \Cref{table:graph_classification_results_ZINC-summary}.

\begin{table}[ht!]
\centering
\caption{}
\begin{tabular}{|c | c |}
    \hline
    \multicolumn{2}{|c|}{Hyperparameters for ZINC Experiment} \\
    \hline
    \multicolumn{2}{|c|}{RPE GTs} \\
    \hline
    \# Transformer Layers: & 14 \\
    \# Transformer Heads: & 8 \\
     \# Gaussian Kernels & 16  \\
     \# MLP Layers & 2 \\
     MLP Hidden Dimension (No Edge Features) & 16 \\
     MLP Hidden Dimension (Edge Features) & 16 \\
    \hline
    \multicolumn{2}{|c|}{EGN APE GTs} \\
    \hline
    \# Transformer Layers: & 8 \\
    \# Transformer Heads: & 8 \\
    \# EGN Layers & 6 \\
    EGN Hidden Dim (No Edge Features) & 48 \\
    EGN Hidden Dim (Edge Features) & 64 \\
    APE Type & Add \\
    \hline
    \multicolumn{2}{|c|}{SPE} \\
    \hline
    \# DeepSets Layers & 3 \\
    DeepSets Hidden Dimension & 64 \\
    \# Parameters & 17217 \\
    \hline
\end{tabular}
\label{table:graph_classification_results_ZINC-summary}
\end{table}

\subsection{Code} Code for all experiments can be found at \small{\url{https://github.com/blackmit/comparing_graph_transformers_via_positional_encodings}}

\end{document}